\documentclass[]{interact}

\usepackage{epstopdf}% To incorporate .eps illustrations using PDFLaTeX, etc.
\usepackage[caption=false]{subfig}% Support for small, `sub' figures and tables

\usepackage[numbers,sort&compress]{natbib}% Citation support using natbib.sty
\bibpunct[, ]{[}{]}{,}{n}{,}{,}% Citation support using natbib.sty
\renewcommand\bibfont{\fontsize{10}{12}\selectfont}% Bibliography support using natbib.sty

\theoremstyle{plain}% Theorem-like structures provided by amsthm.sty
\newtheorem{theorem}{Theorem}[section]
\newtheorem{lemma}[theorem]{Lemma}
\newtheorem{corollary}[theorem]{Corollary}

\newtheorem{assumption}[theorem]{Assumption}

\theoremstyle{definition}
\newtheorem{definition}[theorem]{Definition}

\theoremstyle{remark}

\usepackage{algorithm}
\usepackage{algorithmic}

\usepackage{enumitem}

\usepackage{nicefrac}
\usepackage{sidecap}

\usepackage{tcolorbox}
\usepackage{pifont}
\definecolor{mydarkgreen}{RGB}{39,130,67}
\newcommand{\green}{\color{mydarkgreen}}
\definecolor{mydarkred}{RGB}{192,47,25}

\newcommand{\cmark}{\green\ding{51}}%
\newcommand{\xmark}{\ding{55}}%

\renewcommand{\algorithmicrequire}{\textbf{Input:}}
\renewcommand{\algorithmicensure}{\textbf{Output:}}
\usepackage{pifont}

\newcommand{\R}{\mathbb{R}}
\newcommand{\eqdef}{:=}

\def\<#1,#2>{\left\langle #1,#2\right\rangle}

\usepackage[colorinlistoftodos,bordercolor=orange,backgroundcolor=orange!20,linecolor=orange,textsize=scriptsize]{todonotes}

\newcommand{\peter}[1]{\todo[inline]{{\textbf{Peter:} \emph{#1}}}}

\newcommand{\cD}{{\cal D}}

\newcommand{\cO}{{\cal O}}

\newcommand{\sign}{\mathrm{sign}}

\newcommand{\EE}{\mathbf{E}}
\newcommand{\VV}{\mathbf{V}}

\newcommand{\prox}{\mathop{\mathrm{prox}}\nolimits}
\newcommand{\proxR}{\prox_{\gamma R}}
\newcommand{\proxkR}{\prox_{\gamma^k R}}

\newcommand{\sumin}{\sum_{i=1}^n}

\usepackage{hyperref}
\graphicspath{{./}{./jupyter/img/}}

\definecolor{mycolor}{RGB}{0,150,70}
\newcommand{\revision}[1]{{\color{black}#1}}

\begin{document}

\articletype{}% Specify the article type or omit as appropriate

\title{Distributed Learning with Compressed Gradient Differences\thanks{\revision{The first version of this paper appeared on arXiv in January 2019 \cite{mishchenko2019distributed}. The changes include writing, presentation, and numerical experiments.}}}

\author{
\name{K.~Mishchenko\textsuperscript{a}\thanks{Contact K.Mishchenko. Email: konsta.mish@gmail.com}, E.~Gorbunov\textsuperscript{b,d}, M.~Tak\'{a}\v{c}\textsuperscript{b}, P.~Richt\'arik\textsuperscript{c}}
\affil{\textsuperscript{a}INRIA, Paris, France; \textsuperscript{b}Mohamed bin Zayed University of Artificial Intelligence, Abu Dhabi, UAE; \textsuperscript{c}King Abdullah University of Science and Technology, Thuwal, KSA; \textsuperscript{d}Moscow Institute of Physics and Technology, Moscow, Russia\thanks{Part of the work was done while E.~Gorbunov was a researcher at MIPT.}}
}

\maketitle

\begin{abstract}
Training large machine learning models requires a distributed computing approach, with communication of the model updates  being the bottleneck. For this reason, several methods based on the compression (e.g., sparsification and/or quantization) of updates were recently proposed, including {\tt QSGD} \cite{alistarh2017qsgd}, {\tt TernGrad} \cite{wen2017terngrad}, {\tt SignSGD} \cite{pmlr-v80-bernstein18a}, and {\tt DQGD} \cite{khirirat2018distributed}. However, none of these methods are able to learn the gradients, which renders them incapable of converging to the true optimum in the batch mode. In this work we propose a new distributed learning method---{\tt DIANA}---which resolves this issue via compression of {\em gradient differences}. We perform a theoretical analysis in the strongly convex and nonconvex settings and show that our rates are superior to existing rates. We also provide theory to support non-smooth regularizers study the difference between quantization schemes. Our analysis of block-quantization and differences between $\ell_2$ and $\ell_{\infty}$ quantization closes the gaps in theory and practice. Finally, by applying our analysis technique to {\tt TernGrad}, we establish the first convergence rate for this method.
\end{abstract}


\section{Introduction}

Big machine learning models are typically trained in a distributed fashion, with  the training data distributed across several workers, all of which compute in parallel an update to the model based on their local data. For instance, they can all perform a single step of Gradient Descent ({\tt GD}) or Stochastic Gradient Descent ({\tt SGD}). These updates are then sent to a parameter server which performs aggregation (typically this means just averaging of the updates) and then broadcasts  the aggregated updates back to the workers. 
The process is repeated until a good solution is found.

When doubling the amount of computational power, one usually expects to see the learning process finish in half time. If this is the case, the considered system is called to scale linearly. For various reasons, however, this does not happen, even to the extent that the system might become slower with more resources. At the same time, the surge of big data applications increased the demand for distributed optimization methods, often requiring new properties such as ability to find a sparse solution. It is, therefore, of great importance to design new methods that are versatile, efficient and scale linearly with the amount of available resources. In fact, the applications vary a lot in their desiderata. There is a rising interest in federated learning~\cite{konevcny2016federated}, where the main concerns include the communication cost and ability to use  local data only in an attempt to provide a certain level of privacy. In high-dimensional machine learning problems, non-smooth $\ell_1$-penalty is often utilized, so one wants to have a support for proximable regularization. The efficiency of deep learning, in contrast, is dependent on heavy-ball momentum and nonconvex convergence to criticality, while sampling from the full dataset might not be an issue. In our work, we try to address all of these questions.

{\bf Communication as the bottleneck.}
The  key aspects of distributed optimization efficiency are computational and communication complexity. In general, evaluating full gradients is intractable due to time and memory restrictions, so computation is made cheap by employing stochastic updates. On the other hand, in typical distributed computing architectures, communication is much slower (see Figure~\ref{fig:communication} for our experiments with communication cost of aggregating and broadcasting) than a stochastic update, and the design of a training algorithm needs to find a trade-off between them. There have been considered several ways of dealing with this issue. One of the early approaches is to have each worker perform a block descent step, which leads to the {\tt Hydra} family of methods \cite{Hydra, Hydra2}. By choosing the size of the block, one directly chooses the amount of data that needs to be communicated. An alternative idea is for each worker to do more work between communication rounds (e.g., by employing a more powerful local solver, such as a second order method), so that computation roughly balances out with communication. The key methods in this sphere include {\tt CoCoA} and its variants 
\cite{cocoa, cocoa+, cocoa+journal, AccCOCOA, cocoa-2018-JMLR}, {\tt DANE} \cite{DANE}, {\tt DiSCO} \cite{DISCO,ma2015partitioning}, {\tt DANCE} \cite{jahani2018efficient} and {\tt AIDE} \cite{reddi2016aide}.

{\bf Update compression via randomized sparsification and/or quantization.} Practitioners suggested a number of heuristics to find a remedy for the communication botlleneck. Of special interest to this paper is the idea of compressing SGD updates, proposed in~\cite{seide20141}. Building off of this work, \cite{alistarh2017qsgd} designed a variant of SGD that guarantees convergence with compressed updates that they call {\tt QSGD}. Other works with SGD update structure include~\cite{konecny2016randomized, pmlr-v80-bernstein18a, khirirat2018distributed}. Despite proving a convergence rate, \cite{alistarh2017qsgd} also left many new questions open and introduced an additional, unexplained, heuristic of quantizing only vector blocks. Moreover, their analysis implicitly makes an assumption that all data should be available to each worker, which is hard and sometimes even impossible to satisfy. In a concurrent with~\cite{alistarh2017qsgd} work~\cite{wen2017terngrad}, the {\tt Terngrad} method was analyzed for stochastic updates that in expectation have positive correlation with the vector pointing to the solution. While giving more intuition about convergence of quantized methods, this work used $\ell_{\infty}$ norm for quantization, unlike $\ell_2$-quantization of~\cite{alistarh2017qsgd}.

{\bf The problem.} Let $f_i:\R^d\to \R$ is a loss of model $x$ obtained on data points belonging to distribution $\cD_i$, i.e., 
$
    f_i(x) \eqdef \EE_{\zeta\sim \cD_i} \phi(x, \zeta).
$ and $R:\R^d\to \R\cup \{+\infty\}$ be a proper closed convex regularizer.   In this paper we focus on the problem of training a machine learning model via regularized empirical risk minimization:
\begin{equation} \label{eq:main}
\textstyle    \min_{x\in \R^d} f(x) + R(x) \eqdef \frac{1}{n}\sum \limits_{i=1}^n f_i(x) + R(x).
\end{equation}
We do not assume any kind of similarity between distributions $\cD_1,\dotsc, \cD_n$. 

{\bf Notation.}  By $\sign (t)$ we denote the sign of $t\in \R$ (-1 if $t<0$, 0 if $t=0$ and $1$ if $t>0$). The $j$-th element of a vector $x\in \R^d$ is denoted as $x_{(j)}$. For $x = (x_{(1)},\dots,x_{(d)})\in \R^d$ and $p\geq 1$, we let $\|x\|_p = (\sum_i |x_{(i)}|^p)^{1/p}$. Note that $\|x\|_1 \geq \|x\|_p \geq \|x\|_\infty$ for all $x$. By $\|x\|_0$ we denote the number of nonzero elements of $x$. Detailed description of the notation is in Table~\ref{tbl:notation-table} in the appendix.  

\section{Contributions}
\begin{table}
\caption{Comparison of {\tt DIANA} and related methods. Here ``lin. rate'' means that  linear convergence either to a ball around the optimum or to the optimum was proved, ``loc. data'' describes whether or not authors assume that $f_i$ is available at node $i$ only, ``non-smooth'' means support for a non-smooth regularizer, ``momentum'' says whether or not authors consider momentum in their algorithm,  and ``block quant.'' means theoretical justification for using block quantization.
\label{tbl:comparison}
}
\begin{center}
\small
\begin{tabular}{|c|c|c|c|c|c|}
\hline
method & lin. rate & loc. data & non-smooth & momentum & block quant.\\
\hline
{\tt DIANA} (New!) & \cmark & \cmark & \cmark & \cmark & \cmark\\
\hline
{\tt QSGD}  \cite{alistarh2017qsgd}  & \xmark & \xmark & \xmark & \xmark & \xmark\\
\hline
{\tt TernGrad} \cite{wen2017terngrad} & \xmark & \xmark & \xmark & \xmark & \xmark\\
\hline
{\tt DQGD} \cite{khirirat2018distributed}  & \cmark & \cmark & \xmark & \xmark & \xmark\\
\hline
{\tt QSVRG} \cite{alistarh2017qsgd}  & \cmark & \cmark & \xmark & \xmark & \xmark\\
\hline
\end{tabular}
\end{center}
\end{table}

\begin{itemize}[leftmargin=*]
	\item[$\diamond$] \textbf{DIANA.} We develop a  distributed gradient-type method with compression of {\em gradient differences}, which we call {\tt DIANA} (Algorithm~\ref{alg:distributed1}). Unlike the gradients themselves, gradient differences eventually converge to 0 as the method approaches the optimum \revision{(and our analysis verifies this)}, since the gradient of worker $i$ is the fixed vector $\nabla f_i(x^*)$. Thus, we introduce a much smaller error when compressing the difference of two vectors that converge to $\nabla f_i(x^*)$ instead of compressing $\nabla f_i(x^*)$ itself. To make this possible, we introduced an extra vector $h_i^k$ in Algorithm~\ref{alg:distributed1} that \emph{learns} the gradient at the optimum\revision{, i.e., we design a special rule for updating these vectors that ensures their convergence to $\nabla f_i(x^*)$}. On top of that, we propose a momentum modification to make the method more practical.
 
  \item[$\diamond$] \textbf{Rates in the strongly convex and nonconvex cases.} We show that when applied to the smooth strongly convex minimization problem with arbitrary closed convex regularizer, {\tt DIANA} has the iteration complexity\revision{, i.e., number of iterations sufficient to guarantee that $\EE\|x^k - x^*\|_2^2 \leq \varepsilon + \varepsilon'$, where $\varepsilon'$ is the size of the neighborhood, which depends on the stepsize and variance}, $O\left(\max\left\{\sqrt{\nicefrac{d}{m}}, \kappa\left(1 + \nicefrac{1}{n}\sqrt{\nicefrac{d}{m}}\right)\right\}\ln\nicefrac{1}{\varepsilon}\right)$,  to a ball with center at the optimum (see Sec~\ref{sec:theory-strong-convex}, Thm~\ref{thm:DIANA-strongly_convex} and Cor~\ref{cor:DIANA-strong-convex} for the details). We also prove that {\tt DIANA} works for smooth nonconvex problems \revision{without regularization} and get the iteration complexity\revision{, i.e., number of iterations sufficient to guarantee that $\EE\|\nabla f(\overline{x}^k)\|_2^2 \leq \varepsilon$,} $O\left(\nicefrac{1}{\varepsilon^2}\max\left\{\nicefrac{L^2(f(x^0) - f^*)^2}{n^2\alpha_p^2}, \nicefrac{(\sigma^4\revision{+\zeta^4})}{(1+n\alpha_p)^2}\right\}\right)$ (see Sec~\ref{sec:DIANA-non-convex}, Thm~\ref{thm:DIANA-non-convex} and Cor~\ref{cor:DIANA-non-convex} for the details).
%  In the case of decreasing stepsize we show $O\left(\nicefrac{1}{\varepsilon}\right)$ iteration complexity (see Sec~\ref{sec:DIANA-decreasing-stepsizes}, Thm~\ref{th:str_cvx_decr_step} and Cor~\ref{cor:str_cvx_decr_step} for the details). Unlike  in \cite{khirirat2018distributed}, in a noiseless regime our method converges to the exact optimum, and at a  linear rate.
  
%\item[$\diamond$] \textbf{Rate in the nonconvex case.}  We prove that {\tt DIANA} also works for smooth nonconvex problems with an indicator-like regularizer and get the iteration complexity $O\left(\nicefrac{1}{\varepsilon^2}\max\left\{\nicefrac{L^2(f(x^0) - f^*)^2}{n^2\alpha_p^2}, \nicefrac{\sigma^4}{(1+n\alpha_p)^2}\right\}\right)$ (see Sec~\ref{sec:DIANA-non-convex}, Thm~\ref{thm:DIANA-non-convex} and Cor~\ref{cor:DIANA-non-convex} for the details).

%\item[$\diamond$] \textbf{DIANA with momentum.} We study momentum version of {\tt DIANA} for the case of smooth nonconvex objective with constant regularizer and $f_i = f$ (see Sec~\ref{sec:DIANA-momentum}, Thm~\ref{thm:DIANA-momentum} and Cor~\ref{cor:DIANA-momentum} for the details). We summarize a few key features of our complexity results established  in Table~\ref{tbl:comparison}.

\item[$\diamond$] \textbf{New theory for Terngrad and QSGD.} We provide first convergence rate of {\tt TernGrad} and provide new tight analysis of 1-bit {\tt QSGD} under less restrictive assumptions for both smooth strongly convex objectives with arbitrary closed convex regularizer and nonconvex objective with indicator-like regularizer (see Sec~\ref{sec:Alg} for the detailed comparison). We also study their momentum version for the case of smooth nonconvex objective with constant regularizer and $f_i = f$ (see Sec~\ref{sec:TernGrad-momentum}, Thm~\ref{thm:TernGrad-momentum} and Cor~\ref{cor:TernGrad-momentum}).

%\item[$\diamond$] \textbf{Optimal norm power.} We find the answer for the following question: \textit{which $\ell_p$ norm to use for quantization in order to get the best iteration complexity of the algorithm?} It is easy to see that all the bounds that we propose depend on $\nicefrac{1}{\alpha_p}$ where $\alpha_p$ is an increasing function of $1\le p \le \infty$ (see Lemma~\ref{lema:alpha_p} for the details). That is, \textit{for both Algorithm~\ref{alg:distributed1}~and~\ref{alg:terngrad} the iteration complexity reduces when $p$ is growing and the best iteration complexity for our algorithms is achieved for $p = \infty$.} This implies that {\tt TernGrad} has better iteration complexity than 1-bit {\tt QSGD}.

%\item[$\diamond$]\textbf{First analysis for block-quantization.}	  We give a first analysis of block-quantization (i.e.\ bucket-quantization), which was mentioned in \cite{alistarh2017qsgd} as a useful heuristic.

\end{itemize}

\section{The Algorithm} \label{sec:Alg}

\begin{algorithm}[t]
   \caption{{\tt DIANA} ($n$ nodes)}
   \label{alg:distributed1}
\begin{algorithmic}[1]
   \INPUT learning rates $\alpha>0$ and $\{\gamma^k\}_{k\geq 0}$, initial vectors $x^0, h_1^0,\dotsc, h_n^0 \in \R^d$ and $h^0 = \frac{1}{n}\sum_{i=1}^n h_i^0$, quantization parameter $p \geq 1$, sizes of blocks $\{d_l\}_{l=1}^m$, momentum parameter $0\le \beta < 1$
   \STATE $v^0 = \nabla f(x^0)$
   \FOR{$k=0,1,\dotsc$}
	   \STATE Broadcast $x^{k}$ to all workers
        \FOR{$i=1,\dotsc,n$ in parallel}
	        \STATE Sample $g^{k}_i$ such that $\EE [g^k_i \;|\; x^k]  =\nabla f_i(x^k)$ and let $\Delta^k_i = g^k_i - h^k_i$
\STATE\label{line:quant} Sample $\hat \Delta^k_i \sim {\rm Quant}_p(\Delta^k_i,\{d_l\}_{l=1}^m)$ and let $h_i^{k+1} = h_i^k + \alpha \hat \Delta_i^k$ and  $\hat g_i^k = h_i^k + \hat \Delta_i^k$
        \ENDFOR
%        \STATE Communicate $\|g_i^k - h_i^k\|_2$ and $s_i^k$ from node $i$
%        \STATE $\hat g_i^k = h_i^k + \|g_i^k - h_i^k\|_2 s_i^k$, $i=1,\dotsc,n$
       \STATE $\hat \Delta^k = \frac{1}{n}\sum_{i=1}^n \hat \Delta_i^k$; \quad \revision{$\hat g^k = h^k + \hat \Delta^k$}; \quad $v^k = \beta v^{k-1} + \hat g^k$
%        \If{Option a} \Comment{E.g.\ if $\ell_1$ penalty is used}
%        		\STATE $\hat g^k = \mean g^k$ 
%        \Else{ (Option b)}
%        		\STATE Sample $[\xi^k]_j\sim \mathrm{Be}\left( \frac{|[\mean g^k]_j - [h^k]_j|}{\|\mean g^k - h^k\|_\infty} \right)$, $j=1,\dotsc, d$
%        		\STATE $s^k = \sign(\mean g^k - h^k) \circ \xi^k$
%        		\STATE $\hat g^k = h^k + \|\mean g^k - h^k\|_{2} s^k$
%        		\STATE $h^{k+1} = h^k + \beta\|\mean g^k - h^k\|_{2} s^k$
%        \EndIf
        \STATE $x^{k+1} = \proxkR\left(x^k - \gamma^kv^k \right)$; \quad  $h^{k+1}  = \tfrac{1}{n}\sum_{i=1}^n h_i^{k+1} = h^k + \alpha \hat \Delta^k$
   \ENDFOR
\end{algorithmic}
\end{algorithm}

In this section we describe our main method---{\tt DIANA}. However, we first need to introduce several key concepts and ingredients that come together to make the algorithm.  In each iteration $k$ of {\tt DIANA}, each node will sample an unbiased estimator of the local gradient \revision{(see line~5 of Algorithm~\ref{alg:distributed1})}.  We assume that these gradients have bounded variance.

\begin{assumption}[Stochastic gradients]\label{as:noise}
For every $i = 1,2,\dots,n$,
	$\EE [g_i^k \;|\; x^k] =  \nabla f_i(x^k)$. Moreover, the variance is bounded:
	\begin{align}\label{eq:bounded_noise}
		\EE \|g_i^k - \nabla f_i(x^k)\|_2^2 \le \sigma_i^2.
	\end{align}

\end{assumption}

Note that $g^k \eqdef \tfrac{1}{n}\sumin g_i^k$ is an unbiased estimator of $\nabla f(x^k)$: \begin{equation}\label{eq:hat_g_expectation}\textstyle \EE[g^k \;|\; x^k] = \tfrac{1}{n} \sum \limits_{i=1}^n \nabla f_i(x^k) = \nabla f(x^k).\end{equation}
Let $\sigma^2 \eqdef \tfrac{1}{n}\sumin \sigma_i^2$. By independence of the random vectors $\{g_i^k - \nabla f_i (x^k)\}_{i=1}^n$, its variance is bounded above by
\begin{equation}\label{eq:bgud7t9gf}\EE \left[ \|g^k -\nabla f(x^k)\|_2^2 \; | \; x^k \right] \leq \tfrac{\sigma^2}{n}.\end{equation}

{\bf Quantization.} {\tt DIANA} applies random compression (quantization) to gradient differences \revision{(see line~6 of Algorithm~\ref{alg:distributed1})}, which are then communicated  to a parameter server.  We now define the random quantization transformations used. Our first quantization operator transforms a vector $\Delta \in \R^d$ into a random vector $\hat{\Delta} \in \R^d$ whose entries belong to the set $\{-t,0,t\}$ for some $t>0$.

% The transformation  depends on parameters $p\geq 1$ (norm parameter, specifying the $\ell_p$ norm on $\R^d$) and $m>0$ (scaling parameter) satisfying $m\|v\|_{\infty} \leq  \|v\|_q$. 

\begin{definition}[$p$-quantization]\label{def:p-quant}  Let $\Delta \in \R^d$ and let $p\geq 1$. If $\Delta=0$, we define $\widetilde{\Delta}=\Delta$. If $\Delta \neq 0$, we define $\widetilde{\Delta}$ by setting 
\begin{equation}\label{eq:quant-j}\widetilde{\Delta}_{(j)} = \|\Delta\|_p \sign(\Delta_{(j)}) \xi_{(j)}, \quad j=1,2,\dots,d,\end{equation}  where
 $\xi_{(j)}\sim {\rm Be}\left(|\Delta_{(j)}|/\|\Delta\|_p\right)$  are Bernoulli random variables\footnote{That is, $\xi_{(j)} = 1$ with probability $|\Delta_{(j)}|/\|\Delta\|_p$ (observe that this quantity is always upper bounded by 1) and $\xi_{(j)} = 0$ with probability $1-|\Delta_{(j)}|/\|\Delta\|_p$. }. Note that \begin{equation}\label{eq:quant}\widetilde{\Delta} = \|\Delta\|_p \; \sign(\Delta) \circ \xi,\end{equation} where $\sign$ is applied elementwise, and $\circ$ denotes the Hadamard (i.e.\ elementwise) product.  We say that $\widetilde{\Delta}$ is $p$-quantization of $\Delta$. When sampling $\widetilde{\Delta}$, we shall write $\widetilde{\Delta} \sim {\rm Quant}_{p}(\Delta)$.  
\end{definition}

In addition, we consider a block variant of $p$-quantization operators. These are defined, and their properties \revision{are studied} in \revision{Section~\ref{appendix:block_quant}} in the appendix.

{\bf Communication cost.} If $b$ bits are used to encode a float number, then at most $C(\hat \Delta) \eqdef \|\hat \Delta\|_0^{1/2}(\log \|\hat \Delta\|_0 + \log 2 + 1) + b$ bits are needed to communicate $\hat \Delta$ with Elias coding (see Theorem~3.3 in \cite{alistarh2017qsgd}). In our next result, we given an upper bound on the expected communication cost.

\begin{theorem}[Expected sparsity]\label{th:quantization_quality1} Let $0\neq \Delta\in \R^{\tilde d}$ and ${\widetilde \Delta} \sim {\rm Quant}_{p}(\Delta)$ be its $p$-quantization. Then  \begin{equation} \label{eq:expected_comm_cost} \EE \|\widetilde \Delta\|_0 =  \tfrac{\|\Delta\|_1}{\|\Delta\|_{p}} \leq \|\Delta\|_0^{1-1/p} \leq \tilde{d}^{1-1/p},\end{equation}
 \begin{equation} \label{eq:expected_comm_cost2} C_p \eqdef \EE C(\widetilde \Delta) \leq \tfrac{\|\Delta\|_1^{1/2}}{\|\Delta\|_{p}^{1/2}} (\log \tilde{d} + \log 2 + 1) + b.\end{equation}
 All expressions in \eqref{eq:expected_comm_cost} and  \eqref{eq:expected_comm_cost2} are  increasing functions of $p$. 

\end{theorem}

{\bf Proximal step.} Given $\gamma>0$, the proximal operator for the regularizer $R$ is defined as
$
	\proxR(u) \eqdef \arg\min_v \left\{ \gamma R(v) + \tfrac{1}{2}\|v - u\|_2^2 \right\}.
$ The proximal operator of a closed convex function is nonexpansive. That is, for any $\gamma>0$ and $u,v\in \R^d$, 
    \begin{align}
        \left\|\proxR(u) - \proxR(v)\right\|_2 \le \|u - v\|_2. \label{eq:nonexpansive} 
    \end{align}
    
{\bf DIANA.}    In {\tt DIANA}, each machine $i\in \{1,2,\dots,n\}$ first computes a stochastic gradient $g_i^k$ at current iterate $x^k$. We {\em do not} quantize this information and send it off to the parameter server as that approach would not converge for $R\neq 0$. Instead, we maintain memory $h^k_i$ at each node $i$ (initialized to arbitrary values), and {\em quantize the difference $\delta_i^k\eqdef g_i^k-h_i^k$ instead} \revision{(see line~6 of Algorithm~\ref{alg:distributed1})}. Both the node and the parameter server update $h_i^k$ in an appropriate manner, and a proximal gradient descent step is taken with respect to direction $v^k = \beta v^{k-1} + \hat{g}^k$ \revision{(see line~9 of Algorithm~\ref{alg:distributed1})}, where $0\leq \beta\leq 1$ is a {\em momentum parameter}, whereas $\hat{g}^k$ is an unbiased estimator of the full gradient, assembled from the memory $h_i^k$ and the transmitted quantized vectors. Note that we allows for block quantization for more flexibility. In practice, we want the transmitted quantized vectors to be much easier to communicate than the full dimensional vector in $\R^d$, which can be tuned by the choice of $p$ defining the quantization norm, and the choice of blocks.
    
{\bf Relation to {\tt QSGD} and {\tt TernGrad}.} If the initialization is done with $h^0=0$ and $\alpha=0$, our method reduces to either 1-bit {\tt QSGD} or {\tt TernGrad} with $p=2$ and $p=\infty$ respectively. \revision{These methods apply the compression directly to the stochastic gradients, leading to the larger noise coming from compression.} We unify \revision{these algorithms} in the Algorithm~\ref{alg:terngrad}. We analyse this algorithm (i.e.\ {\tt DIANA} with $\alpha = 0$ and $h_i^0 = 0$) in three cases: i) smooth strongly convex objective with arbitrary closed convex regularizer; ii) smooth nonconvex objective with constant regularizer; iii) smooth nonconvex objective with constant regularizer for the momentum version of the algorithm. We notice, that in the original paper \cite{wen2017terngrad} authors do not provide the rate of convergence for {\tt TernGrad} and we get the convergence rate for the three aforementioned situations as a special case of our results. Moreover, we emphasize that our analysis is new even for 1-bit {\tt QSGD}, since in the original paper \cite{alistarh2017qsgd} authors consider only the case of bounded gradients ($\EE\|g^k\|_2^2 \le B^2$), which is very restrictive assumption, and they do not provide rigorous analysis of block-quantization as we do. In contrast, we consider more general case of block-quantization and assume only that the variance of the stochastic gradients is bounded, which is less restrictive assumption since the inequality $\EE\|g^k\|_2^2 \le B^2$ implies  $\EE\|g^k - \nabla f(x^k)\|_2^2 \le \EE\|g^k\|_2^2 \le B^2$.

We obtain the convergence rate for arbitrary $p\ge 1$ for the three aforementioned cases (see Theorems~\ref{thm:TernGrad-nonconvex},~\ref{thm:TernGrad-momentum},~\ref{thm:terngrad_strg_cvx_prox},~\ref{thm:TernGrad-decreasing-stepsizes} and Corollaries~\ref{cor:TernGrad-nonconvex},~\ref{cor:TernGrad-momentum},~\ref{cor:TernGrad-decreasing-stepsizes} for the details) and all obtained bounds becomes better when $p$ is growing, which means that {\tt TernGrad} has better iteration complexity than {\tt QSGD} and, more generally, the best iteration complexity attains for $\ell_\infty$ norm quantization. 

%\konstantin{Therefore, smbd pessimistic can set $\alpha$ to be a tiny number and get convergence not worse than that of {\tt TernGrad}.}

\section{Theory: Strongly Convex Case} \label{sec:theory-strong-convex}
\begin{table}[t]
    \centering
\caption{Summary of iteration complexity results. }
    \label{tab:results}    
    \footnotesize
    \begin{tabular}{|c|c|c|c|c|c|c|c|}
    \hline 
    Block quant. & Loc. data & Nonconvex & Strongly Convex & $R$ & Momentum & $\alpha>0$ & Theorem\\
	\hline 
	\hline
	\cmark & \cmark & \cmark & \xmark & \xmark & \xmark & \cmark &  \ref{thm:DIANA-non-convex} \\
	\hline
	\cmark & \cmark & \cmark & \xmark & \xmark & \cmark & \cmark &  \ref{thm:DIANA-momentum}\\
	\hline
	\cmark & \cmark & \xmark & \cmark & \cmark & \xmark & \cmark &  \ref{thm:DIANA-strongly_convex}, \ref{th:str_cvx_decr_step} \\
	\hline 
	\cmark & \cmark & \cmark & \xmark & \xmark & \xmark & \xmark &  \ref{thm:TernGrad-nonconvex} \\
	\hline
	\cmark & \cmark & \cmark & \xmark & \xmark & \cmark & \xmark &  \ref{thm:TernGrad-momentum}\\
	\hline
	\cmark & \cmark & \xmark & \cmark & \xmark & \xmark & \xmark &  \ref{thm:terngrad_strg_cvx_prox}, \ref{thm:TernGrad-decreasing-stepsizes} \\
	\hline 
    \end{tabular}
\end{table}

Let us introduce two key assumptions of this section.

\begin{assumption}[$L$--smoothness]
    We say that a function $f$ is $L$-smooth if
    \begin{align}
        f(x) \le f(y) + \< \nabla f(y), x - y> + \tfrac{L}{2}\|x - y\|_2^2, \quad \forall x, y. \label{eq:smoothness_functional}
    \end{align}
\end{assumption}

\begin{assumption}[$\mu$-strong convexity]
	 $f$ is $\mu$-strongly convex, i.e., 
	\begin{align}
		f(x) \ge f(y) + \< \nabla f(y), x - y> + \tfrac{\mu}{2}\|x - y\|_2^2, \quad \forall x,y. \label{eq:strong_cvx_functional}
	\end{align}
\end{assumption}

For $1\leq p \leq +\infty$, define
\begin{equation}\label{eq:alpha_p} \alpha_p(d) \eqdef \inf_{x\neq 0,x\in\R^d} \tfrac{\|x\|_2^2}{\|x\|_1 \|x\|_p}. \end{equation}

\begin{lemma} \label{lema:alpha_p} $\alpha_p$ is  increasing as a function of $p$ and decreasing as a function of $d$. In particular,  $\alpha_1 \leq \alpha_2 \leq \alpha_\infty$, and moreover,
$\alpha_1(d) = \nicefrac{1}{d}$, $\alpha_2(d) = \nicefrac{1}{\sqrt{d}}$, $\alpha_{\infty}(d) = \nicefrac{2}{(1+\sqrt{d})}$
and, as a consequence, for all positive integers $\widetilde{d}$ and $d$ the following relations holds 
$\alpha_1(\widetilde{d}) = \nicefrac{\alpha_1(d) d}{\widetilde{d}}$, $\alpha_2(\widetilde{d}) = \alpha_2(d)\sqrt{\nicefrac{d}{\widetilde{d}}},$ and
$\alpha_\infty(\widetilde{d}) = \nicefrac{\alpha_\infty(d)(1+\sqrt{d})}{(1 + \sqrt{\widetilde{d}})}.$
\end{lemma}
To study the method in the strongly convex case, we define the Lyapunov function
\begin{equation}\label{eq:strong_convex_Lyapunov} \textstyle V^k \eqdef \|x^{k} - x^*\|_2^2 + \tfrac{c\gamma^2}{n}\sum \limits_{i=1}^n \|h_i^{k} - h_i^*\|_2^2, \end{equation}
where $x^*$ is the solution of \eqref{eq:main} and $h^* \eqdef \nabla f(x^*)$. Notice that it consists of two main terms: the squared distance to the solution $\|x^{k} - x^*\|_2^2$, and the sum of all errors of the gradient estimates $\sum \limits_{i=1}^n \|h_i^{k} - h_i^*\|_2^2$. Both of these two terms approach 0 as we run the algorithm, and they help each other, as $x^k$ approaching $x^*$ implies that $h_i^k$ should be approaching $\nabla f_i(x^*)$.

\begin{theorem} \label{thm:DIANA-strongly_convex}  Assume the functions $f_1,\dots,f_n$ are $L$--smooth and $\mu$--strongly convex. Choose stepsizes $\alpha>0$ and $\gamma^k=\gamma>0$, block sizes $\{d_l\}_{l=1}^m$, where $\widetilde{d} = \max\limits_{l=1,\ldots,m}d_l$, and parameter $c>0$ satisfying the following relations:
\begin{equation}\label{eq:cond1} \tfrac{1 + n c \alpha^2}{1 + n c \alpha}   \leq \alpha_p \eqdef \alpha_p(\widetilde{d}),\end{equation}
\begin{equation}\label{eq:cond2}\gamma \leq \min\left\{\tfrac{\alpha}{\mu}, \tfrac{2}{(\mu+L)(1+c \alpha)} \right\}.\end{equation}
Then for all $k\geq 0$,
\begin{equation} \label{eq:strong_convex_rate} \EE V^k\le (1 - \gamma\mu)^k V^0 + \tfrac{\gamma}{\mu}(1+nc\alpha)\tfrac{\sigma^2}{n}. \end{equation}
This implies that as long as $k \geq \frac{1}{\gamma \mu} \log \frac{V^0}{\epsilon}$ \revision{for given $\epsilon > 0$}, we have $\EE V^k \leq \epsilon + \frac{\gamma}{\mu}(1+nc\alpha)\frac{\sigma^2}{n}.$
\end{theorem}
Notice that for each worker $i$, our Lyapunov function includes the term $\|h_i^k - h_i^*\|_2$, which goes to 0 together with the Lyapunov function. This implies that the compressed gradient differences allow us to eventually learn the gradient at the optimum $h_i^*$. This is a key distinction of our work from the prior literature.

\textbf{Proof sketch.} The proof consists of two main parts. Firstly, we provide a recursion for $\|x^{k+1} - x^*\|_2^2$ to  $\|x^{k} - x^*\|_2^2$ that shows that the error depends on the error of the gradient estimates, $\|h_i^k - h_i^*\|_2^2$. Next, we use our quantization lemmas to show how $\|h_i^{k+1} - h_i^*\|_2^2$ can be recursed to $\|h_i^k - h_i^*\|_2^2$ with an extra variance term $\EE[\|g_i^k - \nabla f_i(x^k)\|_2^2]$. Combining the recursions and unrolling them to the initial iterates, we can get the theorem's claim.

Moreover, we can set $\gamma$ to be equal to the minimum in \eqref{eq:cond2}, in which case the leading term in the iteration complexity bound is $\nicefrac{1}{\gamma \mu} = \max \left\{ \nicefrac{1}{\alpha}, \nicefrac{(\mu+L)(1+c\alpha)}{2\mu}\right\}.$ This is summarized in the following corollary.

\begin{corollary} \label{cor:DIANA-strong-convex} Let $\kappa = \nicefrac{L}{\mu}$, $\alpha = \nicefrac{\alpha_p}{2}$, $c = \nicefrac{4(1-\alpha_p)}{n\alpha_p^2}$, and $\gamma = \min\left\{\nicefrac{\alpha}{\mu}, \tfrac{2}{(L+\mu)(1+c \alpha)}\right\}$. Then the conditions \eqref{eq:cond1} and \eqref{eq:cond2} are satisfied, and the leading  iteration complexity term  is equal to
\begin{equation}\label{eq:bu987gd9} \nicefrac{1}{\gamma \mu} = \max\left\{\nicefrac{2}{\alpha_p}, (\kappa+1)\left(\nicefrac{1}{2} - \nicefrac{1}{n} + \nicefrac{1}{n\alpha_p}\right)\right\}.\end{equation} 
This is a decreasing function of $p$, and  hence $p=+\infty$ is the optimal choice.
\end{corollary}

In Table~\ref{tbl:complex_strong_conv} we calculate the leading term \eqref{eq:bu987gd9}  in the complexity of {\tt DIANA}  for $p\in \{1,2,+\infty\}$, each  for two condition number regimes: $n=\kappa$ (standard) and $n=\kappa^2$ (large).

\begin{table}
\begin{center}
\caption{The leading term of the iteration complexity of {\tt DIANA} in the strongly convex case (Thm~\ref{thm:DIANA-strongly_convex}, Cor~\ref{cor:DIANA-strong-convex} and Lem~\ref{lema:alpha_p}). Logarithmic dependence on $1/\epsilon$ is suppressed. Condition number: $\kappa\eqdef \nicefrac{L}{ \mu}.$
\label{tbl:complex_strong_conv}
}
\footnotesize
\begin{tabular}{|c|c|c|c|}
\hline
$p$ & iteration complexity & $\kappa=\Theta(n)$ & $\kappa=\Theta(n^2)$\\
\hline
1 &$\max\left\{\frac{2d}{m},(\kappa+1)A\right\}$; \quad  $A = \left(\frac{1}{2}-\frac{1}{n}+\frac{d}{nm}\right)$   & $O\left(n+\frac{d}{m}\right)$ & $O\left(n^2 + \frac{nd}{m}\right)$ \\
\hline
$2$ & $\max\left\{\frac{2\sqrt{d}}{\sqrt{m}},(\kappa+1)B\right\}$; \quad $B = \left(\frac{1}{2}-\frac{1}{n}+\frac{\sqrt{d}}{n\sqrt{m}}\right)$  & $O\left(n+\sqrt{\frac{d}{m}}\right)$ & $O\left(n^2 + \frac{n\sqrt{d}}{\sqrt{m}}\right)$ \\
\hline
$\infty$ & $\max\left\{1+\sqrt{\frac{d}{m}},(\kappa+1)C\right\}$; \quad $C = \left(\frac{1}{2}-\frac{1}{n}+\frac{1+\sqrt{\frac{d}{m}}}{2n}\right)$  & $O\left(n+\sqrt{\frac{d}{m}}\right)$ & $O\left(n^2 + \frac{n\sqrt{d}}{\sqrt{m}}\right)$ \\
\hline
\end{tabular}
\end{center}
\end{table}

{\bf Matching the rate of gradient descent for quadratic size models.} Note that as long as the model size is not too big; in particular, when $d = O(\min\{\kappa^2,n^2\}), $
the linear rate of {\tt DIANA} with $p\geq 2$ is $O(\kappa \log (1/\epsilon))$, which matches the rate of gradient descent.

{\bf Optimal block quantization.} If the dimension of the problem is large, it becomes reasonable to quantize vector's blocks, also called blocks. For example, if we had a vector which consists of 2 smaller blocks each of which is proportional to the vector of all ones, we can transmit just the blocks without any loss of information. In the real world, we have a similar situation when different parts of the parameter vector have different scale. A straightforward example is deep neural networks, layers of which have pairwise different scales. If we quantized the whole vector at once, we would zero most of the update for the layer with the smallest scale.

Our theory says that if we have $n$ workers, then the iteration complexity increase of quantization is about $\nicefrac{\sqrt{d}}{n}$. However, if quantization is applied to a block of size $n^2$, then this number becomes 1, implying that the complexity remains the same. Therefore, if one uses about 100 workers and splits the parameter vector into parts of size about 10,000, the algorithm will work as fast as SGD, while communicating bits instead of floats!

Some consideration related to the question of optimal number of nodes are included in Section~\ref{sec:opt_no_nodes}.

% Note that if $d \leq n^2$, then $W(n) = O(\kappa + \sqrt{d})$, and is independent of $n$. However, if $d \gg n^2$ (big model regime), and $\kappa \gg n$ (ill conditioned problem), then the dominant term in $W(n)$ is $\kappa \tfrac{\sqrt{d}-1}{n}$, and this term can be controlled by spreading the work across more nodes, i.e., by increasing $n$.

\subsection{Decreasing stepsizes}\label{sec:DIANA-decreasing-stepsizes}

We now provide a convergence result for {\tt DIANA} with decreasing step sizes, obtaining a $\cO(1/k)$ rate.

\begin{theorem}\label{th:str_cvx_decr_step}
    Assume that $f$ is $L$-smooth, $\mu$-strongly convex and we have access to its gradients with bounded noise. Set $\gamma^k = \frac{2}{\mu k + \theta}$ with some $\theta \ge 2\max\left\{\frac{\mu}{\alpha}, \frac{(\mu+L)(1+c\alpha)}{2} \right\}$ for some numbers $\alpha > 0$ and $c > 0$ satisfying $\frac{1+nc\alpha^2}{1+nc\alpha} \le \alpha_p$. After $k$ iterations of {\tt DIANA} we have
    \begin{align*}
        \EE V^k \le \tfrac{1}{\eta k+1}\max\left\{ V^0, 4\tfrac{(1+nc\alpha)\sigma^2}{n\theta\mu} \right\},
    \end{align*}
    where $\eta\eqdef \frac{\mu}{\theta}$, $V^k=\|x^k - x^*\|_2^2 + \frac{c\gamma^k}{n}\sumin\|h_i^0 - h_i^*\|_2^2$ and $\sigma$ is the standard deviation of the gradient noise.
\end{theorem}

\begin{corollary}\label{cor:str_cvx_decr_step}
	If we choose $\alpha = \frac{\alpha_p}{2}$, $c = \frac{4(1-\alpha_p)}{n\alpha_p^2}$,  $\theta=2\max\left\{\frac{\mu}{\alpha}, \frac{\left(\mu+L\right)\left(1 + c\alpha\right)}{2} \right\} = \frac{\mu}{\alpha_p}\max\left\{4, \frac{2(\kappa + 1)}{n} + \frac{(\kappa+1)(n-2)}{ n}\alpha_p\right\}$, then there are three regimes:
i) if $1 = \max\left\{1,\frac{\kappa}{n},\kappa\alpha_p\right\}$, then $\theta = \Theta\left(\frac{\mu}{\alpha_p}\right)$ and to achieve $\EE V^k\le \varepsilon$ we need at most $O\left( \frac{1}{\alpha_p}\max\left\{V^0, \frac{(1-\alpha_p)\sigma^2}{n\mu^2} \right\}\frac{1}{\varepsilon} \right)$ iterations;
ii) if $\frac{\kappa}{n} = \max\left\{1,\frac{\kappa}{n},\kappa\alpha_p\right\}$, then $\theta = \Theta\left(\frac{L}{n\alpha_p}\right)$ and to achieve $\EE V^k\le \varepsilon$ we need at most $O\left( \frac{\kappa}{n\alpha_p}\max\left\{V^0, \frac{(1-\alpha_p)\sigma^2}{\mu L} \right\}\frac{1}{\varepsilon} \right)$ iterations;
iii) if $\kappa\alpha_p = \max\left\{1,\frac{\kappa}{n},\kappa\alpha_p\right\}$, then $\theta = \Theta\left(L\right)$ and to achieve $\EE V^k\le \varepsilon$ we need at most $O\left( \kappa\max\left\{V^0, \frac{(1-\alpha_p)\sigma^2}{\mu Ln\alpha_p} \right\}\frac{1}{\varepsilon} \right)$ iterations.		
\end{corollary}

\section{Theory: Nonconvex Case}\label{sec:DIANA-non-convex}

In this section we consider the nonconvex case under the following assumption which we call \textit{bounded data dissimilarity}. Note that this assumption is necessary since the solution is not unique, so we can no longer upper bound gradient differences using the gradient at the optimum.

\begin{assumption}[Bounded data dissimilarity]\label{as:almost_identical data}
	We assume that there exists a constant $\zeta \ge 0$ such that for all $x\in\R^d$
	\begin{equation}\label{eq:almost_identical_data}
		\tfrac{1}{n}\sum\nolimits_{i=1}^n\|\nabla f_i(x) - \nabla f(x)\|_2^2 \le \zeta^2.
	\end{equation}
\end{assumption}
In particular, Assumption~\ref{as:almost_identical data} holds with $\zeta = 0$ when all $f_i$'s are the same up to some additive constant (i.e.\ each worker samples from one dataset). We note that it is also possible to extend our analysis to a more general assumption with extra $O(\|\nabla f(x)\|_2^2)$ term in the right-hand side of~\eqref{eq:almost_identical_data}. However, this would overcomplicate the theory without providing more insight.

\begin{theorem}\label{thm:DIANA-non-convex}
%    Assume $R$ is such that exists a closed convex set $\cX$ satisfying 1) $\forall z\in\R^n$ $\proxR(z)\in\cX$ and 2) $\forall z\in\cX$ $z = \proxR(z)$ (e.g. indicator function of $\cX$).   %the indicator function of a closed convex set $\cX$, 
%    Also assume that $h^*=0$,  $f$ is $L$-smooth, stepsizes $\alpha>0$ and $\gamma^k=\gamma>0$ and parameter $c>0$ satisfying $\frac{1 + n c \alpha^2}{1 + n c \alpha}   \leq \alpha_p,$ $\gamma \leq \frac{2}{L(1+c\alpha)}$ and $\overline x^k$ is chosen randomly from $\{x^1,\dotsc, x^k \}$. If, further, every worker samples from the full dataset, then
%    \begin{align*}
%        \EE \|\nabla f(\overline x^k)\|_2^2 &\le \frac{2}{k}\frac{f(x^0) - f^* + c \gamma^2\|h^0\|_2^2}{\gamma\left(2 - L\gamma - c\alpha L \gamma\right)}\\
%        &\quad + (1+cn\alpha)\frac{L\gamma}{2 - L\gamma - c\alpha L \gamma}\frac{\sigma^2}{n}.
%    \end{align*}
Assume that $R$ is constant and Assumption~\ref{as:almost_identical data} holds.    
    Also assume that $f$ is $L$-smooth, stepsizes $\alpha>0$ and $\gamma^k=\gamma>0$ and parameter $c>0$ satisfying $\frac{1 + n c \alpha^2}{1 + n c \alpha}   \leq \alpha_p,$ $\gamma \le \frac{2}{L(1+2c\alpha)}$ and $\overline x^k$ is chosen randomly from $\{x^0,\dotsc, x^{k-1} \}$. Then
    \begin{align*}
        \EE \|\nabla f(\overline x^k)\|_2^2 &\le \frac{2}{k}\frac{\Lambda^0}{\gamma(2 - L\gamma - 2c\alpha L \gamma)}+ \frac{(1+2cn\alpha)L\gamma}{2 - L\gamma - 2c\alpha L \gamma}\frac{\sigma^2}{n} + \frac{4c\alpha L\gamma \zeta^2}{2-L\gamma -2c\alpha L\gamma},
    \end{align*}
    where $\Lambda^k\eqdef  f(x^k) - f^* + c\tfrac{L\gamma^2}{2}\frac{1}{n}\sumin \|h_i^{k}-h_i^*\|_2^2$.
\end{theorem}

\begin{corollary}\label{cor:DIANA-non-convex}
	%Set $\alpha = \frac{\alpha_p}{2}$, $c = \frac{4(1-\alpha_p)}{n\alpha_p^2}$, $\gamma = \frac{n\alpha_p}{L(2 + (n-2)\alpha_p)\sqrt{K}}$, $h^0 = 0$ and run the algorithm for $K$ iterations. Then, the final accuracy is at most $\frac{2}{\sqrt{K}} \frac{L(2+(n-2)\alpha_p)}{n\alpha_p} (f(x^0) - f^*) + \frac{1}{\sqrt{K}}\frac{(2-\alpha_p)\sigma^2}{2+(n-2)\alpha_p}$.
	Set $\alpha = \frac{\alpha_p}{2}$, $c = \frac{4(1-\alpha_p)}{n\alpha_p^2}$, $\gamma = \frac{n\alpha_p}{L(4 + (n-4)\alpha_p)\sqrt{K}}$, $h^0 = 0$ and run the algorithm for $K$ iterations. Then, the final accuracy is at most $\frac{2}{\sqrt{K}} \frac{L(4+(n-4)\alpha_p)}{n\alpha_p} \Lambda^0 + \frac{1}{\sqrt{K}}\frac{(4-3\alpha_p)\sigma^2}{4+(n-4)\alpha_p} + \frac{8(1-\alpha_p)\zeta^2}{(4+(n-4)\alpha_p)\sqrt{K}}$.
\end{corollary}

Moreover, if the first term in Corollary~\ref{cor:DIANA-non-convex} is leading and $\nicefrac{1}{n} = \Omega(\alpha_p)$, the resulting complexity is $O(\nicefrac{1}{\sqrt{K}})$, i.e.\ the same as that of {\tt SGD}. For instance, if sufficiently large mini-batches are used, the former condition holds, while for the latter it is enough to quantize vectors in blocks of size $O(n^2)$.

\section{Implementation and Experiments}
Following advice from~\cite{alistarh2017qsgd}, we encourage the use of \textit{blocks} when quantizing large vectors. To this effect, a vector can decomposed into a number of blocks, each of which should  then be quantized separately. If coordinates have different scales, as is the case in deep learning, it will prevent undersampling of those with typically smaller values. Moreover, our theoretical results predict that applying quantization to blocks or layers will result in superlinear acceleration.

In our convex experiments, the optimal values of $\alpha$ were usually around $\min_i\nicefrac{1}{\sqrt{d_i}}$, where the minimum is taken with respect to blocks and $d_i$ are their sizes.

Finally, higher mini-batch sizes make the sampled gradients less noisy, which in turn is favorable to more uniform differences $g_i^k - h_i^k$ and faster convergence.

Detailed description of the experiments can be found in Section~\ref{sec:A:detailsOfNumericalExperiments} as well as extra numerical results.

{\bf {\tt DIANA} with momentum works best.} We implement {\tt DIANA}, {\tt QSGD},  {\tt TernGrad} and {\tt DQGD} in Python\footnote{The code will be made available online upon acceptance of this work.} using MPI4PY for processes communication. This is then tested on a machine with 24 cores, each is Intel(R) Xeon(R) Gold 6146 CPU @ 3.20GHz. The problem considered is binary classification with logistic loss and $\ell_2$ penalty, chosen to be of order $1/N$, where $N$ is the total number of data points. We experiment with choices of $\alpha$, choice of norm type $p$, different number of workers and search for optimal block sizes. $h_i^0$ is always set to be zero vector for all $i$. We observe that for $\ell_{\infty}$-norm the optimal block size is significantly bigger than for $\ell_2$-norm. Here, however, we provide Figure~\ref{fig:diana_main} to show how vast the difference is with other methods.
\begin{figure}[h]
\centering

\includegraphics[scale=0.3]{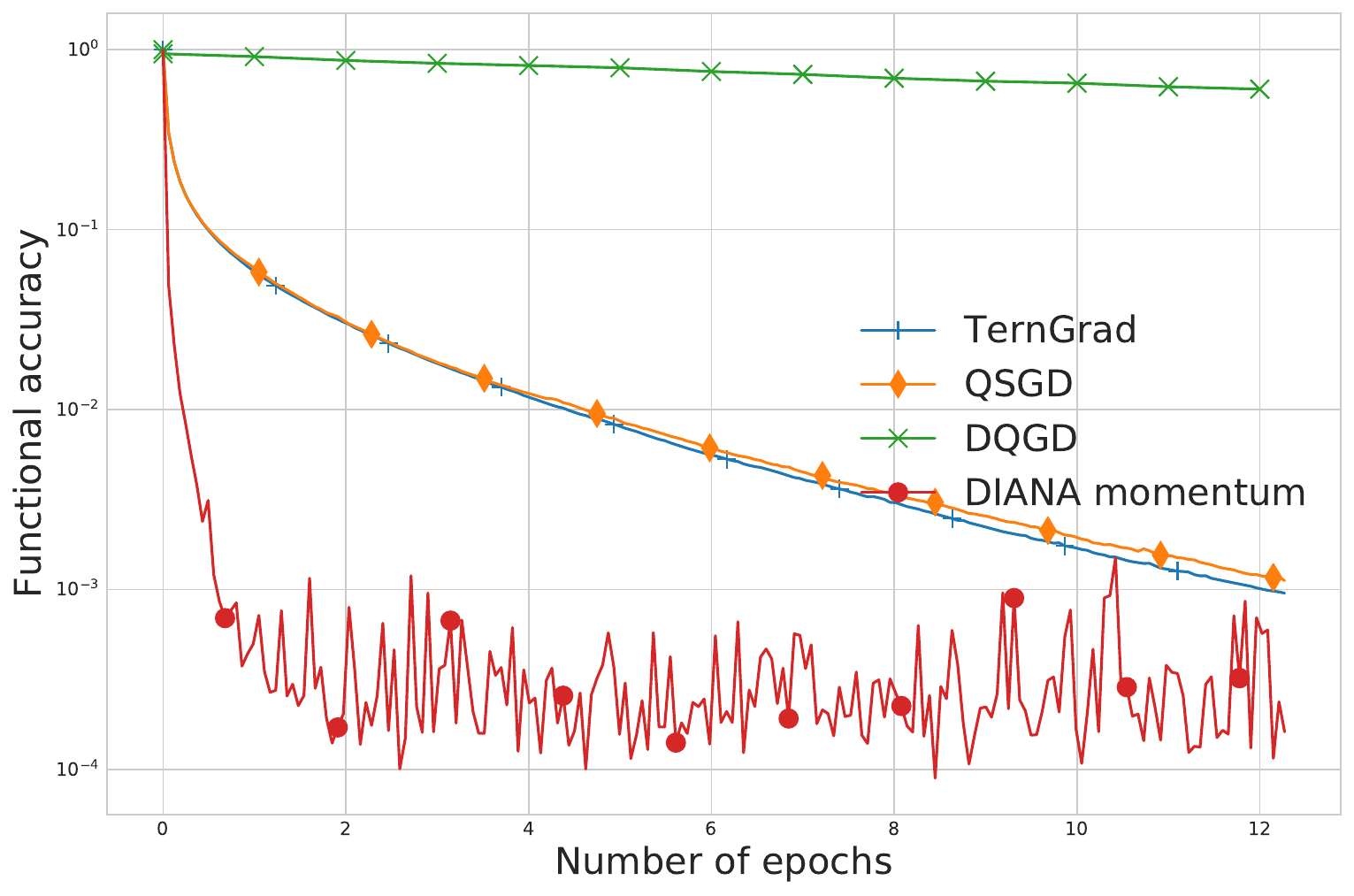}

\caption{Comparison of the {\tt DIANA} ($\beta = 0.95$) with {\tt QSGD}, {\tt TernGrad} and {\tt DQGD} on the logistic regression problem for the "mushrooms" dataset.}\label{fig:diana_main}
\end{figure}

{\bf {\tt DIANA} vs MPI.} In Figure~\ref{fig:imagesPerSecond2}
we compare the performance of {\tt DIANA} vs. doing a MPI reduce operation with 32bit floats. The computing cluster had Cray Aries High Speed Network. However, for {\tt DIANA} we used 2bit per dimension and  have experienced a strange scaling behaviour, which was documented also in~\cite{chunduriperformance}. In our case, this affected speed for alexnet and vgg\_a beyond 64 or 32 MPI processes respectively. For more detailed experiments, see Section~\ref{sec:A:detailsOfNumericalExperiments}.

{\bf Train and test accuracy on Cifar10.} In the next experiments, we run {\tt QSGD} \cite{alistarh2017qsgd}, {\tt TernGrad} \cite{wen2017terngrad}, {\tt SGD} with momentum and {\tt DIANA} on Cifar10 dataset for 3 epochs. We have selected 8 workers and run each method for learning rate from $\{0.1, 0.2, 0.05\}$.
For {\tt QSGD}, {\tt DIANA} and {\tt TernGrad}, we also tried various quantization bucket sizes in $\{32, 128, 512\}$.
For {\tt QSGD} we have chosen $2,4,8$ quantization levels.
For {\tt DIANA} we have chosen $\alpha \in 
\{0, 1.0/\sqrt{\mbox{quantization bucket sizes }}\}$
and have selected initial $h = 0$. 
For {\tt DIANA} and {\tt SGD} we also run a momentum version, with a momentum parameter in $\{0, 0.95, 0.99\}$.
For {\tt DIANA} we also run with two choices of norm $\ell_2$ and $\ell_\infty$.
For each experiment we have selected softmax cross entropy loss. Cifar10-DNN is a convolutional DNN described here
\url{https://github.com/kuangliu/pytorch-cifar/blob/master/models/lenet.py}.

In Figure~\ref{fig:DNN:evolution2} we show the best runs over all the parameters for all the methods. We notice that {\tt DIANA} and {\tt SGD}  significantly outperform other methods.

\begin{figure}[h]

\centering 
\includegraphics[width=0.24\textwidth]{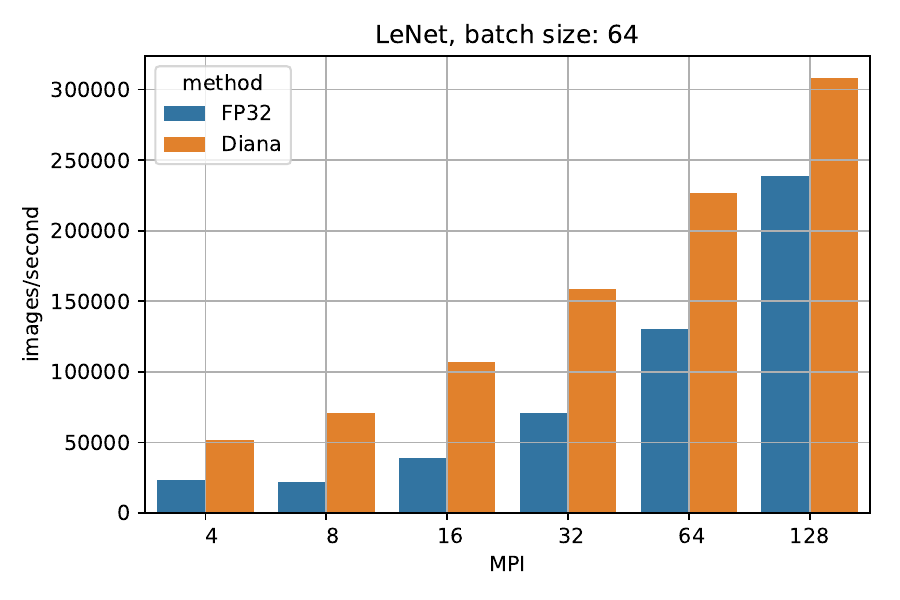}
\includegraphics[width=0.24\textwidth]{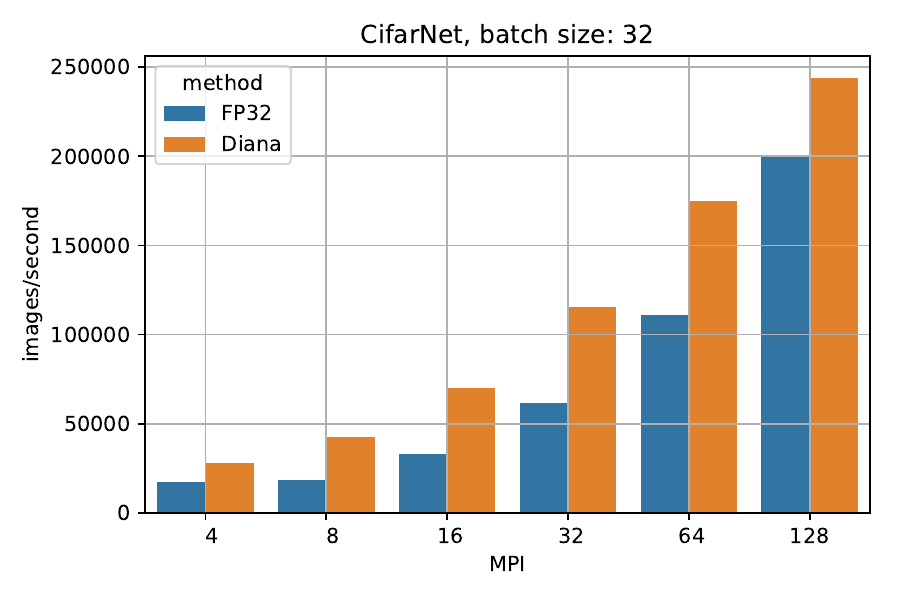}
\includegraphics[width=0.24\textwidth]{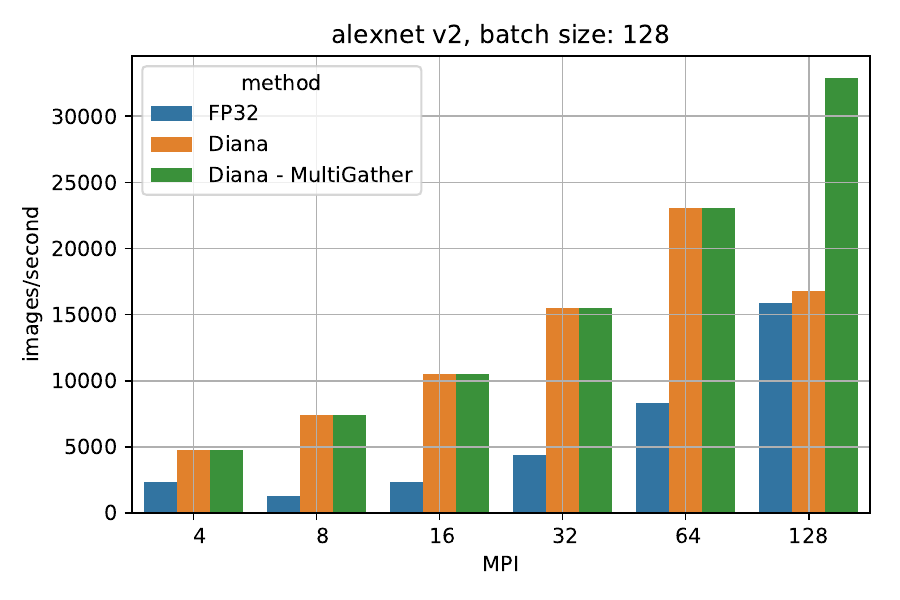}
\includegraphics[width=0.24\textwidth]{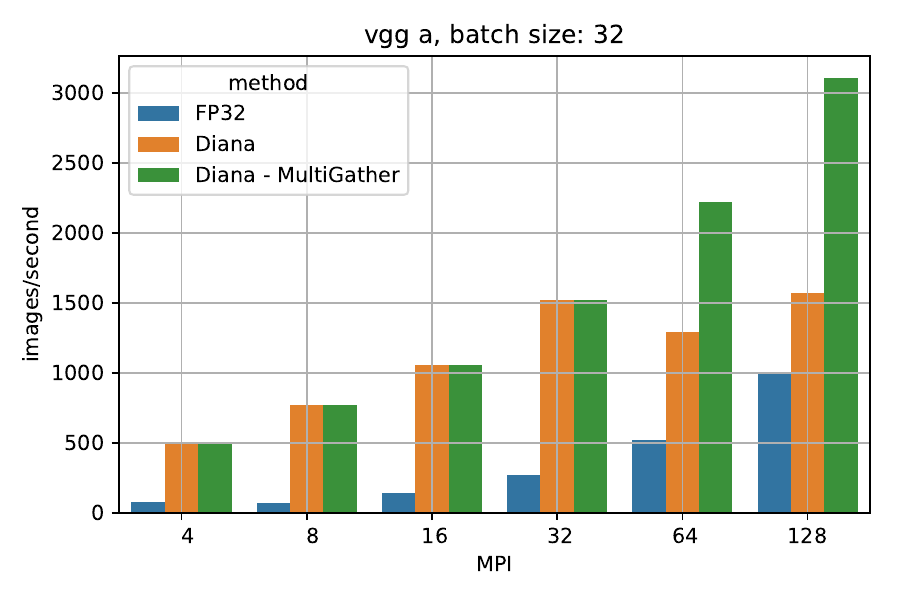}

\caption{Comparison of performance (images/second) for various number of GPUs/MPI processes and sparse communication {\tt DIANA} (2bit) vs. Reduce with 32bit float (FP32).}
\label{fig:imagesPerSecond2}

\end{figure}

\begin{figure}[h] 
\centering 
\includegraphics[width=0.34\textwidth]{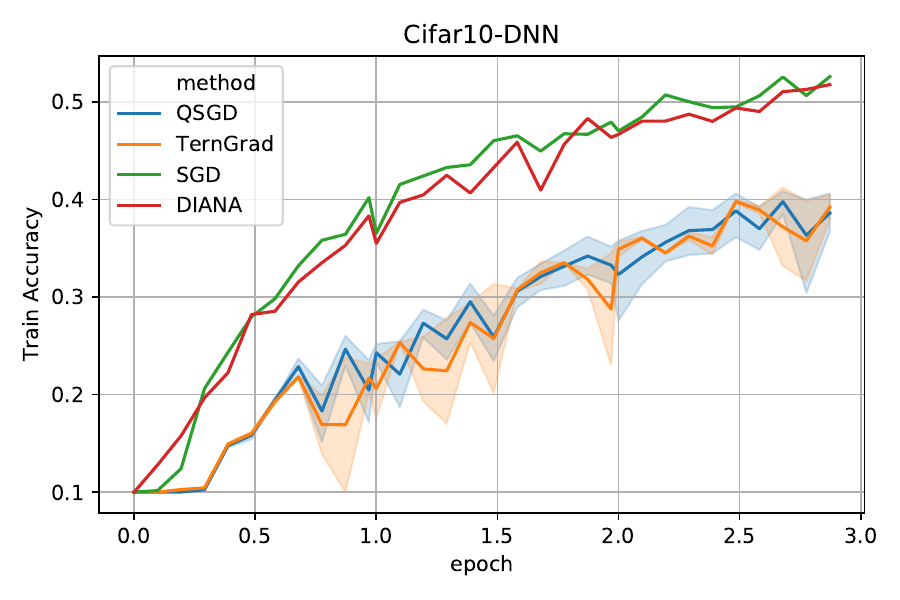}
\includegraphics[width=0.34\textwidth]{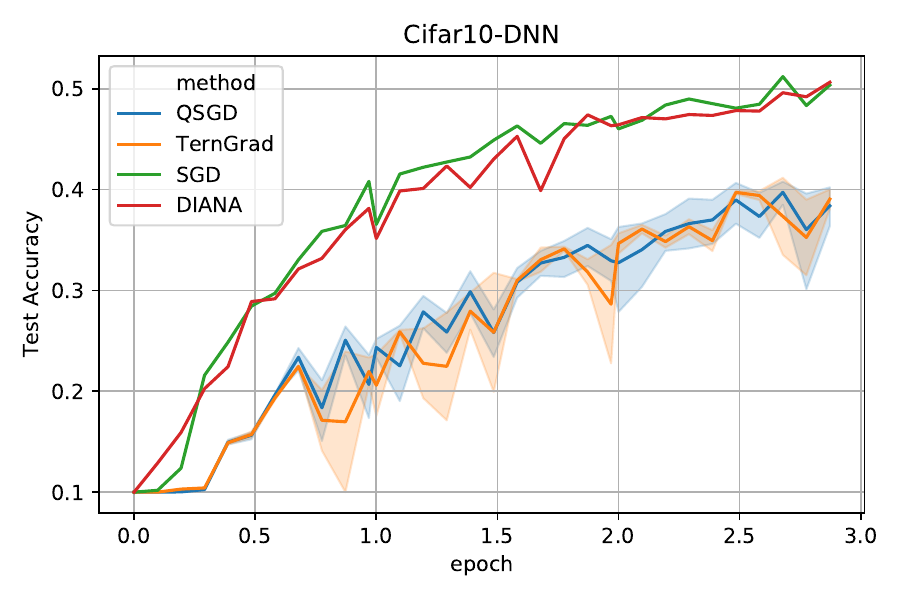}

\caption{Evolution of training (left) and testing (right) accuracy on Cifar10, using 4 algorithms: {\tt DIANA}, {\tt SGD}, {\tt QSGD} and {\tt TernGrad}. 
We have chosen the best runs over all tested hyper-parameters.}
\label{fig:DNN:evolution2}

\end{figure}

\begin{figure}[!h]
\centering
\includegraphics[scale=.5]{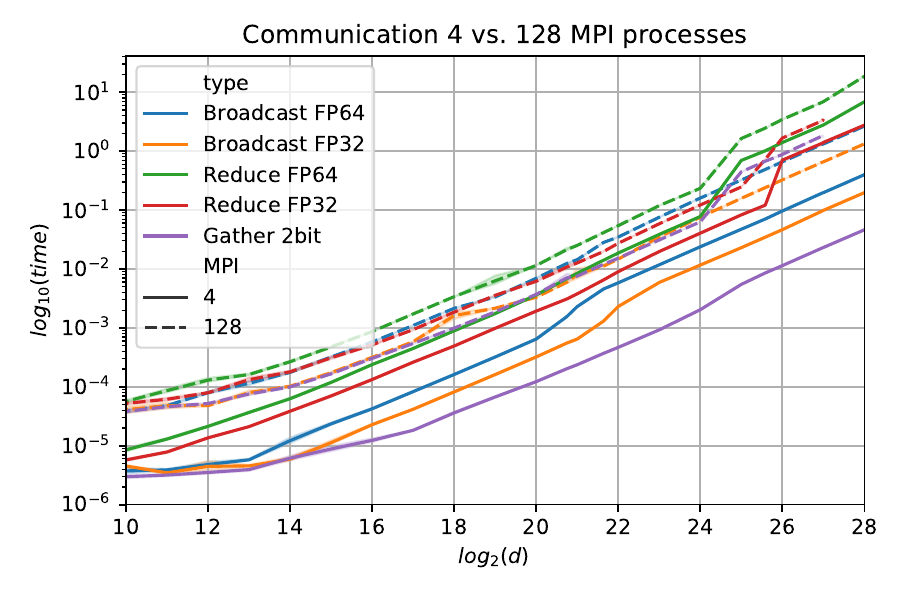}
\caption{Typical communication cost using  broadcast, reduce and gather  for 64 and 32 FP using 4 (solid) resp 128 (dashed) MPI processes.
See Section~\ref{sec:A:MPI} for details about the network.
}
\label{fig:communication}
\end{figure}

\section{Discussion}

In this work, we propose a new distributed optimization method with compression -- {\tt DIANA}. The key feature of {\tt DIANA} is the linear convergence to the exact optimum asymptotically in expectation for the strongly convex problems, when workers use full gradients of their local functions that are allowed to be arbitrary heterogeneous. We believe this is the main reason why our work attracted a lot of attention after it appeared on arXiv in 2019. In particular, our work received multiple extensions including the generalization to the general unbiased compression and variance reduced variants \cite{horvath2019stochastic}, accelerated versions \cite{li2020acceleration}, decentralized versions \cite{kovalev2021linearly}, combinations with biased compression and methods with delayed updates \cite{gorbunov2020linearly}, extensions to the second-order distributed methods \cite{safaryan2021fednl}, methods with bidirectional compression \cite{philippenko2020bidirectional, philippenko2021preserved}, Byzantine-robust distributed methods with compression \cite{zhu2021broadcast}, and methods for distributed variational inequalities and min-max problems \cite{beznosikov2022stochastic}.

\section*{Acknowledgement}

The work of E. Gorbunov was partially supported by a grant for research centers in the field of artificial intelligence, provided by the Analytical Center for the Government of the Russian Federation in accordance with the subsidy agreement (agreement identifier 000000D730321P5Q0002) and the agreement with the Moscow Institute of Physics and Technology dated November 1, 2021 No. 70-2021-00138.

%\section*{Funding}
%
%An unnumbered section, e.g.\ \verb"\section*{Funding}", may be used for grant details, etc.\ if required and included \emph{in the non-anonymous version} before any Notes or References.

\bibliographystyle{tfs}
\bibliography{refs}

\begin{thebibliography}{10}
\providecommand{\MR}{\relax\unskip\space MR }
\providecommand{\url}[1]{\normalfont{#1}}
\providecommand{\urlprefix}{Available at }

\bibitem{alistarh2017qsgd}
D. Alistarh, D. Grubic, J. Li, R. Tomioka, and M. Vojnovic, \emph{{QSGD}: {C}ommunication-efficient {SGD} via gradient quantization and encoding}, in \emph{Advances in Neural Information Processing Systems}. 2017, pp. 1709--1720.

\bibitem{pmlr-v80-bernstein18a}
J. Bernstein, Y.X. Wang, K. Azizzadenesheli, and A. Anandkumar, \emph{sign{SGD}: Compressed Optimisation for Non-Convex Problems}, in \emph{Proceedings of the 35th International Conference on Machine Learning}, J. Dy and A. Krause, eds., Proceedings of Machine Learning Research Vol.~80, 10--15 Jul, Stockholmsmässan, Stockholm Sweden. PMLR, 2018, pp. 560--569.

\bibitem{beznosikov2022stochastic}
A. Beznosikov, E. Gorbunov, H. Berard, and N. Loizou, \emph{Stochastic gradient descent-ascent: Unified theory and new efficient methods}, arXiv preprint arXiv:2202.07262  (2022).

\bibitem{chunduriperformance}
S. Chunduri, P. Coffman, S. Parker, and K. Kumaran, \emph{Performance analysis of mpi on cray xc40 xeon phi system}  (2017).

\bibitem{Hydra2}
O. Fercoq, Z. Qu, P. Richt\'{a}rik, and M. Tak{\'a}{\v{c}}, \emph{Fast distributed coordinate descent for minimizing non-strongly convex losses}, IEEE International Workshop on Machine Learning for Signal Processing  (2014).

\bibitem{gorbunov2020linearly}
E. Gorbunov, D. Kovalev, D. Makarenko, and P. Richt{\'a}rik, \emph{Linearly converging error compensated sgd}, Advances in Neural Information Processing Systems 33 (2020), pp. 20889--20900.

\bibitem{horvath2019stochastic}
S. Horv{\'a}th, D. Kovalev, K. Mishchenko, S. Stich, and P. Richt{\'a}rik, \emph{Stochastic distributed learning with gradient quantization and variance reduction}, arXiv preprint arXiv:1904.05115  (2019).

\bibitem{cocoa}
M. Jaggi, V. Smith, M. Tak{\'a}\v{c}, J. Terhorst, S. Krishnan, T. Hofmann, and M.I. Jordan, \emph{Communication-efficient distributed dual coordinate ascent}, in \emph{Advances in Neural Information Processing Systems 27}. 2014.

\bibitem{jahani2018efficient}
M. Jahani, X. He, C. Ma, A. Mokhtari, D. Mudigere, A. Ribeiro, and M. Tak{\'a}{\v{c}}, \emph{Efficient distributed hessian free algorithm for large-scale empirical risk minimization via accumulating sample strategy}, arXiv:1810.11507  (2018).

\bibitem{khirirat2018distributed}
S. Khirirat, H.R. Feyzmahdavian, and M. Johansson, \emph{Distributed learning with compressed gradients}, arXiv preprint arXiv:1806.06573  (2018).

\bibitem{konevcny2016federated}
J. Kone{\v{c}}n{\'y}, H.B. McMahan, F.X. Yu, P. Richt{\'a}rik, A.T. Suresh, and D. Bacon, \emph{Federated learning: Strategies for improving communication efficiency}, arXiv preprint arXiv:1610.05492  (2016).

\bibitem{konecny2016randomized}
J. Kone\v{c}n{\'y} and P. Richt{\'a}rik, \emph{Randomized distributed mean estimation: Accuracy vs communication}, arXiv preprint arXiv:1611.07555  (2016).

\bibitem{kovalev2021linearly}
D. Kovalev, A. Koloskova, M. Jaggi, P. Richtarik, and S. Stich, \emph{A linearly convergent algorithm for decentralized optimization: Sending less bits for free!}, in \emph{International Conference on Artificial Intelligence and Statistics}. PMLR, 2021, pp. 4087--4095.

\bibitem{li2020acceleration}
Z. Li, D. Kovalev, X. Qian, and P. Richt{\'a}rik, \emph{Acceleration for compressed gradient descent in distributed and federated optimization}, arXiv preprint arXiv:2002.11364  (2020).

\bibitem{AccCOCOA}
C. Ma, M. Jaggi, F.E. Curtis, N. Srebro, and M. Tak\'{a}\v{c}, \emph{An accelerated communication-efficient primal-dual optimization framework for structured machine learning}, arXiv:1711.05305  (2017).

\bibitem{cocoa+journal}
C. Ma, J. Kone\v{c}n\'{y}, M. Jaggi, V. Smith, M.I. Jordan, P. Richt\'{a}rik, and M. Tak\'{a}\v{c}, \emph{Distributed optimization with arbitrary local solvers}, Optimization Methods and Software 32 (2017), pp. 813--848.

\bibitem{cocoa+}
C. Ma, V. Smith, M. Jaggi, M.I. Jordan, P. Richt\'{a}rik, and M. Tak\'{a}\v{c}, \emph{Adding vs. averaging in distributed primal-dual optimization}, in \emph{The 32nd International Conference on Machine Learning}. 2015, pp. 1973--1982.

\bibitem{ma2015partitioning}
C. Ma and M. Tak{\'a}{\v{c}}, \emph{Partitioning data on features or samples in communication-efficient distributed optimization?}, OptML@NIPS 2015, arXiv:1510.06688  (2015).

\bibitem{mishchenko2019distributed}
K. Mishchenko, E. Gorbunov, M. Tak{\'a}{\v{c}}, and P. Richt{\'a}rik, \emph{Distributed learning with compressed gradient differences}, arXiv preprint arXiv:1901.09269  (2019).

\bibitem{philippenko2020bidirectional}
C. Philippenko and A. Dieuleveut, \emph{Bidirectional compression in heterogeneous settings for distributed or federated learning with partial participation: tight convergence guarantees}, arXiv preprint arXiv:2006.14591  (2020).

\bibitem{philippenko2021preserved}
C. Philippenko and A. Dieuleveut, \emph{Preserved central model for faster bidirectional compression in distributed settings}, Advances in Neural Information Processing Systems 34 (2021), pp. 2387--2399.

\bibitem{reddi2016aide}
S.J. Reddi, J. Kone{\v c}n{\'y}, P. Richt{\'a}rik, B. P{\'o}cz{\'o}s, and A. Smola, \emph{{AIDE}: Fast and communication efficient distributed optimization}, arXiv preprint arXiv:1608.06879  (2016).

\bibitem{Hydra}
P. Richt{\'a}rik and M. Tak{\'a}{\v{c}}, \emph{Distributed coordinate descent method for learning with big data}, Journal of Machine Learning Research 17 (2016), pp. 1--25.

\bibitem{safaryan2021fednl}
M. Safaryan, R. Islamov, X. Qian, and P. Richt{\'a}rik, \emph{Fednl: Making newton-type methods applicable to federated learning}, arXiv preprint arXiv:2106.02969  (2021).

\bibitem{seide20141}
F. Seide, H. Fu, J. Droppo, G. Li, and D. Yu, \emph{1-bit stochastic gradient descent and its application to data-parallel distributed training of speech dnns}, in \emph{Fifteenth Annual Conference of the International Speech Communication Association}. 2014.

\bibitem{DANE}
O. Shamir, N. Srebro, and T. Zhang, \emph{Communication-Efficient Distributed Optimization using an Approximate {N}ewton-type Method}, in \emph{Proceedings of the 31st International Conference on Machine Learning, PMLR}, Vol.~32. 2014, pp. 1000--1008.

\bibitem{cocoa-2018-JMLR}
V. Smith, S. Forte, C. Ma, M. Tak\'{a}\v{c}, M.I. Jordan, and M. Jaggi, \emph{{C}o{C}o{A}: A general framework for communication-efficient distributed optimization}, Journal of Machine Learning Research 18 (2018), pp. 1--49.

\bibitem{wen2017terngrad}
W. Wen, C. Xu, F. Yan, C. Wu, Y. Wang, Y. Chen, and H. Li, \emph{Terngrad: Ternary gradients to reduce communication in distributed deep learning}, in \emph{Advances in Neural Information Processing Systems}. 2017, pp. 1509--1519.

\bibitem{DISCO}
Y. Zhang and L. Xiao, \emph{{D}i{SCO}: Distributed Optimization for Self-Concordant Empirical Loss}, in \emph{Proceedings of the 32nd International Conference on Machine Learning, PMLR}, Vol.~37. 2015, pp. 362--370.

\bibitem{zhu2021broadcast}
H. Zhu and Q. Ling, \emph{Broadcast: Reducing both stochastic and compression noise to robustify communication-efficient federated learning}, arXiv preprint arXiv:2104.06685  (2021).

\end{thebibliography}

\appendix

\section{Basic Identities and Inequalities}

{\bf Smoothness and strong convexity.}  If $f$ is $L$-smooth and $\mu$-strongly convex, then for any vectors $x,y\in \R^d$ we have
    \begin{align}\label{eq:str_convexity}
        \< \nabla f(x) - \nabla f(y), x - y> \ge \frac{\mu L}{\mu + L}\|x - y\|_2^2 + \frac{1}{\mu + L}\|\nabla f(x) - \nabla f(y)\|^2_2.
    \end{align}

{\bf Norm of a convex combination.} For any $0\leq \alpha \leq 1$ and $x,y\in \R^d$, we have
\begin{align}\label{eq:variance_decompos}
    \|\alpha x + (1 - \alpha) y  \|_2^2 =  \alpha \|x\|_2^2 + (1 - \alpha)\|y\|_2^2  - \alpha(1 - \alpha) \|x - y\|_2^2.
\end{align}

{\bf Variance decomposition.} The (total) variance of a random vector $g\in \R^d$ is defined as the trace of its covariance matrix:
\begin{eqnarray*} \VV [g] \eqdef  {\rm Tr} \left[ \EE \left[ (g-\EE g)(g- \EE g)^\top \right] \right] &=& \EE \left[{\rm Tr} \left[ (g-\EE g)(g- \EE g)^\top \right] \right] \\
&=&  \EE \|g-\EE g\|_2^2 \\
&=& \sum_{j=1}^d \EE (g_{(j)} - \EE g_{(j)})^2. \end{eqnarray*}
For any vector $h\in \R^d$, the variance of $g$ can be decomposed as follows: 
\begin{align}
 \EE\|g - \EE g\|_2^2  = 	\EE \|g - h\|_2^2 -  \|\EE g - h\|_2^2 . \label{eq:second_moment_decomposition}
\end{align}

\section{Block $p$-quantization}\label{appendix:block_quant}

We now introduce a block version of $p$-quantization. We found these quantization operators to have better properties in practice.

\begin{definition}[block-$p$-quantization]\label{def:block-p-quant} Let $\Delta = (\Delta(1), \Delta(2),\ldots, \Delta(m)) \in \R^d$, where $\Delta(1)\in\R^{d_1},\ldots, \Delta(m)\in\R^{d_m}$, $d_1+\ldots+d_m = d$ and $d_l > 1$ for all $l=1,\ldots,m$. We say that $\hat\Delta$ is $p$-quantization of $\Delta$ with sizes of blocks $\{d_l\}_{l=1}^m$ and write $\hat\Delta \sim {\rm Quant}_{p}(\Delta,\{d_l\}_{l=1}^m)$ if $\hat\Delta(l) \sim {\rm Quant}_{p}(\Delta)$ for all $l=1,\ldots,m$.
\end{definition}

In other words, we quantize blocks called \textit{blocks} of the initial vector. Note that in the special case when $m=1$ we get full quantization: ${\rm Quant}_{p}(\Delta,\{d_l\}_{l=1}^m) = {\rm Quant}_{p}(\Delta)$. Note that we do not assume independence of the quantization of blocks or independence of $\xi_{(j)}$. Lemma~\ref{lem:moments1} in the appendix states that $\hat{\Delta}$ is an unbiased estimator of  $\Delta$, and gives a formula for its variance.

Next we show that the block $p$-quantization operator $\hat{\Delta}$ introduced in Definition \ref{def:block-p-quant} is an unbiased estimator of  $\Delta$, and give a formula for its variance.

\begin{lemma}\label{lem:moments1}
	Let $\Delta \in \R^d$ and $\hat{\Delta} \sim {\rm Quant}_p(\Delta)$. Then for $l=1,\ldots,m$ 
	\begin{equation}
		\EE  \hat \Delta(l)   = \Delta(l), \qquad
        \EE \|\hat \Delta(l) - \Delta(l)\|_2^2  = \Psi_l(\Delta), \label{eq:tilde_v_moments1}
        \end{equation}	
	\begin{equation}
		\EE  \hat \Delta   = \Delta, \qquad
        \EE \|\hat \Delta - \Delta\|_2^2  = \Psi(\Delta), \label{eq:hat_v_moments1} 
            \end{equation}
        where
      $x = (x(1),x(2),\ldots, x(m))$, 
       $ \Psi_l(x) \eqdef \|x(l)\|_1 \|x(l)\|_p - \|x(l)\|_2^2 \geq 0,$
		and $\Psi(x) \eqdef \sum\limits_{l=1}^m\Psi_l(x) \ge 0.$
 Thus,  $\hat{\Delta}$ is an unbiased estimator of $\Delta$. Moreover, the variance of $\hat{\Delta}$ is a decreasing function of $p$, and is minimized for $p=\infty$.
\end{lemma}

\begin{proof}
Note that the first part of \eqref{eq:hat_v_moments1} follows from the first part of \eqref{eq:tilde_v_moments1} and the second part of \eqref{eq:hat_v_moments1} follows from the second part of \eqref{eq:tilde_v_moments1} and \[\|\hat{\Delta}-\Delta\|_2^2 = \sum\limits_{l=1}^m \|\hat\Delta(l)-\Delta(l)\|_2^2.\] Therefore, it is sufficient to prove \eqref{eq:tilde_v_moments1}. If $\Delta(l)=0$, the statements follow trivially. Assume $\Delta(l)\neq 0$.
In view of \eqref{eq:quant-j}, we have \[\EE \hat{\Delta}_{(j)}(l) = \|\Delta(l)\|_p \sign(\Delta_{(j)}(l)) \EE \xi_{(j)} =  \|\Delta(l)\|_p \sign(\Delta_{(j)}(l)) |\Delta_{(j)}(l)|/ \|\Delta(l)\|_p  = \Delta_{(j)}(l),\] which establishes the first claim. We can write
\begin{eqnarray*}\EE \|\hat \Delta(l) - \Delta(l)\|_2^2 &=& \EE \sum_j (\hat{\Delta}_{(j)}(l) - \Delta_{(j)}(l))^2 \\
&=& \EE \sum_j (\hat{\Delta}_{(j)}(l) - \EE \hat{\Delta}_{(j)}(l))^2 \\
& \overset{\eqref{eq:quant-j}}{=} &  \|\Delta(l)\|_p^2   \sum_j \sign^2(\Delta_{(j)}(l)) \EE (\xi_{(j)} - \EE \xi_{(j)})^2 \\
&=& \|\Delta(l)\|_p^2 \sum_j \sign^2(\Delta_{(j)}(l)) \tfrac{|\Delta_{(j)}(l)|}{\|\Delta(l)\|_p} (1 - \tfrac{|\Delta_{(j)}(l)|}{\|\Delta(l)\|_p})\\
& =&  \sum_j |\Delta_{(j)}(l)| (\|\Delta(l)\|_p - |\Delta_{(j)}(l)|) \\
&= & \|\Delta(l)\|_1 \|\Delta(l)\|_p - \|\Delta(l)\|_2.
\end{eqnarray*}
\end{proof}

\section{Proof of Theorem~\ref{th:quantization_quality1}}

Let $1_{[\cdot]}$ denote the indicator random variable of an event.  In view of \eqref{eq:quant-j}, $\hat \Delta_{(j)} = \|\Delta\|_p \sign(\Delta_{(j)}) \xi_{(j)}$, where $\xi_{(j)}\sim {\rm Be}(|\Delta_{(j)}|/\|\Delta\|_p)$. Therefore,
\[
\|\hat \Delta\|_0 = \sum_{j=1}^d 1_{[\hat \Delta_{(j)} \neq 0]} = \sum_{j \;:\; \Delta_{(j)} \neq 0}^d 1_{[  \xi_{(j)}=1]} ,
\]
which implies that
\[
\EE  \|\hat \Delta\|_0 = \EE \sum_{j \;:\; \Delta_{(j)} \neq 0}^d 1_{[  \xi_{(j)}=1]}  = \sum_{j \;:\; \Delta_{(j)} \neq 0}^d  \EE  1_{[  \xi_{(j)}=1]}
 = \sum_{j \;:\; \Delta_{(j)} \neq 0}^d  \frac{|\Delta_{(j)}|}{\|\Delta\|_p} = \frac{\|\Delta\|_1}{\|\Delta\|_p}.
\]
To establish the first clam, it remains to recall that for all $x\in \R^d$ and $1 \leq q \leq p \leq +\infty$, one has the bound
\[ \|x\|_p \leq \|x\|_q \leq \|x\|_0^{1/q-1/p} \|x\|_p, \]
and apply it with $q  =1$.

The proof of the second claim follows the same pattern, but uses the concavity of $t \mapsto \sqrt{t}$ and Jensen's inequality in one step.

\section{Proof of Lemma~\ref{lema:alpha_p}}

$\alpha_p(d)$ is  increasing as a function of $p$ because $\|\cdot\|_p$ is decreasing as a function of $p$. Moreover, $\alpha_p(d)$ is decreasing as a function of $d$ since if we have $d < b$ then
\[
	\alpha_p(b) = \inf\limits_{x\neq 0, x\in\R^b}\frac{\|x\|_2^2}{\|x\|_1\|x\|_p} \leqslant \inf\limits_{x\neq 0, x\in\R_d^b}\frac{\|x\|_2^2}{\|x\|_1\|x\|_p} = \inf\limits_{x\neq 0, x\in\R^d}\frac{\|x\|_2^2}{\|x\|_1\|x\|_p},
\]
where $R_d^b \eqdef \{x\in\R^b: x_{(d+1)} = \ldots = x_{(b)} = 0\}$.
It is known that $\tfrac{\|x\|_2}{\|x\|_1}\geq \tfrac{1}{\sqrt{d}}$, and that this bound is tight. Therefore, \[\alpha_1(d) = \inf_{x\neq 0,x\in\R^d} \frac{\|x\|_2^2}{\|x\|_1^2} = \frac{1}{d}\] 
and
\[\alpha_2(d) =  \inf_{x\neq 0,x\in\R^d} \frac{\|x\|_2}{\|x\|_1} = \frac{1}{\sqrt{d}}.\]

Let us now establish that $\alpha_\infty(d) = \tfrac{2}{1+\sqrt{d}}$. Note that
\[
\frac{\|x\|_2^2}{\|x\|_1\|x\|_\infty} = \frac{\left\|\frac{x}{\|x\|_\infty}\right\|_2^2}{\left\|\frac{x}{\|x\|_\infty}\right\|_1\left\|\frac{x}{\|x\|_\infty}\right\|_\infty} = \frac{\left\|\frac{x}{\|x\|_\infty}\right\|_2^2}{\left\|\frac{x}{\|x\|_\infty}\right\|_1}.
\]
Therefore, w.l.o.g.\ one can assume that $\|x\|_\infty = 1$. Moreover, signs of coordinates of vector $x$ do not influence aforementioned quantity either, so one can consider only $x\in\R_+^d$. In addition, since $\|x\|_\infty = 1$, one can assume that $x_{(1)} = 1$. Thus, our goal now is to show that the minimal value of the function
\[
	f(x) = \frac{1 + x_{(2)}^2 + \ldots + x_{(d)}^2}{1 + x_{(2)} + \ldots + x_{(d)}}
\]
on the set $M = \{x\in\R^d \mid x_{(1)} = 1, 0\le x_{(j)} \le 1, j=2,\ldots,d\}$ is equal to $\frac{2}{1+\sqrt{d}}$. 
By Cauchy-Schwartz inequality: $x_{(2)}^2 + \ldots + x_{(d)}^2 \ge \frac{(x_{(2)} + \ldots + x_{(d)})^2}{d-1}$ and it becomes equality if and only if all $x_{(j)},j=2,\ldots,d$ are equal. It means that if we fix $x_{(j)} = a$ for $j=2,\ldots,d$ and some $0\le a \le 1$ than the minimal value of the function
\[
	g(a) = \frac{1 + \frac{((d-1)a)^2}{d-1}}{1 + (d-1)a} = \frac{1 + (d-1)a^2}{1 + (d-1)a}
\]
on $[0,1]$ coincides with minimal value of $f$ on $M$. The derivative
\[
	g'(a) = \frac{2(d-1)a}{1+(d-1)a} - \frac{(d-1)(1+(d-1)a^2)}{(1+(d-1)a)^2}
\]
has the same $\sign$ on $[0,1]$ as the difference $a - \frac{1 + (d-1)a^2}{2(1+ (d-1)a)}$, which implies that $g$ attains its minimal value on $[0,1]$ at such $a$ that $a = \frac{1 + (d-1)a^2}{2(1+ (d-1)a)}$. It remains to find $a\in[0,1]$ which satisfies
\[
	a = \frac{1 + (d-1)a^2}{2(1 + (d-1)a)},\quad a\in[0,1] \Longleftrightarrow (d-1)a^2 + 2a - 1 = 0,\quad a\in[0,1].
\]
This quadratic equation has unique positive solution $a^* = \frac{-1 + \sqrt{d}}{d-1} = \frac{1}{1 + \sqrt{d}} < 1$. Direct calculations show that $g(a^*) = \frac{2}{1+\sqrt{d}}$. It implies that $\alpha(d) = \frac{2}{1+\sqrt{d}}$.

\section{Strongly Convex Case: Optimal Number of Nodes} \label{sec:opt_no_nodes}

In practice one has access to a finite data set, consisting of $N$ data points, where $N$ is very large, and wishes to solve an empirical risk minimization (``finite-sum'') of the form
\begin{equation}\label{eq:N} \textstyle \min_{x\in \R^d} f(x) = \tfrac{1}{N} \sum \limits_{i=1}^N \phi_i(x) + R(x),\end{equation}
where each $\phi_i$ is $L$--smooth and $\mu$-strongly convex. If $n \leq N$ compute nodes of a distributed system  are available, one may partition the $N$ functions into $n$ groups, $G_1,\dots,G_n$, each of size $|G_i| = N/n$, and define
$f_i(x) = \frac{n}{N} \sum_{i \in G_i} \phi_i(x). $
Note that $f(x) = \frac{1}{n} \sumin f_i(x) + R(x).$
Note that each $f_i$ is also $L$--smooth and $\mu$--strongly convex.

This way, we  have  fit the original (and large) problem \eqref{eq:N} into our framework. One may now ask the question: {\em How many many nodes $n$ should we use (other things equal)?} If what we care about is iteration complexity, then insights can be gained by investigating Eq.\  \eqref{eq:bu987gd9}. For instance, if $p=2$, then the complexity is $W(n) \eqdef \max\left\{\nicefrac{2\sqrt{d}}{\sqrt{m}},(\kappa+1)\left(\nicefrac{1}{2}-\nicefrac{1}{n}+\nicefrac{\sqrt{d}}{n\sqrt{m}}\right)\right\}.$ The optimal choice is to choose $n$ so that the term $-\nicefrac{1}{n} + \nicefrac{\sqrt{d}}{n\sqrt{m}}$ becomes (roughly) equal to $\nicefrac{1}{2}$:
$-\nicefrac{1}{n} + \nicefrac{\sqrt{d}}{n\sqrt{m}} = \nicefrac{1}{2}.$ This gives the formula for the optimal number of nodes
$n^* = n(d) \eqdef 2\left(\sqrt{\nicefrac{d}{m}} - 1\right),$
and the resulting iteration complexity is $W(n^*) = \max\left\{\nicefrac{2\sqrt{d}}{\sqrt{m}}, \kappa + 1\right\}$.  Note that $n(d)$ is increasing in $d$. Hence, it makes sense to use more nodes for larger models (big $d$).

\section{Quantization Lemmas}

Consider iteration $k$ of the {\tt DIANA} method (Algorithm~\ref{alg:distributed1}).  Let $\EE_{Q^k}$ be the expectation with respect to the randomness inherent in the quantization steps $\hat \Delta_i^k \sim {\rm Quant}_p(\Delta_i^k,\{d_l\}_{l=1}^m)$ for $i=1,2,\dots,n$ (i.e.\ we condition on everything else).

\begin{lemma}\label{lem:3in1} For all iterations $k\geq 0$  of {\tt DIANA} and $i=1,2,\dots,n$ we have the identities
\begin{equation}\label{eq:hat_gi_moments1} \EE_{Q^k} \hat g_i^k   = g_i^k, \qquad \EE_{Q^k}  \| \hat g_i^k - g_i^k\|_2^2   =\Psi(\Delta_i^k) \end{equation}
and
\begin{equation} \label{eq:distr_hat_g_moments1} \EE_{Q^k} \hat g^k   = g^k \eqdef \frac{1}{n}\sum_{i=1}^n g_i^k, \qquad \EE_{Q^k}  \| \hat g^k - g^k\|_2^2   = \frac{1}{n^2}\sum_{i=1}^n\Psi(\Delta_i^k) .\end{equation}
Furthermore, letting $h^* = \nabla f(x^*)$, and invoking Assumption~\ref{as:noise}, we have
\begin{equation}\EE \hat g^k = \nabla f(x^k), \qquad \EE \|\hat g^k -h^*\|_2^2 \leq \EE\|\nabla f(x^k) - h^*\|_2^2 + \left(\frac{1}{n^2}\sum_{i=1}^n\EE \Psi(\Delta_i^k)\right) + \frac{\sigma^2}{n} . \label{eq:full_variance_of_mean_g1}
\end{equation}
\end{lemma}
\begin{proof}

\begin{itemize}
\item[(i)] Since $\hat g_i^k = h_i^k + \hat \Delta_i^k$ and $\Delta_i^k = g_i^k - h_i^k$, we can apply Lemma~\ref{lem:moments1} and obtain
\[\EE_{Q^k} \hat g_i^k  = h_i^k + \EE_{Q^k} \hat \Delta_i^k \overset{\eqref{eq:hat_v_moments1}}{=}  h_i^k + \Delta_i^k = g_i^k.\]
Since
$\hat g_i^k - g_i^k = \hat \Delta_i^k - \Delta_i^k $, applying the second part of Lemma 1 gives the second identity in \eqref{eq:hat_gi_moments1}.

\item [(ii)] The first part of \eqref{eq:distr_hat_g_moments1}  follows directly from the first part of \eqref{eq:hat_gi_moments1}: 
\[\EE_{Q^k} \hat g^k = \EE_{Q^k} \left[ \frac{1}{n} \sum_{i=1}^n \hat g_i^k \right] =  \frac{1}{n} \sum_{i=1}^n \EE_{Q^k} \hat g_i^k \overset{\eqref{eq:hat_gi_moments1}}{=} \frac{1}{n} \sum_{i=1}^n  g_i^k \overset{\eqref{eq:distr_hat_g_moments1} }{=}  g^k.\]
The second part in \eqref{eq:distr_hat_g_moments1} follows from the second part of \eqref{eq:hat_gi_moments1} and independence of $\hat g_1^k,\dotsc, \hat g_n^k$.

\item [(iii)] The first part of \eqref{eq:full_variance_of_mean_g1} follows directly from the first part of \eqref{eq:distr_hat_g_moments1} and the assumption that  $g_i^k$ is  and unbiased estimate of $\nabla f_i(x^k)$.
 It remains to establish the second part of  \eqref{eq:full_variance_of_mean_g1}. First, we shall decompose
 \begin{eqnarray*}
    	\EE_{Q^k}\|\hat g^k - h^*\|_2^2 
    	&\stackrel{\eqref{eq:second_moment_decomposition}}{=}& \EE_{Q^k}\|\hat g^k - \EE_{Q^k} \hat g^k\|_2^2 + \| \EE_{Q^k} \hat g^k - h^*\|_2^2\\
    	&\overset{\eqref{eq:distr_hat_g_moments1} }{=}&\EE_{Q^k}\|\hat g^k - g^k\|_2^2 + \|  g^k - h^*\|_2^2\\
    &\overset{\eqref{eq:distr_hat_g_moments1} }{=}& \frac{1}{n^2}\sum_{i=1}^n \Psi(\Delta_i^k) +  \|  g^k - h^*\|_2^2.
\end{eqnarray*}

Further, applying variance decomposition \eqref{eq:second_moment_decomposition}, we get
  \begin{eqnarray*}
    	\EE\left[  \|g^k - h^*\|_2^2 \;|\; x^k \right]
    	&\stackrel{\eqref{eq:second_moment_decomposition}}{=}&\EE \left[ \|g^k - \EE [g^k\;|\; x^k]\|_2^2 \; | \; x^k \right] +  \|\EE [g^k \;|\; x^k ] - h^*\|_2^2 \\
    	&\overset{\eqref{eq:hat_g_expectation}}{=}&\EE \left[ \|g^k - \nabla f(x^k)\|_2^2 \; | \; x^k \right] +  \| \nabla f(x^k)- h^*\|_2^2 \\
    	& \overset{\eqref{eq:bgud7t9gf}}{\leq} & \frac{\sigma^2}{n} + \|\nabla f(x^k) - h^*\|_2^2 .
    \end{eqnarray*}

Combining the two results, we get
\begin{eqnarray*}
    \EE[	\EE_{Q^k}\|\hat g^k - h^*\|_2^2 \;|\; x^k ]
    	\leq  \frac{1}{n^2} \sum_{i=1}^n \EE\left[ \Psi(\Delta_i^k) \;|\; x^k \right] + \frac{\sigma^2}{n} + \|\nabla f(x^k) - h^*\|_2^2.
\end{eqnarray*}
After applying full expectation, and using tower property, we get the result.
\end{itemize}

\end{proof}

\begin{lemma}\label{lem:distr_h_recurrence1}
    Let $x^*$ be a solution of \eqref{eq:main} and let $h_i^* = \nabla f_i(x^*)$ for $i=1,2,\dots,d$. For every $i$, we can estimate the first two moments of $h_i^{k+1}$ as
    \begin{align}
        \EE_{Q^k} h_i^{k+1} 
        &= (1 - \alpha)h^k_i + \alpha g_i^k, \nonumber\\
        \EE_{Q^k} \|h_i^{k+1} - h_i^*\|_2^2
        & = (1 - \alpha)\|h_i^k - h_i^*\|_2^2 + \alpha \|g_i^k - h_i^*\|_2^2\notag\\
        &\quad - \alpha \left( \|\Delta_i^k\|_2^2 - \alpha \sum\limits_{l=1}^m\|\Delta_i^k(l)\|_1 \|\Delta_i^k(l)\|_p  \right) . \label{eq:distr_h^(k+1)le}
    \end{align}
\end{lemma}

\begin{proof}
    Since \begin{equation}\label{eq:b87f9h8hf9}h_i^{k+1} = h_i^k + \alpha \hat \Delta_i^k\end{equation} and $\Delta_i^k = g_i^k - h_i^k$, in view of Lemma~\ref{lem:moments1} we have
\begin{equation}\label{eq:09h80hdf}
        \EE_{Q^k} h_i^{k+1}  \overset{\eqref{eq:b87f9h8hf9}}{=}   h_i^k + \alpha \EE_{Q^k} \hat \Delta_i^k 
         \stackrel{\eqref{eq:hat_v_moments1}}{=}  h_i^k + \alpha \Delta_i^k   
           =  (1 - \alpha)h_i^k + \alpha g_i^k,
\end{equation}
which establishes the first claim. Further, using $\|\Delta_i^k\|_2^2 = \sum\limits_{l=1}^m\|\Delta_i^k(l)\|_2^2$ we obtain
    \begin{eqnarray*}
        \EE_{Q^k} \|h_i^{k+1} - h_i^*\|_2^2 
        &\overset{\eqref{eq:second_moment_decomposition}}{=}& \|\EE_{Q^k} h_i^{k+1} - h_i^*\|_2^2 + \EE_{Q^k} \| h_i^{k + 1} - \EE_{Q^k} h_i^{k+1}\|_2^2 \nonumber \\
        &\overset{\eqref{eq:09h80hdf}+
        \eqref{eq:b87f9h8hf9}      }{=}& \|(1 - \alpha) h_i^k + \alpha g_i^k - h_i^*\|_2^2 + \alpha^2 \EE_{Q^k}\| \hat \Delta_i^k - \EE_{Q^k} \hat \Delta_i^k \|_2^2 \nonumber \\
        &\overset{\eqref{eq:hat_v_moments1}}{=}& \|(1 - \alpha) (h_i^k - h_i^*) + \alpha (g_i^k - h_i^*)\|_2^2 \notag\\
        &&\quad + \alpha^2 \sum\limits_{l=1}^m(\|\Delta_i^k(l)\|_1 \|\Delta_i^k(l)\|_p - \|\Delta_i^k(l)\|_2^2) \nonumber \\
        &\stackrel{\eqref{eq:variance_decompos} }{=}& (1 - \alpha)\| h_i^k - h_i^*\|_2^2 + \alpha \| g_i^k - h_i^*\|_2^2 - \alpha(1 - \alpha)\| \Delta_i^k\|_2^2\nonumber\\
        &&\quad + \alpha^2 \sum\limits_{l=1}^m(\|\Delta_i^k(l)\|_1 \|\Delta_i^k(l)\|_p) - \alpha^2\|\Delta_i^k\|_2^2\nonumber \\
&=&        (1 - \alpha)\| h_i^k - h_i^*\|_2^2 + \alpha \| g_i^k - h_i^*\|_2^2 \notag\\
		&&\quad + \alpha^2\sum\limits_{l=1}^m( \|\Delta_i^k(l)\|_1 \|\Delta_i^k(l)\|_p) - \alpha \|\Delta_i^k\|_2^2.\nonumber \\
    \end{eqnarray*}

\end{proof}

\begin{lemma}
	We have
	\begin{align}
		\EE\left[\|\hat g^k - h^*\|_2^2 \mid x^k\right] &\le \|\nabla f(x^k) - h^*\|_2^2 + \left(\frac{1}{\alpha_p} - 1\right)\frac{1}{n^2}\sumin \|\nabla f_i(x^k) - h_i^k\|_2^2\notag\\
		&\quad + \frac{\sigma^2}{\alpha_p n}. \label{eq:full_second_moment_of_hat_g}
	\end{align}
\end{lemma}
\begin{proof}
	Since $\alpha_p = \alpha_p(\max\limits_{l=1,\ldots,m}d_l)$ and $\alpha_p(d_l) = \inf\limits_{x\neq 0,x\in\R^{d_l}}\frac{\|x\|_2^2}{\|x\|_1\|x\|_p} $, we have for a particular choice of $x=\Delta_i^k(l)$ that $\alpha_p\le\alpha_p(d_l) \le \frac{\|\Delta_i^k(l)\|_2^2}{\|\Delta_i^k(l)\|_1\|\Delta_i^k(l)\|_p}$. Therefore,
	\begin{eqnarray*}
		\Psi(\Delta_i^k) &=& \sum\limits_{l=1}^m\Psi_l(\Delta_i^k) = \sum\limits_{l=1}^m(\|\Delta_i^k(l)\|_1\|\Delta_i^k(l)\|_\infty - \|\Delta_i^k(l)\|_2) \\
		&\le& \sum\limits_{l=1}^m \left(\frac{1}{\alpha_p}-1\right)\|\Delta_i^k(l)\|_2^2 = \left(\frac{1}{\alpha_p}-1\right)\|\Delta_i^k\|_2^2.
	\end{eqnarray*}		
	 This can be applied to~\eqref{eq:full_variance_of_mean_g1} in order to obtain
	\begin{eqnarray*}
		\EE\left[\|\hat g^k - h^* \|_2^2 \mid x^k\right] 
		&\le & \|\nabla f(x^k) - h^*\|_2^2 + \frac{1}{n^2}\sumin \EE\left[\Psi(\Delta_i^k)\mid x^k \right]  + \frac{\sigma^2}{n}\\
		&\le & \|\nabla f(x^k) - h^*\|_2^2 + \frac{1}{n^2}\sumin \left(\frac{1}{\alpha_p} - 1\right)\EE\left[\|\Delta_i^k\|_2^2\mid x^k \right]  + \frac{\sigma^2}{n}.
	\end{eqnarray*} 	
 	Note that for every $i$ we have $\EE\left[\Delta_i^k\mid x^k\right] = \EE\left[g_i^k - h_i^k\mid x\right] = \nabla f_i(x^k) - h_i^k$, so
 	\begin{eqnarray*}
 		\EE \left[\|\Delta_i^k\|_2^2 \mid x^k\right] 
 		&\overset{\eqref{eq:second_moment_decomposition}}{=}& \|\nabla f_i(x^k) - h_i^k\|_2^2 +  \EE\left[\|g_i^k - \nabla f_i(x^k)\|_2^2 \mid x^k\right] \\
 		&\le& \|\nabla f_i(x^k) - h_i^k\|_2^2 +  \sigma_i^2 .
 	\end{eqnarray*}
 	Summing the produced bounds, we get the claim.
\end{proof}

\section{Proof of Theorem~\ref{thm:DIANA-strongly_convex}}

\begin{proof}
Note that $x^*$ is a solution of \eqref{eq:main} if and only if $x^* = \prox_{\gamma R}(x^*-\gamma h^*)$ (this holds for any $\gamma>0$).  Using this identity together with the nonexpansiveness of the proximaloperator, we shall  bound the first term of the Lyapunov function:
    \begin{eqnarray*}
        \EE_{Q^k}\|x^{k+1} - x^*\|_2^2 
        &=& \EE_{Q^k}\|\proxR(x^k - \gamma \hat g^k) - \proxR(x^* - \gamma h^*)\|_2^2 \\
        &\overset{\eqref{eq:nonexpansive}}{\le}& \EE_{Q^k}\|x^k - \gamma \hat g^k - (x^* - \gamma h^*)\|_2^2 \\
        &=& \|x^k - x^*\|_2^2 - 2\gamma \EE_{Q^k} \<\hat g^k - h^*, x^k - x^*> + \gamma^2\EE_{Q^k}\|\hat g^k - h^*\|_2^2 \\
        &\overset{\eqref{eq:distr_hat_g_moments1}}{=}& \|x^k - x^*\|_2^2 - 2\gamma \<g^k - h^*, x^k - x^*> + \gamma^2\EE_{Q^k} \|\hat g^k - h^*\|_2^2.
	\end{eqnarray*}
Next, taking conditional expectation on both sides of the above inequality, and using \eqref{eq:hat_g_expectation}, we get
\begin{eqnarray*}
       \EE\left[ \EE_{Q^k}\|x^{k+1} - x^*\|_2^2 \;|\; x^k \right]
        &\leq & \|x^k - x^*\|_2^2 - 2\gamma \< \nabla f(x^k) - h^*, x^k - x^*>\\
        &&\quad + \gamma^2\EE\left[ \EE_{Q^k} \|\hat g^k - h^*\|_2^2 \;|\; x^k\right].
	\end{eqnarray*}
	Taking full expectation on both sides of the above inequality, and  applying the tower property and Lemma~\ref{lem:3in1} leads to
	\begin{eqnarray}
       \EE \|x^{k+1} - x^*\|_2^2 
        &\leq & \EE \|x^k - x^*\|_2^2 - 2\gamma \EE \< \nabla f(x^k) - h^*, x^k - x^*> + \gamma^2\EE  \|\hat g^k - h^*\|_2^2 \notag \\
        &\overset{\eqref{eq:full_variance_of_mean_g1}}{\leq} & \EE \|x^k - x^*\|_2^2 - 2\gamma \EE \< \nabla f(x^k) - h^*, x^k - x^*> \notag \\
        && \quad + \gamma^2 \EE\|\nabla f(x^k) - h^*\|_2^2 + \frac{\gamma^2}{n^2}\sum_{i=1}^n \left(\EE \Psi(\Delta_i^k)\right) + \frac{\gamma^2 \sigma^2}{n} \notag \\
        &\leq & 
        \EE \|x^k - x^*\|_2^2 - 2\gamma \EE \< \nabla f(x^k) - h^*, x^k - x^*> \notag \\
        && \quad + \frac{\gamma^2}{n} \sumin \EE \|\nabla f_i(x^k) - h_i^*\|_2^2\notag\\
        &&\quad  + \frac{\gamma^2}{n^2}\sum_{i=1}^n\left(\EE \Psi(\Delta_i^k)\right) + \frac{\gamma^2 \sigma^2}{n},
        \label{eq:buf89gh38bf98}
	\end{eqnarray}
	where the last inequality follows from the identities $\nabla f(x^k) = \frac{1}{n}\sumin f_i(x^k)$, $h^* = \frac{1}{n}\sumin h_i^*$ and an application of Jensen's inequality.

Averaging over the identities \eqref{eq:distr_h^(k+1)le} for $i=1,2,\dots,n$ in Lemma~\ref{lem:distr_h_recurrence1}, we get
\begin{eqnarray*}
	\frac{1}{n}\sumin \EE_{Q^k} \|h_i^{k+1} - h_i^*\|_2^2
         &=& \frac{1 - \alpha}{n}\sumin\|h_i^k - h_i^*\|_2^2 +  \frac{\alpha}{n}\sumin\|g_i^k - h_i^*\|_2^2\\
         &&\quad  -  \frac{\alpha}{n} \sumin \left(\|\Delta_i^k\|_2^2 -\alpha \sum\limits_{l=1}^m \|\Delta_i^k(l)\|_1 \|\Delta_i^k(l)\|_p  \right) .
\end{eqnarray*}
    Applying expectation to both sides, and using the tower property, we get
   \begin{eqnarray}
       \frac{1}{n}\sumin \EE \|h_i^{k+1} - h_i^*\|_2^2
         &=& \frac{1 - \alpha}{n}\sumin \EE \|h_i^k - h_i^*\|_2^2 +  \frac{\alpha}{n}\sumin \EE \|g_i^k - h_i^*\|_2^2 \notag\\
         &&\quad -  \frac{\alpha}{n} \sumin \EE \left[\|\Delta_i^k\|_2^2 -\alpha \sum\limits_{l=1}^m \|\Delta_i^k(l)\|_1 \|\Delta_i^k(l)\|_p  \right]. \label{eq:nbi987fg98bf9hbf}
\end{eqnarray}

Since
\begin{equation*}% \label{eq:b897fg98fss}
		\EE [ \|g_i^k - h_i^*\|_2^2 \;|\; x^k]
		\stackrel{\eqref{eq:second_moment_decomposition}}{=} \|\nabla f_i(x^k) - h_i^*\|_2^2 + \EE [\|g_i^k - \nabla f_i(x^k)\|_2^2\;|\; x^k ] \stackrel{\eqref{eq:bounded_noise}}{\le} \|\nabla f_i(x^k) - h_i^*\|_2^2 + \sigma_i^2,
	\end{equation*}
the second term on the right hand side of \eqref{eq:nbi987fg98bf9hbf} can be bounded above as
	\begin{equation}\label{eq:b897fg98fss}
		\EE \|g_i^k - h_i^*\|_2^2 
		\leq  \EE \|\nabla f_i(x^k) - h_i^*\|_2^2 + \sigma_i^2.
	\end{equation}
Plugging \eqref{eq:b897fg98fss} into \eqref{eq:nbi987fg98bf9hbf} leads to the estimate
   \begin{eqnarray}
       \frac{1}{n}\sumin \EE \|h_i^{k+1} - h_i^*\|_2^2
         &\leq& \frac{1 - \alpha}{n}\sumin \EE \|h_i^k - h_i^*\|_2^2 +  \frac{\alpha}{n}\sumin \EE \|\nabla f_i(x^k) - h_i^*\|_2^2 + \alpha \sigma^2  \notag \\
         &&\quad  -  \frac{\alpha}{n} \sumin \EE \left[\|\Delta_i^k\|_2^2 -\alpha \sum\limits_{l=1}^m \|\Delta_i^k(l)\|_1 \|\Delta_i^k(l)\|_p  \right]  . \label{eq:nbi9bh98sgs}
\end{eqnarray}
	
Adding \eqref{eq:buf89gh38bf98} with the $c\gamma^2$ multiple of \eqref{eq:nbi9bh98sgs}, we get an upper bound one the Lyapunov function:
\begin{eqnarray}
\EE V^{k+1} & \leq & \EE \|x^k-x^*\|_2^2  +  \frac{(1-\alpha) c\gamma^2}{n}\sumin \EE\|h_i^k - h_i^*\|_2^2 \notag \\ 
&& \quad +  \frac{\gamma^2 ( 1 + \alpha c) }{n}\sum_{i=1}^n\EE\|\nabla f_i(x^k) - h_i^*\|_2^2   - 2\gamma \EE\<\nabla f(x^k) - h^*, x^k - x^*>  \notag \\
&& \quad  + \frac{\gamma^2}{n^2} \sumin\sum\limits_{l=1}^m \EE \left[   T_i^k(l) \right]  + (nc\alpha + 1)\frac{\gamma^2\sigma^2}{n}, \label{eq:bd7g98gdh8d}
\end{eqnarray}
where $T_i^k(l) \eqdef \left(\|\Delta_i^k(l)\|_1 \|\Delta_i^k(l)\|_p - \|\Delta_i^k(l)\|_2^2\right)  - n \alpha c  \left(\|\Delta_i^k(l)\|_2^2 -\alpha  \|\Delta_i^k(l)\|_1 \|\Delta_i^k(l)\|_p  \right)$.	

We now claim that due to our choice of $\alpha$ and $c$, we have $ T_i^k(l)\leq 0$ for all $\Delta_i^k(l)\in \R^{d_l}$, which means that we can bound this term away by zero. Indeed, note that $T_k^i(l) = 0$ for $\Delta_i^k(l) = 0$. If $\Delta_i^k(l) \neq 0$, then $T_k^i(l) \leq 0$ can be equivalently written as 
\[\frac{1 + nc \alpha^2}{1 + nc \alpha} \leq \frac{\|\Delta_i^k(l)\|_2^2}{\|\Delta_i^k(l)\|_1\|\Delta_i^k(l)\|_p}.\] 
However, this inequality holds since in view of the first inequality in \eqref{eq:cond1} and the definitions of $\alpha_p$ and $\alpha_p(d_l)$, we have
\[\frac{1 + nc \alpha^2}{1 + nc \alpha} \overset{\eqref{eq:cond1} }{\leq} \alpha_p \le \alpha_p(d_l) \overset{\eqref{eq:alpha_p}}{=} \inf_{x\neq 0,x\in\R^{d_l}} \frac{\|x\|_2^2}{\|x\|_1\|x\|_p} \leq  \frac{\|\Delta_i^k(l)\|_2^2}{\|\Delta_i^k(l)\|_1\|\Delta_i^k(l)\|_p}.\] 

Therefore, from \eqref{eq:bd7g98gdh8d} we get

\begin{eqnarray}
 \EE V^{k+1} & \leq & \EE \|x^k-x^*\|_2^2  +  \frac{(1-\alpha) c\gamma^2}{n}\sumin \EE\|h_i^k - h_i^*\|_2^2 \notag \\ 
&& \qquad +  \frac{\gamma^2 ( 1 + \alpha c) }{n}\sum_{i=1}^n\EE\|\nabla f_i(x^k) - h_i^*\|_2^2   - 2\gamma \EE\<\nabla f(x^k) - h^*, x^k - x^*>  \notag \\
&& \qquad + (nc\alpha + 1)\frac{\gamma^2\sigma^2}{n}. \label{eq:bd7g98gdh8d11}
\end{eqnarray}

The next trick is to split $\nabla f(x^k)$ into the average of $\nabla f_i(x^k)$ in order to apply strong convexity of each term:
	\begin{eqnarray}
		\EE\<\nabla f(x^k) - h^*, x^k - x^*> 
		&=& \frac{1}{n}\sumin \EE\<\nabla f_i(x^k) - h_i^*, x^k - x^*>\notag\\
		&\stackrel{\eqref{eq:str_convexity}}{\ge} &  \frac{1}{n} \sumin \EE \left( \frac{\mu L}{\mu + L} \|x^k - x^*\|_2^2 +   \frac{1}{\mu + L} \|\nabla f_i(x^k) - h_i^*\|_2^2 \right)\notag\\
		&=& \frac{\mu L}{\mu + L}\EE\|x^k - x^*\|_2^2\notag\\
		&&\quad  + \frac{1}{\mu + L}\frac{1}{n}\sumin \EE\|\nabla f_i(x^k) - h_i^*\|_2^2\label{eq:inner_product_splitting}.
	\end{eqnarray}

Plugging these estimates into \eqref{eq:bd7g98gdh8d11}, we obtain 
\begin{eqnarray}
\EE V^{k+1} & \leq & \left(1 - \frac{2\gamma \mu L}{\mu+L}\right) \EE \|x^k-x^*\|_2^2  +  \frac{(1-\alpha) c\gamma^2}{n}\sumin \EE\|h_i^k - h_i^*\|_2^2 \notag \\ 
&& \quad +  \left(\gamma^2 ( 1 + \alpha c) -\frac{2\gamma}{\mu+L} \right)\frac{1}{n}\sum_{i=1}^n\EE\|\nabla f_i(x^k) - h_i^*\|_2^2\notag\\
&&\quad + (nc\alpha + 1)\frac{\gamma^2\sigma^2}{n}. \label{eq:bd7nb98sgsdd}
\end{eqnarray}

Notice that in view of the second inequality in \eqref{eq:cond2}, we have $ \gamma^2 ( 1 + \alpha c) -\frac{2\gamma}{\mu+L} \leq 0$. Moreover, since $f_i$ is $\mu$--strongly convex, we have
$\mu \|x^k-x^*\|_2^2 \leq \langle \nabla f_i(x^k) - h_i^*, x^k -x^* \rangle$. Applying the Cauchy-Schwarz inequality to further  bound the right hand side, we get the inequality $\mu \|x^k-x^*\|_2 \leq \|\nabla f_i(x^k) - h_i^*\|_2$. Using these observations, we can get rid of the term on the second line of \eqref{eq:bd7nb98sgsdd} and absorb it with the first term, obtaining
\begin{eqnarray}
	\EE V^{k+1}  &\leq&  \left(1 - 2\gamma \mu + \gamma^2 \mu^2 + c\alpha \gamma^2 \mu^2 \right) \EE \|x^k-x^*\|_2^2 \notag\\
	&&\quad +  \frac{(1-\alpha) c\gamma^2}{n}\sumin \EE\|h_i^k - h_i^*\|_2^2  + (nc\alpha + 1)\frac{\gamma^2\sigma^2}{n}. \label{eq:bd7n98sh90shdw}
\end{eqnarray}
It follows from the second inequality in \eqref{eq:cond2} that $1 - 2\gamma \mu + \gamma^2 \mu^2 + c\alpha \gamma^2 \mu^2 \leq 1 - \gamma \mu$. Moreover, the first inequality in \eqref{eq:cond2} implies that $1-\alpha \leq 1-\gamma \mu$. Consequently, from \eqref{eq:bd7n98sh90shdw} we obtain the recursion
\[\EE V^{k+1}  \leq (1-\gamma \mu) \EE V^k +(nc\alpha + 1)\frac{\gamma^2\sigma^2}{n}. \] 
Finally, unrolling the recurrence leads to
    \begin{align*}
    		\EE V^k 
    		&\le (1 - \gamma\mu)^k V^0 + \sum\limits_{l=0}^{k-1}(1-\gamma\mu)^l\gamma^2(1+nc\alpha)\frac{\sigma^2}{n} \\
    		&\le (1 - \gamma\mu)^k V^0 + \sum\limits_{l=0}^{\infty}(1-\gamma\mu)^l\gamma^2(1+nc\alpha)\frac{\sigma^2}{n} \\
    		&= (1 - \gamma\mu)^k V^0 + \frac{\gamma}{\mu}(1+nc\alpha)\frac{\sigma^2}{n}.
    \end{align*}
\end{proof}

\section{Proof of Corollary~\ref{cor:DIANA-strong-convex}}

\iffalse
\begin{corollary}
Let $\kappa = \tfrac{L+\mu}{2 \mu}$, $\sqrt{c} =  \tfrac{\sqrt{\kappa}(\alpha_p^{-1}-1)}{n} + \tfrac{\alpha_p^{-1}}{\sqrt{\kappa}}$, $\alpha = \tfrac{1}{\sqrt{\kappa c}}$, and $\gamma = \tfrac{2}{(L+\mu)(1+c \alpha)}$. Then the conditions \eqref{eq:cond1} and \eqref{eq:cond2} are satisfied, and the leading term in the iteration complexity bound is equal to
\begin{equation} \frac{1}{\gamma \mu} = \kappa  + \frac{\kappa}{n} (\alpha_p^{-1}-1) + \alpha_p^{-1}.\end{equation} 
This is a decreasing function of $p$.  Hence, from iteration complexity perspective, $p=+\infty$ is the optimal choice.
\end{corollary}
\fi

\begin{corollary} Let $\kappa = \tfrac{L}{\mu}$, $\alpha = \frac{\alpha_p}{2}$, $c = \frac{4(1-\alpha_p)}{n\alpha_p^2}$, and $\gamma = \min\left\{\frac{\alpha}{\mu}, \tfrac{2}{(L+\mu)(1+c \alpha)}\right\}$. Then the conditions \eqref{eq:cond1} and \eqref{eq:cond2} are satisfied, and the leading term in the iteration complexity bound is equal to
\begin{equation} \frac{1}{\gamma \mu} = \max\left\{\frac{2}{\alpha_p}, (\kappa+1)\left(\frac{1}{2} - \frac{1}{n} + \frac{1}{n\alpha_p}\right)\right\}.\end{equation} 
This is a decreasing function of $p$.  Hence, from iteration complexity perspective, $p=+\infty$ is the optimal choice.
\end{corollary}
\begin{proof}
	Condition \eqref{eq:cond2} is satisfied since $\gamma = \min\left\{\frac{\alpha}{\mu}, \tfrac{2}{(L+\mu)(1+c \alpha)}\right\}$. Now we check that \eqref{eq:cond1} is also satisfied:
	\begin{eqnarray*}
		\frac{1+nc\alpha^2}{1+nc\alpha} \frac{1}{\alpha_p} &=& \frac{1 + n\cdot\frac{4(1-\alpha_p)}{n\alpha_p^2}\cdot\frac{\alpha_p^2}{4}}{1 + n\cdot\frac{4(1-\alpha_p)}{n\alpha_p^2}\cdot\frac{\alpha_p}{2}}\cdot\frac{1}{\alpha_p}\\
		&=& \frac{2 - \alpha_p}{\alpha_p + 2(1-\alpha_p)}\\
		&=& 1.
	\end{eqnarray*}
	Since $\alpha = \frac{\alpha_p}{2}$ and $c = \frac{4(1-\alpha_p)}{n\alpha_p^2}$ we have
	\[
		1+\alpha c = 1 + \frac{2(1-\alpha_p)}{n\alpha_p} = 1 - \frac{2}{n} + \frac{2}{n\alpha_p}
	\]
	and, therefore,
	\[
		\frac{1}{\gamma\mu} = \max\left\{\frac{1}{\alpha}, \frac{L+\mu}{2\mu}(1+c\alpha)\right\} = \max\left\{\frac{2}{\alpha_p}, (\kappa+1)\left(\frac{1}{2} - \frac{1}{n} + \frac{1}{n\alpha_p}\right)\right\},	
	\]
	which is a decreasing function of $p$, because $\alpha_p$ increases when $p$ increases.
\end{proof}

\section{Strongly convex case: decreasing stepsize}\label{sec:DIANA-decreasing-stepsize-appendix}

\begin{lemma}
\label{lem:sgd}
	Let a sequence $\{a^k\}_k$ satisfy inequality $a^{k+1}\le (1 - \gamma^k\mu) a^k + (\gamma^k)^2 N$ for any positive $\gamma^k \le \gamma_0$ with some constants $\mu > 0, N>0, \gamma_0 > 0$. Further, let $\theta \ge \frac{2}{\gamma_0} $ and take $C$ such that $N\le \frac{\mu\theta}{4}C$ and $a_0\le C$. Then, it holds
	\begin{align*}
		a^k \le \frac{ C}{\frac{\mu}{\theta} k + 1}
	\end{align*}
	if we set $\gamma^k=\frac{2}{\mu k + \theta}$.
\end{lemma}
\begin{proof}
	We will show the inequality for $a^k$ by induction. Since inequality $a_0\le C$ is one of our assumptions, we have the initial step of the induction. To prove the inductive step, consider
	\begin{align*}
		a^{k+1} 
		\le (1 - \gamma^k\mu) a^k + (\gamma^k)^2 N 
		\le \left(1 - \frac{2\mu}{\mu k + \theta} \right) \frac{\theta C}{\mu k + \theta} + \theta\mu\frac{C}{(\mu k + \theta)^2}.
	\end{align*}
	To show that the right-hand side is upper bounded by $\frac{\theta C}{\mu(k + 1) + \theta}$, one needs to have, after multiplying both sides by $(\mu k + \theta)(\mu k + \mu + \theta)(\theta C)^{-1}$,
	\begin{align*}
		\left(1 - \frac{2\mu}{\mu k + \theta} \right) (\mu k + \mu + \theta) + \mu\frac{\mu k + \mu + \theta}{\mu k + \theta} \le \mu k + \theta,
	\end{align*}
	which is equivalent to
	\begin{align*}
		\mu - \mu\frac{\mu k + \mu + \theta}{\mu k + \theta} \le 0.
	\end{align*}
	The last inequality is trivially satisfied for all $k \ge 0$.
\end{proof}

\begin{theorem}\label{th:str_cvx_decr_step}
    Assume that $f$ is $L$-smooth, $\mu$-strongly convex and we have access to its gradients with bounded noise. Set $\gamma^k = \frac{2}{\mu k + \theta}$ with some $\theta \ge 2\max\left\{\frac{\mu}{\alpha}, \frac{(\mu+L)(1+c\alpha)}{2} \right\}$ for some numbers $\alpha > 0$ and $c > 0$ satisfying $\frac{1+nc\alpha^2}{1+nc\alpha} \le \alpha_p$. After $k$ iterations of {\tt DIANA} we have
    \begin{align*}
        \EE V^k \le \frac{1}{\eta k+1}\max\left\{ V^0, 4\frac{(1+nc\alpha)\sigma^2}{n\theta\mu} \right\},
    \end{align*}
    where $\eta\eqdef \frac{\mu}{\theta}$, $V^k=\|x^k - x^*\|_2^2 + \frac{c\gamma^k}{n}\sumin\|h_i^0 - h_i^*\|_2^2$ and $\sigma$ is the standard deviation of the gradient noise.
\end{theorem}
\begin{proof}
	To get a recurrence, let us recall an upper bound we have proved before:
    \[
        \EE V^{k+1}\le (1 - \gamma^k\mu)\EE V^k + (\gamma^k)^2(1+nc\alpha)\frac{\sigma^2}{n}.
    \]
    Having that, we can apply Lemma~\ref{lem:sgd} to the sequence $\EE V^k$. The constants for the lemma are: $N = (1 + nc\alpha)\frac{\sigma^2}{n}$, $C=\max\left\{V^0, 4\frac{(1+nc\alpha)\sigma^2}{n\theta\mu}\right\}$, and $\mu$ is the strong convexity constant.
\end{proof}
\begin{corollary}
	If we choose $\alpha = \frac{\alpha_p}{2}$, $c = \frac{4(1-\alpha_p)}{n\alpha_p^2}$,  $\theta=2\max\left\{\frac{\mu}{\alpha}, \frac{\left(\mu+L\right)\left(1 + c\alpha\right)}{2} \right\} = \frac{\mu}{\alpha_p}\max\left\{4, \frac{2(\kappa + 1)}{n} + \frac{(\kappa+1)(n-2)}{ n}\alpha_p\right\}$, then there are three regimes:
	\begin{enumerate}
		\item[1)] if $1 = \max\left\{1,\frac{\kappa}{n},\kappa\alpha_p\right\}$, then $\theta = \Theta\left(\frac{\mu}{\alpha_p}\right)$ and to achieve $\EE V^k\le \varepsilon$ we need at most $O\left( \frac{1}{\alpha_p}\max\left\{V^0, \frac{(1-\alpha_p)\sigma^2}{n\mu^2} \right\}\frac{1}{\varepsilon} \right)$ iterations;
		\item[2)] if $\frac{\kappa}{n} = \max\left\{1,\frac{\kappa}{n},\kappa\alpha_p\right\}$, then $\theta = \Theta\left(\frac{L}{n\alpha_p}\right)$ and to achieve $\EE V^k\le \varepsilon$ we need at most $O\left( \frac{\kappa}{n\alpha_p}\max\left\{V^0, \frac{(1-\alpha_p)\sigma^2}{\mu L} \right\}\frac{1}{\varepsilon} \right)$ iterations;
		\item[3)] if $\kappa\alpha_p = \max\left\{1,\frac{\kappa}{n},\kappa\alpha_p\right\}$, then $\theta = \Theta\left(L\right)$ and to achieve $\EE V^k\le \varepsilon$ we need at most $O\left( \kappa\max\left\{V^0, \frac{(1-\alpha_p)\sigma^2}{\mu Ln\alpha_p} \right\}\frac{1}{\varepsilon} \right)$ iterations.
	\end{enumerate}		
\end{corollary}
\begin{proof}
	First of all, let us show that $c = \frac{4(1-\alpha_p)}{n\alpha_p^2}$ and $\alpha$ satisfy inequality $\frac{1+nc\alpha^2}{1+nc\alpha} \le \alpha_p$:
	\begin{eqnarray*}
		\frac{1+nc\alpha^2}{1+nc\alpha} \frac{1}{\alpha_p} &=& \frac{1 + n\cdot\frac{4(1-\alpha_p)}{n\alpha_p^2}\cdot\frac{\alpha_p^2}{4}}{1 + n\cdot\frac{4(1-\alpha_p)}{n\alpha_p^2}\cdot\frac{\alpha_p}{2}}\cdot\frac{1}{\alpha_p}\\
		&=& \frac{2 - \alpha_p}{\alpha_p + 2(1-\alpha_p)}\\
		&=& 1.
	\end{eqnarray*}
	Moreover, since
	\[
		1+c\alpha = 1 + \frac{2(1-\alpha_p)}{n\alpha_p} = \frac{2 + (n-2)\alpha_P}{n\alpha_p}
	\]
	we can simplify the definition of $\theta$:
	\begin{eqnarray*}
		\theta &= &2\max\left\{\frac{\mu}{\alpha}, \frac{\left(\mu+L\right)\left(1 + c\alpha\right)}{2} \right\}\\
		 &=& \frac{\mu}{\alpha_p}\max\left\{4, \frac{2(\kappa + 1)}{n} + \frac{(\kappa+1)(n-2)}{ n}\alpha_p\right\}\\
		 &=& \Theta\left(\frac{\mu}{\alpha_p}\max\left\{1,\frac{\kappa}{n},\kappa\alpha_p\right\}\right).
	\end{eqnarray*}
	Using Theorem~\ref{th:str_cvx_decr_step}, we get in the case:
	\begin{enumerate}
		\item[1)] if $1 = \max\left\{1,\frac{\kappa}{n},\kappa\alpha_p\right\}$, then $\theta = \Theta\left(\frac{\mu}{\alpha_p}\right)$, $\eta = \Theta\left(\alpha_p\right)$, $\frac{4(1+nc\alpha)\sigma^2}{n\theta\mu} = \Theta\left(\frac{(1-\alpha_p)\sigma^2}{n\mu^2}\right)$ and to achieve $\EE V^k\le \varepsilon$ we need at most $O\left( \frac{1}{\alpha_p}\max\left\{V^0, \frac{(1-\alpha_p)\sigma^2}{n\mu^2} \right\}\frac{1}{\varepsilon} \right)$ iterations;
		\item[2)] if $\frac{\kappa}{n} = \max\left\{1,\frac{\kappa}{n},\kappa\alpha_p\right\}$, then $\theta = \Theta\left(\frac{L}{n\alpha_p}\right)$, $\eta = \Theta\left(\frac{\alpha_pn}{\kappa}\right)$, $\frac{4(1+nc\alpha)\sigma^2}{n\theta\mu} = \Theta\left(\frac{(1-\alpha_p)\sigma^2}{\mu L}\right)$ and to achieve $\EE V^k\le \varepsilon$ we need at most $O\left( \frac{\kappa}{n\alpha_p}\max\left\{V^0, \frac{(1-\alpha_p)\sigma^2}{\mu L} \right\}\frac{1}{\varepsilon} \right)$ iterations;
		\item[3)] if $\kappa\alpha_p = \max\left\{1,\frac{\kappa}{n},\kappa\alpha_p\right\}$, then $\theta = \Theta\left(L\right)$, $\eta = \Theta\left(\frac{1}{\kappa}\right)$, $\frac{4(1+nc\alpha)\sigma^2}{n\theta\mu} = \Theta\left(\frac{(1-\alpha_p)\sigma^2}{\mu L n\alpha_p}\right)$ and to achieve $\EE V^k\le \varepsilon$ we need at most $O\left( \kappa\max\left\{V^0, \frac{(1-\alpha_p)\sigma^2}{\mu Ln\alpha_p} \right\}\frac{1}{\varepsilon} \right)$ iterations.
	\end{enumerate}		
\end{proof}

\section{Nonconvex analysis}

\begin{theorem}
    %Assume $R$ is such that exists a closed convex set $\cX$ satisfying 1) $\forall z\in\R^n$ $\proxR(z)\in\cX$ and 2) $\forall z\in\cX$ $z = \proxR(z)$ (e.g. indicator function of $\cX$).   %the indicator function of a closed convex set $\cX$, 
	Assume that $R$ is constant  and Assumption~\ref{as:almost_identical data} holds.    
    Also assume that $f$ is $L$-smooth, stepsizes $\alpha>0$ and $\gamma^k=\gamma>0$ and parameter $c>0$ satisfying $\frac{1 + n c \alpha^2}{1 + n c \alpha}   \leq \alpha_p,$ $\gamma \le \frac{2}{L(1+2c\alpha)}$ and $\overline x^k$ is chosen randomly from $\{x^0,\dotsc, x^{k-1} \}$. Then
    \begin{align*}
        \EE \|\nabla f(\overline x^k)\|_2^2 \le \frac{2}{k}\frac{\Lambda^0}{\gamma(2 - L\gamma - 2c\alpha L \gamma)} + \frac{(1+2cn\alpha)L\gamma}{2 - L\gamma - 2c\alpha L \gamma}\frac{\sigma^2}{n} + \frac{4c\alpha L\gamma \zeta^2}{2-L\gamma -2c\alpha L\gamma},
    \end{align*}
    where $\Lambda^k\eqdef  f(x^k) - f^* + c\tfrac{L\gamma^2}{2}\frac{1}{n}\sumin \|h_i^{k}-h_i^*\|_2^2$.
\end{theorem}
\begin{proof}
	The assumption that $R$ is constant implies that $x^{k+1} = x^k - \gamma\hat{g}^k$ and $h^* = 0$.    
    Moreover, by smoothness of $f$
    \begin{eqnarray*}
        \EE f(x^{k+1}) 
        &\stackrel{\eqref{eq:smoothness_functional}}{\le}& \EE f(x^k) + \EE\< \nabla f(x^k), x^{k+1} - x^k> + \frac{L}{2}\EE\|x^{k+1} - x^k\|_2^2 \\
        &\stackrel{\eqref{eq:distr_hat_g_moments1}}{\le}& \EE f(x^k) - \gamma \EE\|\nabla f(x^k)\|_2^2 + \frac{L\gamma^2}{2}\EE\|\hat g^k\|_2^2 \\
        &\stackrel{\eqref{eq:full_variance_of_mean_g1}}{\le}& \EE f(x^k) - \left(\gamma - \frac{L\gamma^2}{2}\right)\EE\|\nabla f(x^k)\|_2^2 + \frac{L \gamma^2}{2} \frac{1}{n^2}\sumin \EE\left[\Psi(\Delta_i^k)\right]\\
        &&\quad + \frac{L\gamma^2}{2n^2}\sumin\sigma_i^2.
    \end{eqnarray*}
    Denote $\Lambda^k\eqdef  f(x^k) - f^* + c\tfrac{L\gamma^2}{2}\frac{1}{n}\sumin \|h_i^{k}-h_i^*\|_2^2$. Due to Assumption~\ref{as:almost_identical data} we can rewrite the equation~\eqref{eq:distr_h^(k+1)le} after summing it up for $i=1,\ldots,n$ in the following form
	\begin{eqnarray*}
		\frac{1}{n}\sum\limits_{i=1}^n\EE\left[\|h_i^{k+1}-h_i^*\|_2^2 \mid x^k\right] &\le& \frac{1-\alpha}{n}\sum\limits_{i=1}^n\|h_i^k - h_i^*\|_2^2+ \frac{\alpha}{n}\sum\limits_{i=1}^n\EE\left[\|g_i^k - h_i^*\|_2^2\mid x^k\right]\\
		&&\quad - \alpha\left(\|\Delta_i^k\|_2^2 - \alpha\sum\limits_{l=1}^m\|\Delta_i^k(l)\|_1\|\Delta_i^k(l)\|_p\right)\\
		&\le& \frac{1-\alpha}{n}\sum\limits_{i=1}^n\|h_i^k - h_i^*\|_2^2\\
		&&\quad + \frac{2\alpha}{n}\sum\limits_{i=1}^n\EE\left[\|g_i^k\|_2^2\mid x^k\right] + \frac{2\alpha}{n}\sum\limits_{i=1}^n\|h_i^* - \underbrace{h^*}_{0}\|_2^2\\
		&&\quad - \alpha\left(\|\Delta_i^k\|_2^2 - \alpha\sum\limits_{l=1}^m\|\Delta_i^k(l)\|_1\|\Delta_i^k(l)\|_p\right)\\
		&\overset{\eqref{as:almost_identical data}}{\le}& \frac{1-\alpha}{n}\sum\limits_{i=1}^n\|h_i^k - h_i^*\|_2^2 + \frac{2\alpha}{n}\sum\limits_{i=1}^n\|\nabla f_i(x^k)\|_2^2\\
		&&\quad  + \frac{2\alpha}{n}\sum\limits_{i=1}^n\sigma_i^2 + 2\alpha\zeta^2\\
		&&\quad - \alpha\left(\|\Delta_i^k\|_2^2 - \alpha\sum\limits_{l=1}^m\|\Delta_i^k(l)\|_1\|\Delta_i^k(l)\|_p\right)\\
		&\overset{\eqref{as:almost_identical data} + \eqref{eq:second_moment_decomposition}}{\le}& \frac{1-\alpha}{n}\sum\limits_{i=1}^n\|h_i^k - h_i^*\|_2^2 + 2\alpha\|\nabla f(x^k)\|_2^2\\
		&&\quad  + 2\alpha\sigma^2 + 4\alpha\zeta^2\\
		&&\quad - \alpha\left(\|\Delta_i^k\|_2^2 - \alpha\sum\limits_{l=1}^m\|\Delta_i^k(l)\|_1\|\Delta_i^k(l)\|_p\right).
	\end{eqnarray*}	    
	If we add it the bound above, we  get
    \begin{align*}
        \EE\Lambda^{k+1}
        &= \EE f(x^{k+1}) - f^* + c\frac{L\gamma^2}{2}\frac{1}{n}\sumin\EE \|h_i^{k+1}-h_i^*\|_2^2\\
        &\le \EE f(x^k) - f^* - \gamma \left( 1 - \frac{L\gamma}{2} - c\alpha L\gamma \right)\EE \|\nabla f(x^k)\|_2^2 \\
        &\quad + (1 - \alpha)c\frac{L\gamma^2}{2}\frac{1}{n}\sumin \EE\|h_i^{k}-h_i^*\|_2^2 \\
        &\quad + \frac{L \gamma^2}{2} \frac{1}{n^2}\sumin\sum_{l=1}^m \EE\left[T_i^k(l)\right] + (1+2cn\alpha)\frac{L\gamma^2}{2}\frac{\sigma^2}{n} + 2c\alpha L\gamma^2\zeta^2,
    \end{align*}
    where $T_i^k(l) \eqdef (\|\Delta_i^k(l)\|_1\|\Delta_i^k(l)\|_p - \|\Delta_i^k(l)\|_2^2) - nc\alpha(\|\Delta_i^k(l)\|_2^2 - \alpha\|\Delta_i^k(l)\|_1\|\Delta_i^k(l)\|_p)$.
    
    We now claim that due to our choice of $\alpha$ and $c$, we have $ T_i^k(l)\leq 0$ for all $\Delta_i^k(l)\in \R^{d_l}$, which means that we can bound this term away by zero. Indeed, note that $T_k^i(l) = 0$ for $\Delta_i^k(l) = 0$. If $\Delta_i^k(l) \neq 0$, then $T_k^i(l) \leq 0$ can be equivalently written as 
\[\frac{1 + nc \alpha^2}{1 + nc \alpha} \leq \frac{\|\Delta_i^k(l)\|_2^2}{\|\Delta_i^k(l)\|_1\|\Delta_i^k(l)\|_p}.\] 
However, this inequality holds since in view of the first inequality in \eqref{eq:cond1} and the definitions of $\alpha_p$ and $\alpha_p(d_l)$, we have
\[\frac{1 + nc \alpha^2}{1 + nc \alpha} \overset{\eqref{eq:cond1} }{\leq} \alpha_p \le \alpha_p(d_l) \overset{\eqref{eq:alpha_p}}{=} \inf_{x\neq 0,x\in\R^{d_l}} \frac{\|x\|_2^2}{\|x\|_1\|x\|_p} \leq  \frac{\|\Delta_i^k(l)\|_2^2}{\|\Delta_i^k(l)\|_1\|\Delta_i^k(l)\|_p}.\]
	Putting all together we have
    \begin{eqnarray*}
    	\EE\Lambda^{k+1} &\le& \EE f(x^k) - f^* + c\frac{L\gamma^2}{2}\frac{1}{n}\sumin \EE\|h_i^{k}-h_i^*\|_2^2\\
    	&&\quad  - \gamma \left( 1 - \frac{L\gamma}{2} - c\alpha L\gamma\right) \EE\|\nabla f(x^k)\|_2^2\\
    	&&\quad + (1+2cn\alpha)\frac{L\gamma^2}{2}\frac{\sigma^2}{n} + 2c\alpha L\gamma^2\zeta^2.
    \end{eqnarray*}
    Due to $\gamma \le \frac{2}{L(1+2c\alpha)}$ the coefficient before $\|\nabla f(x^k)\|_2^2$ is positive. Therefore, we can rearrange the terms and rewrite the last bound as
    \begin{align*}
        \EE[\|\nabla f(x^k)\|_2^2] \le 2\frac{\EE\Lambda^{k} - \EE\Lambda^{k+1}}{\gamma(2 - L\gamma - 2c\alpha L \gamma)} + \frac{(1+2cn\alpha)L\gamma}{2 - L\gamma - 2c\alpha L \gamma}\frac{\sigma^2}{n} + \frac{4c\alpha L\gamma \zeta^2}{2-L\gamma -2c\alpha L\gamma}.
    \end{align*}
    Summing from $0$ to $k-1$ results in telescoping of the right-hand side, giving
    \begin{align*}
        \sum_{l=0}^{k-1}\EE[\|\nabla f(x^l)\|_2^2] &\le 2\frac{\Lambda^{0} - \EE\Lambda^{k}}{\gamma(2 - L\gamma - 2c\alpha L \gamma)} + k\frac{(1+2cn\alpha)L\gamma}{2 - L\gamma - 2c\alpha L \gamma}\frac{\sigma^2}{n} + k\frac{4c\alpha L\gamma \zeta^2}{2-L\gamma -2c\alpha L\gamma}.
    \end{align*}
    Note that $\EE\Lambda^k$ is non-negative and, thus, can be dropped. After that, it suffices to divide both sides by $k$ and rewrite the left-hand side as $\EE\|\nabla f(\overline x^k)\|_2^2$ where expectation is taken w.r.t.\ all randomness.
\end{proof}

\begin{corollary}\label{cor:ncvx}
	Set $\alpha = \frac{\alpha_p}{2}$, $c = \frac{4(1-\alpha_p)}{n\alpha_p^2}$, $\gamma = \frac{n\alpha_p}{L(4 + (n-4)\alpha_p)\sqrt{K}}$, $h^0 = 0$ and run the algorithm for $K$ iterations. Then, the final accuracy is at most $\frac{2}{\sqrt{K}} \frac{L(4+(n-4)\alpha_p)}{n\alpha_p} \Lambda^0 + \frac{1}{\sqrt{K}}\frac{(4-3\alpha_p)\sigma^2}{4+(n-4)\alpha_p} + \frac{8(1-\alpha_p)\zeta^2}{(4+(n-4)\alpha_p)\sqrt{K}}$.
\end{corollary}
\begin{proof}
	Our choice of $\alpha$ and $c$ implies
	\[
		c\alpha = \frac{2 ( 1 - \alpha_p )}{ n \alpha_p},\quad 1 + 2c\alpha = \frac{4+(n-4)\alpha_p}{n\alpha_p},\quad  1+2cn\alpha = \frac{4-3\alpha_p}{\alpha_p}.
	\]
	Using this and the inequality $\gamma = \frac{n\alpha_p}{L(4 + (n-4)\alpha_p)\sqrt{K}} \le \frac{n\alpha_p}{L(4 + (n-4)\alpha_p)}$ we get $2 - L\gamma - 2c\alpha L \gamma = 2 - (1+2c\alpha)L\gamma \ge 1$. Putting all together we obtain $\frac{2}{K}\frac{\Lambda^0}{\gamma\left(2 - L\gamma - 2c\alpha L \gamma\right)} + (1+2cn\alpha)\frac{L\gamma}{2 - L\gamma - 2c\alpha L \gamma}\frac{\sigma^2}{n} +\frac{4c\alpha L\gamma \zeta^2}{2-L\gamma -2c\alpha L\gamma} \le \frac{2}{\sqrt{K}} \frac{L(4+(n-4)\alpha_p)}{n\alpha_p} \Lambda^0 + \frac{1}{\sqrt{K}}\frac{(4-3\alpha_p)\sigma^2}{4+(n-4)\alpha_p} + \frac{8(1-\alpha_p)\zeta^2}{(4+(n-4)\alpha_p)\sqrt{K}}$.
\end{proof}

%\begin{corollary}\label{cor:ncvx}
%	Set $\alpha = \frac{1}{2\cnorm}$, $\gamma = \frac{1}{2L(1 + (\cnorm - 1)/n)\sqrt{K}}$, $h^0 = 0$ and run the algorithm for $K$ iterations. Then, the final accuracy is at most $\frac{4}{\sqrt{K}} L( 1 + \frac{\cnorm - 1}{n}) (f(x^0) - f^*) + \frac{(2\cnorm - 1)\sigma^2}{2\sqrt{K}(n + \cnorm - 1)}$.
%\end{corollary}
%Moreover, if the first term in Corollary~\ref{cor:ncvx} is leading and $\frac{1}{n} = \Omega(\alpha_p)$, the resulting complexity is $O(\frac{1}{\sqrt{K}})$, i.e.\ the same as that of {\tt SGD}. For instance, if sufficiently large mini-batches are used, the former condition holds, while for the latter it is enough to quantize vectors in blocks of size $O(n^2)$.

\section{Momentum version of {\tt DIANA}}\label{sec:DIANA-momentum}
%The update rules I use are
%\begin{align}
%	v^k &= \beta v^{k-1} + \hat g^k, \label{eq:vk_recursion}\\
%	x^{k+1} &= x^k - \gamma v^k, \label{eq:xk_recursion}
%\end{align}
%where $\hat g^k$ is the average of quantized gradients $g_i^k$.
\begin{theorem}\label{thm:DIANA-momentum}
	Assume that $f$ is $L$-smooth, $R\equiv \text{const}$, $h_i^0 = 0$ and Assumption~\ref{as:almost_identical data} holds. Choose $0\le \alpha < \alpha_p$, $\beta < 1-\alpha$ and $\gamma < \frac{1-\beta^2}{2L\left(2\omega - 1\right)}$, such that $\frac{\beta^2}{(1-\beta)^2\alpha} \le \frac{1-\beta^2-2L\gamma\left(2\omega -1\right)}{\gamma^2L^2\delta}$, where $\delta \eqdef 1 + \frac{2}{n}\left(\frac{1}{\alpha_p}-1\right)\left(1+\frac{\alpha}{1-\alpha-\beta}\right)$ and $\omega \eqdef \frac{n-1}{n} + \frac{1}{n\alpha_p}$, and sample $\overline x^k$ uniformly from $\{x^0, \dotsc, x^{k-1}\}$. Then
	\begin{align*}
		\EE \|\nabla f(\overline x^k)\|_2^2 &\le \frac{4(f(z^0) - f^*)}{\gamma k} + 2\gamma\frac{L \sigma^2}{(1-\beta)^2  n}\left(\frac{3}{\alpha_p}-2\right) + 2\gamma^2\frac{ L^2\beta^2\sigma^2}{(1 - \beta)^5n}\left(\frac{3}{\alpha_p}-2\right)\\
		&\quad + 3\gamma^2\frac{L^2\beta^2\zeta^2}{(1-\beta)^5n}\left(\frac{1}{\alpha_p}-1\right).
	\end{align*}
\end{theorem}
\begin{proof}
	The main idea of the proof is to find virtual iterates $z^k$ whose recursion would satisfy $z^{k+1} = z^k - \frac{\gamma}{1-\beta} \hat g^k$. Having found it, we can prove convergence by writing a recursion on $f(z^k)$. One possible choice is defined below:
	\begin{align}
		z^k \eqdef x^k - \frac{\gamma \beta}{1 - \beta} v^{k-1}, \label{eq:def_zk}
	\end{align}
	where for the edge case $k=0$ we simply set $v^{-1}=0$ and $z^0=x^0$.
	Although $z^k$ is just a slight perturbation of $x^k$, applying smoothness inequality~\eqref{eq:smoothness_functional} to it produces a more convenient bound than the one we would have if used $x^k$. But first of all, let us check that we have the desired recursion for $z^{k+1}$:
	\begin{eqnarray*}
		z^{k+1} 
		&\overset{\eqref{eq:def_zk}}{=}& x^{k+1} -  \frac{\gamma \beta}{1 - \beta} v^{k}  \\
		&{=}& x^k -  \frac{\gamma}{1 - \beta} v^{k} \\
		&{=}& x^k -  \frac{\gamma \beta}{1 - \beta} v^{k-1} -  \frac{\gamma}{1 - \beta} \hat g^k \\
		&\overset{\eqref{eq:def_zk}}{=}& z^k - \frac{\gamma}{1 - \beta} \hat g^k.
	\end{eqnarray*}
	Now, it is time to apply smoothness of $f$:
	\begin{eqnarray}
		\EE f(z^{k+1}) 
		&\le& \EE \left[f(z^k) + \< \nabla f(z^k), z^{k+1} - z^k> + \frac{L}{2}\|z^{k+1} - z^k\|_2^2 \right] \nonumber\\
		&\overset{\eqref{eq:def_zk}}{=}& \EE \left[f(z^k) - \frac{\gamma}{1 - \beta} \< \nabla f(z^k), \hat g^k> + \frac{L\gamma^2}{2(1-\beta)^2}\|\hat g^k\|_2^2 \right] . \label{eq:technical2}
%		&\overset{\eqref{eq:full_variance_of_mean_g1}}{\le}& \EE \left[f(z^k) - \gamma \< \nabla f(z^k), \nabla f(x^k)> + \frac{L\gamma^2\sigma^2}{2n} + \frac{L\gamma^2}{2}\|\nabla f(x^k)\|_2^2 + \frac{L\gamma^2}{2}\|\hat g^k - \nabla f(x^k)\|_2^2 \right]. \label{eq:technical2}
	\end{eqnarray}
	The scalar product in~\eqref{eq:technical2} can be bounded using the fact that for any vectors $a$ and $b$ one has $-\< a, b> = \frac{1}{2}(\|a - b\|_2^2 - \|a\|_2^2 - \|b\|_2^2)$. In particular,
	\begin{align*}
		 - \frac{\gamma}{1 - \beta} \< \nabla f(z^k), \nabla f(x^k)> 
		 &= \frac{\gamma}{2(1-\beta)}\left(\|\nabla f(x^k) - \nabla f(z^k)\|_2^2 - \|\nabla f(x^k)\|_2^2 - \|\nabla f(z^k)\|_2^2 \right) \\
		 &\le  \frac{\gamma}{2(1-\beta)}\left(L^2\|x^k - z^k\|_2^2 - \|\nabla f(x^k)\|_2^2\right) \\
		 &= \frac{\gamma^3L^2\beta^2}{2(1 - \beta)^3}\|v^{k-1}\|_2^2 - \frac{\gamma}{2(1-\beta)}\|\nabla f(x^k)\|_2^2.
	\end{align*}
	The next step is to come up with an inequality for $\EE\|v^k\|_2^2$. Since we initialize $v^{-1}=0$, one can show by induction that 
	\begin{equation*}
		v^k = \sum_{l=0}^{k}\beta^{l} \hat g^{k - l}.
	\end{equation*}
	Define $B \eqdef \sum_{l=0}^k \beta^l = \frac{1 - \beta^{k+1}}{1 - \beta}$. Then, by Jensen's inequality
	\begin{align*}
		\EE\|v^k\|_2^2 
		&= B^2\EE\left\|\sum_{l=0}^{k}\frac{\beta^{l}}{B} \hat g^{k - l} \right\|_2^2 
		\le B^2 \sum_{l=0}^{k}\frac{\beta^{l}}{B} \EE\|\hat g^{k - l}\|_2^2.
	\end{align*}
	Since $\alpha < \alpha_p \le \alpha_p(d_l) \le \frac{\|\Delta_i^k(l)\|_2^2}{\|\Delta_i^k(l)\|_1\|\Delta_i^k(l)\|_p}$ for all $i,k$ and $l$, we have
	\[
		\|\Delta_i^k(l)\|_2^2 - \alpha\|\Delta_i^k(l)\|_1\|\Delta_i^k(l)\|_p \ge (\alpha_p - \alpha)\|\Delta_i^k(l)\|_1\|\Delta_i^k(l)\|_p \ge 0
	\]
	for the case when $\Delta_i^k(l)\neq 0$. When $\Delta_i^k(l) = 0$ we simply have $\|\Delta_i^k(l)\|_2^2 - \alpha\|\Delta_i^k(l)\|_1\|\Delta_i^k(l)\|_p = 0$. Taking into account this and the following equality
	\[
		\|\Delta_i^k\|_2^2 - \alpha\sum\limits_{l=1}^m\|\Delta_i^k(l)\|_1\|\Delta_i^k(l)\|_p = \sum\limits_{l=1}^m\left(\|\Delta_i^k(l)\|_2^2 - \alpha\|\Delta_i^k(l)\|_1\|\Delta_i^k(l)\|_p\right),	
	\]
	we get from \eqref{eq:distr_h^(k+1)le}
	\begin{eqnarray*}
		\EE\|h_i^k\|_2^2 &\le& (1-\alpha)\EE\left[\|h_i^{k-1}\|_2^2\right] + \alpha\EE\left[\|g_i^{k-1}\|_2^2\right]\\
		&\le& (1-\alpha)^2\EE\left[\|h_i^{k-2}\|_2^2\right] + \alpha(1-\alpha)\EE\left[\|g_i^{k-2}\|_2^2\right] + \alpha\EE\left[\|g_i^{k-1}\|_2^2\right]\\
		&\le& \ldots \le (1-\alpha)^k\underbrace{\|h_i^0\|_2^2}_0 + \alpha\sum\limits_{j=0}^{k-1}(1-\alpha)^{j}\EE\left[\|g_i^{k-1-j}\|_2^2\right]\\
		&=& \alpha\sum\limits_{j=0}^{k-1}(1-\alpha)^j\EE\|\nabla f_i(x^{k-1-j})\|_2^2 + \alpha\sum\limits_{j=0}^{k-1}(1-\alpha)^j\sigma_i^2\\
		&\le& \alpha\sum\limits_{j=0}^{k-1}(1-\alpha)^j\EE\|\nabla f_i(x^{k-1-j})\|_2^2 + \alpha\cdot\frac{\sigma_i^2}{1-(1-\alpha)}\\
		&=& \alpha\sum\limits_{j=0}^{k-1}(1-\alpha)^j\EE\|\nabla f_i(x^{k-1-j})\|_2^2 + \sigma_i^2.
	\end{eqnarray*}
	Under our special assumption inequality~\eqref{eq:full_second_moment_of_hat_g} gives us
	\begin{align*}
		\EE\left[\|\hat g^k\|_2^2\right] &\le \EE\|\nabla f(x^k)\|_2^2 + \left(\frac{1}{\alpha_p} - 1\right)\frac{1}{n^2}\sumin\EE\underbrace{\left[\|\nabla f_i(x^k) - h_i^k\|_2^2\right]}_{\le 2\|\nabla f_i(x^k)\|_2^2 + 2\|h_i^k\|_2^2} + \frac{\sigma^2}{\alpha_p n}\\
		&\le \EE\|\nabla f(x^k)\|_2^2 + \frac{2}{n^2}\left(\frac{1}{\alpha_p}-1\right)\sum\limits_{i=1}^n\|\nabla f_i(x^k)\|_2^2\\
		&\quad + \frac{2}{n^2}\left(\frac{1}{\alpha_p}-1\right)\sumin \EE\|h_i^k\|_2^2 + \frac{\sigma^2}{\alpha_pn}\\		
		&\overset{\eqref{eq:almost_identical_data}}{\le} \left(1+\frac{2}{n}\left(\frac{1}{\alpha_p}-1\right)\right)\EE\|\nabla f(x^k)\|_2^2 + \frac{2}{n}\left(\frac{1}{\alpha_p}-1\right)\zeta^2\\
		&\quad  + \frac{2}{n^2}\left(\frac{1}{\alpha_p}-1\right)\sumin \EE\|h_i^k\|_2^2 + \frac{\sigma^2}{\alpha_pn}\\
		&\le \left(1+\frac{2}{n}\left(\frac{1}{\alpha_p}-1\right)\right)\EE\|\nabla f(x^k)\|_2^2 \\
		&\quad + \frac{2\alpha}{n^2}\left(\frac{1}{\alpha_p}-1\right)\sum\limits_{i=1}^n\sum\limits_{j=0}^{k-1}(1-\alpha)^j\EE\|\nabla f_i(x^{k-1-j})\|_2^2\\
		&\quad + \left(\frac{1}{\alpha_p}-1\right)\frac{2\sigma^2+2\zeta^2}{n} + \frac{\sigma^2}{\alpha_pn}\\
		&\overset{\eqref{eq:almost_identical_data}}{\le} \left(1+\frac{2}{n}\left(\frac{1}{\alpha_p}-1\right)\right)\EE\|\nabla f(x^k)\|_2^2\\
		&\quad  + \frac{2\alpha}{n}\left(\frac{1}{\alpha_p}-1\right)\sum\limits_{j=0}^{k-1}(1-\alpha)^j\EE\|\nabla f(x^{k-1-j})\|_2^2\\
		&\quad + \left(\frac{1}{\alpha_p}-1\right)\frac{2\sigma^2+2\zeta^2}{n} + \frac{\sigma^2}{\alpha_pn} + \frac{2\alpha}{n}\left(\frac{1}{\alpha_p}-1\right)\sum\limits_{j=0}^{k-1}(1-\alpha)^j\zeta^2\\
		&\le \left(1+\frac{2}{n}\left(\frac{1}{\alpha_p}-1\right)\right)\EE\|\nabla f(x^k)\|_2^2\\
		&\quad  + \frac{2\alpha}{n}\left(\frac{1}{\alpha_p}-1\right)\sum\limits_{j=0}^{k-1}(1-\alpha)^j\EE\|\nabla f(x^{k-1-j})\|_2^2\\
		&\quad + \left(\frac{1}{\alpha_p}-1\right)\frac{2\sigma^2+3\zeta^2}{n} + \frac{\sigma^2}{\alpha_pn}.
	\end{align*}
	Using this, we continue our evaluation of $\EE\|v^k\|_2^2$:
	\begin{eqnarray*}
		\EE\|v^k\|_2^2 &\le& B\sum\limits_{l=0}^k\beta^l\left(1+\frac{2}{n}\left(\frac{1}{\alpha_p}-1\right)\right)\EE\|\nabla f(x^{k-l})\|_2^2\\
		&&\quad + B\left(\frac{1}{\alpha_p}-1\right)\frac{2\alpha}{n}\sum\limits_{l=0}^k\sum\limits_{j=0}^{k-l-1}\beta^l(1-\alpha)^j\EE\|\nabla f(x^{k-l-1-j})\|_2^2\\
		&&\quad + B\sum\limits_{l=0}^k\beta^l\left(\left(\frac{1}{\alpha_p}-1\right)\frac{2\sigma^2+3\zeta^2}{n} + \frac{\sigma^2}{\alpha_pn}\right).
	\end{eqnarray*}
	Now we are going to simplify the double summation:
	\begin{eqnarray*}
		\sum\limits_{l=0}^k\sum\limits_{j=0}^{k-l-1}\beta^l(1-\alpha)^j\EE\|\nabla f(x^{k-l-1-j})\|_2^2 &=&\sum\limits_{l=0}^k\sum\limits_{j=0}^{k-l-1}\beta^l(1-\alpha)^{k-l-1-j}\EE\|\nabla f(x^{j})\|_2^2\\
		 &=& \sum\limits_{j=0}^{k-1}\EE\|\nabla f(x^{j})\|_2^2\sum\limits_{l=0}^{k-j-1}\beta^l(1-\alpha)^{k-l-1-j}\\
		 &=& \sum\limits_{j=0}^{k-1}\EE\|\nabla f(x^{j})\|_2^2\cdot \frac{(1-\alpha)^{k-j} - \beta^{k-j}}{1-\alpha-\beta}\\
		 &\le& \sum\limits_{j=0}^{k}\EE\|\nabla f(x^{j})\|_2^2\cdot \frac{(1-\alpha)^{k-j}}{1-\alpha-\beta}\\
		 &=& \frac{1}{1-\alpha - \beta}\sum\limits_{j=0}^{k}(1-\alpha)^{j}\EE\|\nabla f(x^{k-j})\|_2^2.
	\end{eqnarray*}
	Note that $B \eqdef \sum\limits_{l=0}^k\beta^l \le \frac{1}{1-\beta}$. Putting all together we get
	\begin{eqnarray*}
		\EE\|v^k\|_2^2 &\le& \frac{\delta}{1-\beta}\sum\limits_{l=0}^k(1-\alpha)^l\EE\|\nabla f(x^{k-l})\|_2^2 + \frac{\sigma^2}{n(1-\beta)^2}\left(\frac{3}{\alpha_p}-2\right)\\
		&&\quad  + \frac{3\zeta^2}{n(1-\beta)^2}\left(\frac{1}{\alpha_p}-1\right),
	\end{eqnarray*}
	where $\delta \eqdef 1 + \frac{2}{n}\left(\frac{1}{\alpha_p}-1\right)\left(1+\frac{\alpha}{1-\alpha-\beta}\right)$, and as a result
	\begin{eqnarray*}
		\frac{\gamma^3L^2\beta^2}{2(1-\beta)^3}\EE\|v^{k-1}\|_2^2 &\le& \frac{\gamma^3L^2\beta^2\delta}{2(1-\beta)^4}\sum\limits_{l=0}^{k-1}(1-\alpha)^{k-1-l}\EE\|\nabla f(x^{l})\|_2^2\\
		&&\quad + \frac{\gamma^3L^2\beta^2\sigma^2}{2n(1-\beta)^5}\left(\frac{3}{\alpha_p}-2\right) + \frac{3\gamma^3L^2\beta^2\zeta^2}{2n(1-\beta)^5}\left(\frac{1}{\alpha_p}-1\right).
	\end{eqnarray*}
	To sum up, we have
	\begin{eqnarray*}
		\EE\left[f(z^{k+1})\right] &\le& \EE\left[f(z^k)\right] -\frac{\gamma}{2(1-\beta)}\left(1 - \frac{L\gamma\omega}{1-\beta}\right)\EE\|\nabla f(x^k)\|_2^2\\
		&&\quad + \left(\frac{L\gamma^2\alpha(\omega-1)}{2(1-\beta)^2}+\frac{\gamma^3L^2\beta^2\delta}{2(1-\beta)^4}\right)\sum\limits_{l=0}^{k-1}(1-\alpha)^{k-1-l}\EE\|\nabla f(x^l)\|_2^2\\
		&&\quad + \frac{\sigma^2}{n}\left(\frac{3}{\alpha_p}-2\right)\left(\frac{L\gamma^2}{2(1-\beta)^2}+\frac{\gamma^3L^2\beta^2}{2(1-\beta)^5}\right) + \frac{3\gamma^3L^2\beta^2\zeta^2}{2n(1-\beta)^5}\left(\frac{1}{\alpha_p}-1\right).
	\end{eqnarray*}
	Telescoping this inequality from 0 to $k-1$, we get
	\begin{eqnarray*}
		\EE f(z^k) - f(z^0)
		&\le& k\frac{\sigma^2}{n}\left(\frac{3}{\alpha_p}-2\right)\left(\frac{L\gamma^2}{2(1-\beta)^2}+\frac{\gamma^3L^2\beta^2}{2(1-\beta)^5}\right) + k\frac{3\gamma^3L^2\beta^2\zeta^2}{2n(1-\beta)^5}\left(\frac{1}{\alpha_p}-1\right)\\
		&&+ \frac{\gamma}{2}\sum\limits_{l=0}^{k-2}\left(\left(\frac{L\gamma\alpha(\omega-1)}{(1-\beta)^2}+\frac{\gamma^2L^2\beta^2\delta}{(1-\beta)^4}\right)\sum\limits_{k'=l+1}^{k-1}(1-\alpha)^{k'-1-l}\right)\EE\|\nabla f(x^l)\|_2^2\\
		&&+ \frac{\gamma}{2}\sum\limits_{l=0}^{k-2}\left(\frac{L\gamma\omega}{(1-\beta)^2} - \frac{1}{1-\beta}\right)\EE\|\nabla f(x^l)\|_2^2\\
		&&+ \frac{\gamma}{2}\left(\frac{L\gamma\omega}{(1-\beta)^2} - \frac{1}{1-\beta}\right)\EE\|\nabla f(x^{k-1})\|_2^2\\
		&\le& k\frac{\sigma^2}{n}\left(\frac{3}{\alpha_p}-2\right)\left(\frac{L\gamma^2}{2(1-\beta)^2}+\frac{\gamma^3L^2\beta^2}{2(1-\beta)^5}\right)+ k\frac{3\gamma^3L^2\beta^2\zeta^2}{2n(1-\beta)^5}\left(\frac{1}{\alpha_p}-1\right)\\
		&&+ \frac{\gamma}{2}\sum\limits_{l=0}^{k-1}\left(\frac{\gamma^2L^2\beta^2\delta}{(1-\beta)^4\alpha} + \frac{L\gamma}{(1-\beta)^2}\left(2\omega-1\right) - \frac{1}{1-\beta}\right)\EE\|\nabla f(x^l)\|_2^2
	\end{eqnarray*}
	It holds $f^*\le f(z^k)$ and our assumption on $\beta$ implies that $\frac{\gamma^2L^2\beta^2\delta}{(1-\beta)^4\alpha} + \frac{L\gamma}{(1-\beta)^2}\left(2\omega -1\right) - \frac{1}{1-\beta} \le -\frac{1}{2}$, so it all results in
	\begin{eqnarray*}
		\frac{1}{k}\sum_{l=0}^{k-1} \|\nabla f(x^{l})\|_2^2 &\le&  \frac{4(f(z^0) - f^*)}{\gamma k} + 2\gamma\frac{L \sigma^2}{(1-\beta)^2  n}\left(\frac{3}{\alpha_p}-2\right) + 2\gamma^2\frac{ L^2\beta^2\sigma^2}{(1 - \beta)^5n}\left(\frac{3}{\alpha_p}-2\right)\\
		&&\quad + \frac{3\gamma^2L^2\beta^2\zeta^2}{n(1-\beta)^5}\left(\frac{1}{\alpha_p}-1\right).
	\end{eqnarray*}
	Since $\overline x^k$ is sampled uniformly from $\{x^0, \dotsc, x^{k-1}\}$, the left-hand side is equal to $\EE \|\nabla f(\overline x^k)\|_2^2$. Also note that $z^0=x^0$.
\end{proof}
\begin{corollary}\label{cor:DIANA-momentum}
		If we set $\gamma=\frac{1-\beta^2}{2\sqrt{k}L\left(2\omega -1\right)}$ and $\beta$ such that $\frac{\beta^2}{(1 - \beta)^2\alpha}\le \frac{4k\left(2\omega -1\right)}{\delta}$ with $k>1$, then the accuracy after $k$ iterations is at most 
		\begin{eqnarray*}
			\frac{1}{\sqrt{k}}\left(\frac{8L(2\omega -1)(f(x^0)-f^*)}{1-\beta^2} + \frac{(1+\beta)\sigma^2}{(2\omega -1)\alpha_p n(1-\beta)}\left(\frac{3}{\alpha_p}-2\right)\right)\\ + \frac{1}{k}\frac{(1+\beta)^4\beta^2\sigma^2}{2(1 - \beta)(2\omega -1)\alpha_pn}\left(\frac{3}{\alpha_p}-2\right)
			+\frac{1}{k}\frac{3(1+\beta)^4\beta^2\zeta^2}{4(1 - \beta)(2\omega -1)\alpha_pn}\left(\frac{1}{\alpha_p}-1\right).
		\end{eqnarray*}
\end{corollary}
\begin{proof}
	Our choice of $\gamma = \frac{1-\beta^2}{2\sqrt{k}L\left(2\omega -1\right)}$ implies that
	\[
		\frac{\beta^2}{(1 - \beta)^2\alpha}\le \frac{4k\left(2\omega -1\right)}{\delta} \Longleftrightarrow \frac{\beta^2}{(1-\beta)^2\alpha} \le \frac{1-\beta^2-2L\gamma\left(2\omega -1\right)}{\gamma^2L^2\delta}.
	\]
	After that it remains to put $\gamma = \frac{1-\beta^2}{2\sqrt{k}L\left(2\omega -1\right)}$ in $\frac{4(f(z^0) - f^*)}{\gamma k} + 2\gamma\frac{L \sigma^2}{(1-\beta)^2  n}\left(\frac{3}{\alpha_p}-2\right) + 2\gamma^2\frac{ L^2\beta^2\sigma^2}{(1 - \beta)^5n}\left(\frac{3}{\alpha_p}-2\right) + \frac{3\gamma^2L^2\beta^2\zeta^2}{n(1-\beta)^5}\left(\frac{1}{\alpha_p}-1\right)$ to get the desired result. 
\end{proof}

\clearpage

\section{Analysis of {\tt DIANA} with $\alpha = 0$ and $h_i^0 = 0$}
\label{sec:Terngrad}

\subsection{Convergence Rate of {\tt TernGrad}}

Here we give the convergence guarantees for {\tt TernGrad} and provide upper bounds for this method. The method coincides with Algorithm~\ref{alg:terngrad} for the case when $p = \infty$. In the original paper \cite{wen2017terngrad} no convergence rate was given and {\em we close this gap.}

To maintain consistent notation we rewrite the {\tt TernGrad} in notation which is close to the notation we used for {\tt DIANA}. Using our notation it is easy to see that {\tt TernGrad} is {\tt DIANA} with $h_1^0 = h_2^0 = \ldots = h_n^0 = 0, \alpha = 0$ and $p=\infty$. Firstly, it means that $h_i^k = 0$ for all $i=1,2,\ldots,n$ and $k\ge 1$. What is more, this observation tells us that Lemma \ref{lem:3in1} holds for the iterates of {\tt TernGrad} too. What is more, in the original paper \cite{wen2017terngrad} the quantization parameter $p$ was chosen as $\infty$. We generalize the method and we don't restrict our analysis only on the case of $\ell_\infty$ sampling.

As it was in the analysis of {\tt DIANA} our proofs for {\tt TernGrad} work under Assumption~\ref{as:noise}.

\begin{algorithm}[!h]
   \caption{{\tt DIANA} with $\alpha=0$ \& $h_i^0 = 0$; {\tt QSGD} for $p=2$ (1-bit)/ {\tt TernGrad} for $p=\infty$ ({\tt SGD}), \cite{alistarh2017qsgd, wen2017terngrad}}
   \label{alg:terngrad}
\begin{algorithmic}[1]
   \INPUT learning rates $\{\gamma^k\}_{k\ge 0}$, initial vector $x^0$, quantization parameter $p \geq 1$, sizes of blocks $\{d_l\}_{l=1}^m$, momentum parameter $0\le \beta < 1$
   \STATE $v^0 = \nabla f(x^0)$
   \FOR{$k=1,2,\dotsc$}
	   \STATE Broadcast $x^{k}$ to all workers
        \FOR{$i=1,\dotsc,n$ do in parallel}
			\STATE Sample $g^{k}_i$ such that $\EE [g^k_i \;|\; x^k]  =\nabla f_i(x^k)$;\quad Sample $\hat g^k_i \sim {\rm Quant}_{p}(g^k_i,\{d_l\}_{l=1}^m)$
        \ENDFOR
        \STATE $\hat g^k = \frac{1}{n}\sum_{i=1}^n \hat g_i^k$; \quad  $v^k = \beta v^{k-1} + \hat g^k$; \quad        $x^{k+1} = \proxkR\left(x^k - \gamma^kv^k \right)$
   \ENDFOR
\end{algorithmic}
\end{algorithm}

\subsection{Technical lemmas}

First of all, we notice that since {\tt TernGrad} coincides with {\tt Diana}, having $h_i^k = 0, i,k \ge 1$, $\alpha = 0$ and $p = \infty$, all inequalities from Lemma~\ref{lem:3in1} holds for the iterates of {\tt TernGrad} as well because $\Delta_i^k = g_i^k$ and $\hat \Delta_i^k = \hat g_i^k$.

\begin{lemma}\label{lem:gamma_choice_terngrad_special}
	Assume $\gamma \le \frac{n\alpha_p}{L((n-1)\alpha_p+1)}$. Then
	\begin{align}
		2\gamma\mu\left(1 - \frac{\gamma L((n-1)\alpha_p+1)}{2n\alpha_p}\right) \ge \gamma\mu.\label{eq:conseq_gamma_choice_terngrad_special}
	\end{align}
\end{lemma}
\begin{proof}
	Since $\gamma \le \frac{n\alpha_p}{L((n-1)\alpha_p+1)}$ we have
	$$
		2\gamma\mu\left(1 - \frac{\gamma L((n-1)\alpha_p+1)}{2n\alpha_p}\right) \ge 2\gamma\mu \left(1-\frac{1}{2}\right) = \gamma\mu.
	$$
\end{proof}

\begin{lemma}\label{lem:gamma_choice_terngrad}
	Assume $\gamma \le \frac{1}{L(1+\kappa(1 - \alpha_p)/(n\alpha_p))}$, where $\kappa \eqdef \frac{L}{\mu}$ is the condition number of $f$. Then
	\begin{align}
		r \ge \gamma\mu\label{eq:conseq_gamma_choice_terngrad},
	\end{align}
	where $r = 2\mu\gamma - \gamma^2 \left(\mu L + \frac{L^2(1-\alpha_p)}{ n\alpha_p}\right)$.
\end{lemma}
\begin{proof}
	Since $\gamma \le \frac{1}{L(1+\kappa(1-\alpha_p)/(n\alpha_p))} = \frac{\mu n\alpha_p}{\mu n\alpha_p L + L^2(1-\alpha_p)}$ we have
	$$
		n\alpha_p r = \gamma\left(2\mu n\alpha_p - \gamma\left(\mu n\alpha_p L + L^2(1-\alpha_p)\right)\right) \ge \gamma\mu n\alpha_p,
	$$
	whence $r \ge \gamma\mu$.
\end{proof}

\begin{lemma}\label{lem:gamma_choice_terngrad_prox}
	Assume $\gamma \le \frac{2n\alpha_p}{(\mu + L)(2+(n-2)\alpha_p)}$. Then
	\begin{eqnarray}\label{eq:conseq_gamma_choice_terngrad_prox}
		2\gamma\mu - \gamma^2\mu^2\left(1 + \frac{2(1-\alpha_p)}{n\alpha_p}\right) &\ge & \gamma\mu.
	\end{eqnarray}
\end{lemma}
\begin{proof}
	Since $\gamma \le \frac{2n\alpha_p}{(\mu + L)(2+(n-2)\alpha_p)}$ we have
	\[
		\gamma\mu \le  \frac{2\mu n\alpha_p}{(\mu + L)(2+(n-2)\alpha_p)} \le  \frac{(\mu + L)n\alpha_p}{(\mu + L)(2+(n-2)\alpha_p)} =  \frac{n\alpha_p}{2+(n-2)\alpha_p},
	\]
	whence
	\[
		2\gamma\mu - \gamma^2\mu^2\left(1 + \frac{2(1-\alpha_p)}{n\alpha_p}\right) \ge 2\gamma\mu - \gamma\mu \frac{n\alpha_p}{2+(n-2)\alpha_p} \left(1 + \frac{2(1-\alpha_p)}{n\alpha_p}\right) = 2\gamma\mu - \gamma\mu = \gamma\mu.
	\]
\end{proof}

\begin{lemma}\label{lem:L_smoothness_consequense}
	Assume that each function $f_i$ is $L$-smooth and $R$ is a constant function. Then for the iterates of Algorithm $\ref{alg:terngrad}$ with $\gamma^k = \gamma$ we have
	\begin{eqnarray}
    	\EE\Theta^{k+1} &\le &  \EE\Theta^k + \left(\frac{\gamma^2 L}{2} - \gamma\right)\EE\left[\|\nabla f(x^k)\|_2^2\right] + \frac{\gamma^2 L}{2n^2}\left(\frac{1}{\alpha_p}-1\right)\sumin\EE\left[\|g_i^k\|_2^2\right]\notag\\
    	&&\quad + \frac{\gamma^2 L\sigma^2}{2n},\label{eq:L_smoothness_consequense}
    \end{eqnarray}
    where $\Theta^k = f(x^k) - f(x^*)$ and $\sigma^2\eqdef \frac{1}{n}\sumin\sigma_i^2$.
\end{lemma}
\begin{proof}
	Since $R$ is a constant function we have $x^{k+1} = x^k - \gamma \hat g^k$. Moreover, from the $L$-smoothness of $f$ we have
	\begin{eqnarray*}
		\EE\Theta^{k+1} &\le & \EE\Theta^k + \EE\left[\langle \nabla f(x^k), x^{k+1} - x^k \rangle\right] + \frac{L}{2}\|x^{k+1}-x^k\|_2^2\\
		&= & \EE\Theta^k - \gamma\EE\left[\|\nabla f(x^k)\|_2^2\right] + \frac{\gamma^2 L}{2}\EE\left[\left\|\hat g^k\right\|_2^2\right]\\
		&\stackrel{\eqref{eq:full_variance_of_mean_g1}}{\le}& \EE\Theta^k + \left(\frac{\gamma^2 L}{2} - \gamma\right)\EE\left[\|\nabla f(x^k)\|_2^2\right]\\
		&&\quad + \frac{\gamma^2 L}{2n^2}\sumin\sum\limits_{l=1}^m \EE\left[\|g_i^k(l)\|_1\|g_i^k(l)\|_p - \|g_i^k(l)\|_2^2 \right] + \frac{\gamma^2 L}{2n^2}\sumin\sigma_i^2,
	\end{eqnarray*}
	where the first equality follows from $x^{k+1} - x^k = \hat g^k$, $\EE\left[\hat g^k\mid x^k\right] = \nabla f(x^k)$ and the tower property of mathematical expectation. 
	By definition $\alpha_p(d_l) = \inf\limits_{x\neq 0,x\in\R^{d_l}}\frac{\|x\|_2^2}{\|x\|_1\|x\|_p} = \left(\sup\limits_{x\neq 0,x\in\R^{d_l}}\frac{\|x\|_1\|x\|_p}{\|x\|_2^2}\right)^{-1}$ and $\alpha_p = \alpha_p(\max\limits_{l=1,\ldots,m}d_l)$ which implies
	\begin{eqnarray*}
		\EE\left[\|g_i^k(l)\|_1\|g_i^k(l)\|_p - \|g_i^k(l)\|_2^2 \right] &=& \EE\left[\|g_i^k(l)\|_2^2\left(\frac{\|g_i^k(l)\|_1\|g_i^k(l)\|_p}{\|g_i^k(l)\|_2^2} - 1\right) \right]\\
		&\le& \left(\frac{1}{\alpha_p(d_l)}-1\right)\EE\|g_i^k(l)\|_2^2\\
		&\le& \left(\frac{1}{\alpha_p}-1\right)\EE\|g_i^k(l)\|_2^2.
	\end{eqnarray*}
 	Since $\|g_i^k\|_2^2 = \sum\limits_{l=1}^m\|g_i^k(l)\|_2^2$ we have
	\begin{eqnarray*}
		\EE\Theta^{k+1} &\le &  \EE\Theta^k + \left(\frac{\gamma^2 L}{2} - \gamma\right)\EE\left[\|\nabla f(x^k)\|_2^2\right] + \frac{\gamma^2 L}{2n^2}\left(\frac{1}{\alpha_p}-1\right)\sumin\EE\left[\|g_i^k\|_2^2\right] + \frac{\gamma^2 L\sigma^2}{2n},
	\end{eqnarray*}
	where $\sigma^2 = \frac{1}{n}\sumin\sigma_i^2$.
\end{proof}

\subsection{Nonconvex analysis}\label{sec:TernGrad-nonconvex}
\begin{theorem}\label{thm:TernGrad-nonconvex}
    Assume that $R$ is constant  and Assumption~\ref{as:almost_identical data} holds.    
    Also assume that $f$ is $L$-smooth, $\gamma \le \frac{n\alpha_p}{L((n-1)\alpha_p+1)}$ and $\overline x^k$ is chosen randomly from $\{x^0,\dotsc, x^{k-1} \}$. Then
    \begin{align*}
        \EE \|\nabla f(\overline x^k)\|_2^2 \le \frac{2}{k} \frac{f(x^0) - f(x^*)}{\gamma\left(2 - \gamma\frac{L((n-1)\alpha_p+1)}{n\alpha_p}\right)} + \frac{\gamma L\left(\sigma^2+(1-\alpha_p)\zeta^2\right)}{n\alpha_p}.
    \end{align*}
\end{theorem}
\begin{proof}
	Recall that we defined $\Theta^k$ as $f(x^k) - f(x^*)$ in Lemma~\ref{lem:L_smoothness_consequense}. From \eqref{eq:L_smoothness_consequense} we have
	\begin{eqnarray*}
		\EE\Theta^{k+1} &\le &  \EE\Theta^k + \left(\frac{\gamma^2 L}{2} - \gamma\right)\EE\left[\|\nabla f(x^k)\|_2^2\right] + \frac{\gamma^2 L}{2n^2}\left(\frac{1}{\alpha_p}-1\right)\sumin\EE\left[\|g_i^k\|_2^2\right] + \frac{\gamma^2 L\sigma^2}{2n}.
	\end{eqnarray*}
	Using variance decomposition
	\begin{eqnarray*}
		\EE\left[\|g_i^k\|_2^2\right] &=& \EE\left[\|\nabla f_i(x^k)\|_2^2\right] + \EE\left[\|g_i^k - \nabla f(x^k)\|_2^2\right] \le \EE\left[\|\nabla f_i(x^k)\|_2^2\right] + \sigma_i^2
	\end{eqnarray*}
	we get
	\begin{eqnarray*}
		\frac{1}{n}\sumin \EE\left[\|g_i^k\|_2^2\right] &\le & \frac{1}{n}\sumin\EE\left[\|\nabla f_i(x^k)\|_2^2\right] + \sigma^2\\
		&\overset{\eqref{eq:second_moment_decomposition}}{=}& \EE\|\nabla f(x^k)\|_2^2 + \frac{1}{n}\sumin\EE\left[\|\nabla f_i(x^k)-\nabla f(x^k)\|_2^2\right] + \sigma^2\\
		&\overset{\eqref{eq:almost_identical_data}}{\le}& \EE\|\nabla f(x^k)\|_2^2 + \zeta^2 + \sigma^2.
	\end{eqnarray*}
	Putting all together we obtain
	\begin{eqnarray}
		\EE\Theta^{k+1} &\le &  \EE\Theta^k + \left(\frac{\gamma^2L}{2}\cdot \frac{(n-1)\alpha_p+1}{n\alpha_p} - \gamma\right)\EE\left[\|\nabla f(x^k)\|_2^2\right] +  \frac{\gamma^2 L\sigma^2}{2n\alpha_p}\notag\\
		&&\quad + \frac{\gamma^2L\zeta^2(1-\alpha_p)}{2n\alpha_p} \label{eq:terngrad_key_estimation_special}.
	\end{eqnarray}
	Since $\gamma \le \frac{n\alpha_p}{L((n-1)\alpha_p+1)}$ the factor $\left(\frac{\gamma^2L}{2}\cdot \frac{(n-1)\alpha_p+1}{n\alpha_p} - \gamma\right)$ is negative and therefore
	\begin{eqnarray*}
		\EE\left[\|\nabla f(x^k)\|_2^2\right] &\le & \frac{\EE\Theta^k - \EE\Theta^{k+1}}{\gamma\left(1 - \gamma\frac{L((n-1)\alpha_p+1)}{2n\alpha_p}\right)} + \frac{\gamma L\left(\sigma^2+(1-\alpha_p)\zeta^2\right)}{2n\alpha_p - \gamma L((n-1)\alpha_p+1)}.
	\end{eqnarray*}
	Telescoping the previous inequality from $0$ to $k-1$ and using $\gamma\le \frac{n\alpha_p}{L((n-1)\alpha_p+1)}$ we obtain
	\begin{eqnarray*}
		\frac{1}{k}\sum\limits_{l=0}^{k-1}\EE\left[\|\nabla f(x^l)\|_2^2\right] &\le &\frac{2}{k} \frac{\EE\Theta^0 - \EE\Theta^{k}}{\gamma\left(2 - \gamma\frac{L((n-1)\alpha_p+1)}{n\alpha_p}\right)} + \frac{\gamma L\left(\sigma^2+(1-\alpha_p)\zeta^2\right)}{n\alpha_p}.
	\end{eqnarray*}
	It remains to notice that left-hand side is just $\EE\left[\|\nabla f(\overline x^k)\|_2^2\right]$, $\Theta^k \ge 0$ and $\Theta^0 = f(x^0) - f(x^*)$.
\end{proof}

\begin{corollary}\label{cor:TernGrad-nonconvex}
	If we choose $\gamma = \frac{n\alpha_p}{L((n-1)\alpha_p+1)\sqrt{K}}$ then the rate we get is $\frac{2}{\sqrt{K}}L\frac{(n-1)\alpha_p+1}{n\alpha_p}\left(f(x^0) - f(x^*)\right) + \frac{1}{\sqrt{K}}\frac{\sigma^2+(1-\alpha_p)\zeta^2}{(1+(n-1)\alpha_p)}$.
\end{corollary}
\begin{proof}
	Our choice of $\gamma = \frac{n\alpha_p}{L((n-1)\alpha_p+1)\sqrt{K}} \le \frac{n\alpha_p}{L((n-1)\alpha_p+1)}$ implies that $2 - \gamma\frac{L((n-1)\alpha_p+1)}{n\alpha_p} \ge 1$. After that it remains to notice that for our choice of $\gamma = \frac{n\alpha_p}{L((n-1)\alpha_p+1)\sqrt{K}}$ we have $\frac{2}{K} \frac{f(x^0) - f(x^*)}{\gamma\left(2 - \gamma\frac{L((n-1)\alpha_p+1)}{n\alpha_p}\right)} + \frac{\gamma L\left(\sigma^2+(1-\alpha_p)\zeta^2\right)}{n\alpha_p} \le \frac{2}{\sqrt{K}}L\frac{(n-1)\alpha_p+1}{n\alpha_p}\left(f(x^0) - f(x^*)\right) + \frac{1}{\sqrt{K}}\frac{\sigma^2+(1-\alpha_p)\zeta^2}{(1+(n-1)\alpha_p)}$. 
\end{proof}

\subsection{Momentum version}\label{sec:TernGrad-momentum}
\begin{theorem}\label{thm:TernGrad-momentum}
	Assume that $f$ is $L$-smooth, $R\equiv \text{const}$, $\alpha=0$, $h_i=0$ and Assumption~\ref{as:almost_identical data} holds. Choose $\beta<1$ and $\gamma < \frac{1-\beta^2}{2L\omega}$ such that $\frac{\beta^2}{(1 - \beta)^3}\le \frac{1 - \beta^2 - 2L\gamma\omega}{\gamma^2 L^2\omega}$, where $\omega \eqdef \frac{n-1}{n} + \frac{1}{n\alpha_p}$ and sample $\overline x^k$ uniformly from $\{x^0, \dotsc, x^{k-1}\}$. Then
	\begin{align*}
		\EE \|\nabla f(\overline x^k)\|_2^2 \le \frac{4(f(z^0) - f^*)}{\gamma k} + 2\gamma\frac{L \sigma^2}{\alpha_p  n(1-\beta)^2} + 2\gamma^2\frac{ L^2\beta^2\sigma^2}{(1 - \beta)^5\alpha_p n} + 2\gamma^2\frac{L^2\beta^2(1-\alpha_p)\zeta^2}{2(1 - \beta)^5\alpha_p n}.
	\end{align*}
\end{theorem}
\begin{proof}
	The main idea of the proof is to find virtual iterates $z^k$ whose recursion would satisfy $z^{k+1} = z^k - \frac{\gamma}{1-\beta} \hat g^k$. Having found it, we can prove convergence by writing a recursion on $f(z^k)$. One possible choice is defined below:
	\begin{align}
		z^k \eqdef x^k - \frac{\gamma \beta}{1 - \beta} v^{k-1}, \label{eq:def_zk}
	\end{align}
	where for the edge case $k=0$ we simply set $v^{-1}=0$ and $z^0=x^0$.
	Although $z^k$ is just a slight perturbation of $x^k$, applying smoothness inequality~\eqref{eq:smoothness_functional} to it produces a more convenient bound than the one we would have if used $x^k$. But first of all, let us check that we have the desired recursion for $z^{k+1}$:
	\begin{eqnarray*}
		z^{k+1} 
		&\overset{\eqref{eq:def_zk}}{=}& x^{k+1} -  \frac{\gamma \beta}{1 - \beta} v^{k}  \\
		&{=}& x^k -  \frac{\gamma}{1 - \beta} v^{k} \\
		&{=}& x^k -  \frac{\gamma \beta}{1 - \beta} v^{k-1} -  \frac{\gamma}{1 - \beta} \hat g^k \\
		&\overset{\eqref{eq:def_zk}}{=}& z^k - \frac{\gamma}{1 - \beta} \hat g^k.
	\end{eqnarray*}
	Now, it is time to apply smoothness of $f$:
	\begin{eqnarray}
		\EE f(z^{k+1}) 
		&\le& \EE \left[f(z^k) + \< \nabla f(z^k), z^{k+1} - z^k> + \frac{L}{2}\|z^{k+1} - z^k\|_2^2 \right] \nonumber\\
		&\overset{\eqref{eq:def_zk}}{=}& \EE \left[f(z^k) - \frac{\gamma}{1 - \beta} \< \nabla f(z^k), \hat g^k> + \frac{L\gamma^2}{2(1-\beta)^2}\|\hat g^k\|_2^2 \right] . \label{eq:technical2}
%		&\overset{\eqref{eq:full_variance_of_mean_g1}}{\le}& \EE \left[f(z^k) - \gamma \< \nabla f(z^k), \nabla f(x^k)> + \frac{L\gamma^2\sigma^2}{2n} + \frac{L\gamma^2}{2}\|\nabla f(x^k)\|_2^2 + \frac{L\gamma^2}{2}\|\hat g^k - \nabla f(x^k)\|_2^2 \right]. \label{eq:technical2}
	\end{eqnarray}
	Under our special assumption inequality~\eqref{eq:full_second_moment_of_hat_g} simplifies to
	\begin{align*}
		\EE\left[\|\hat g^k\|_2^2 \mid x^k\right] &\le \|\nabla f(x^k)\|_2^2 + \left(\frac{1}{\alpha_p} - 1\right)\frac{1}{n^2}\sum\limits_{i=1}^n \|\nabla f_i(x^k)\|_2^2 + \frac{\sigma^2}{\alpha_p n}\\
		&\overset{\eqref{eq:almost_identical_data}}{\le} \|\nabla f(x^k)\|_2^2 + \left(\frac{1}{\alpha_p} - 1\right)\frac{1}{n}\|\nabla f(x^k)\|_2^2 + \left(\frac{1}{\alpha_p} - 1\right)\frac{\zeta^2}{n} + \frac{\sigma^2}{\alpha_p n}.
	\end{align*}
	The scalar product in~\eqref{eq:technical2} can be bounded using the fact that for any vectors $a$ and $b$ one has $-\< a, b> = \frac{1}{2}(\|a - b\|_2^2 - \|a\|_2^2 - \|b\|_2^2)$. In particular,
	\begin{align*}
		 - \frac{\gamma}{1 - \beta} \< \nabla f(z^k), \nabla f(x^k)> 
		 &= \frac{\gamma}{2(1-\beta)}\left(\|\nabla f(x^k) - \nabla f(z^k)\|_2^2 - \|\nabla f(x^k)\|_2^2 - \|\nabla f(z^k)\|_2^2 \right) \\
		 &\le  \frac{\gamma}{2(1-\beta)}\left(L^2\|x^k - z^k\|_2^2 - \|\nabla f(x^k)\|_2^2\right) \\
		 &= \frac{\gamma^3L^2\beta^2}{2(1 - \beta)^3}\|v^{k-1}\|_2^2 - \frac{\gamma}{2(1-\beta)}\|\nabla f(x^k)\|_2^2.
	\end{align*}
	The next step is to come up with an inequality for $\EE\|v^k\|_2^2$. Since we initialize $v^{-1}=0$, one can show by induction that 
	\begin{equation*}
		v^k = \sum_{l=0}^{k}\beta^{l} \hat g^{k - l}.
	\end{equation*}
	Define $B \eqdef \sum_{l=0}^k \beta^l = \frac{1 - \beta^{k+1}}{1 - \beta}$. Then, by Jensen's inequality
	\begin{align*}
		\EE\|v^k\|_2^2 
		&= B^2\EE\left\|\sum_{l=0}^{k}\frac{\beta^{l}}{B} \hat g^{k - l} \right\|_2^2 
		\le B^2 \sum_{l=0}^{k}\frac{\beta^{l}}{B} \EE\|\hat g^{k - l}\|_2^2\\
%		= B \sum_{l=0}^{k}\beta^{l}(\|\EE \hat g^{k-l}\|_2^2 + \EE\|\hat g^{k-l} - \EE \hat g^{k-l}\|_2^2) 
		&\le B \sum_{l=0}^{k}\beta^{l} \left(\left(\frac{n-1}{n} + \frac{1}{n\alpha_p}\right)\EE\|\nabla f(x^{k-l})\|_2^2 + \left(\frac{1}{\alpha_p} - 1\right)\frac{\zeta^2}{n} + \frac{\sigma^2}{\alpha_p n}\right).
	\end{align*}
	Note that $B\le \frac{1}{1-\beta}$, so
	\begin{eqnarray*}
		\frac{\gamma^3L^2\beta^2}{2(1 - \beta)^3}\EE\|v^{k-1}\|_2^2 &\le& \frac{\gamma^3 L^2\beta^2}{2(1 - \beta)^5} \frac{\sigma^2}{\alpha_p n} + \frac{\gamma^3 L^2\beta^2}{2(1 - \beta)^5} \frac{(1-\alpha_p)\zeta^2}{\alpha_p n} \\
		&&\quad + \omega\frac{\gamma^3 L^2\beta^2}{2(1 - \beta)^4}\sum_{l=0}^{k-1}\beta^{k-1-l}\EE\|\nabla f(x^{l})\|_2^2
	\end{eqnarray*}
	with $\omega\eqdef \frac{n-1}{n} + \frac{1}{n\alpha_p}$.
	We, thus, obtain
	\begin{eqnarray*}
		\EE f(z^{k+1}) &\le& \EE f(z^k) - \frac{\gamma}{2(1-\beta)}\left(1 - \frac{L \gamma\omega}{1-\beta} \right)\EE\|\nabla f(x^k)\|_2^2\\
		&&\quad  + \frac{L\gamma^2 \sigma^2}{2n\alpha_p(1-\beta)^2} + \frac{\gamma^3 L^2\beta^2\sigma^2}{2(1 - \beta)^5\alpha_p n}\\
		&&\quad + \frac{\gamma^3 L^2\beta^2(1-\alpha_p)\zeta^2}{2(1 - \beta)^5\alpha_p n} + \omega\frac{\gamma^3 L^2\beta^2}{2(1 - \beta)^4} \sum_{l=0}^{k-1}\beta^{k-1-l}\EE\|\nabla f(x^{l})\|_2^2.
	\end{eqnarray*}
	Telescoping this inequality from 0 to $k-1$, we get
	\begin{eqnarray*}
		\EE f(z^k) - f(z^0)
		&\le& k\left( \frac{L\gamma^2 \sigma^2}{2\alpha_p n(1-\beta)^2} + \frac{\gamma^3 L^2\beta^2\sigma^2}{2(1 - \beta)^5\alpha_p n} + \frac{\gamma^3 L^2\beta^2(1-\alpha_p)\zeta^2}{2(1 - \beta)^5\alpha_p n} \right)\\
		&&\quad + \frac{\gamma}{2}\sum_{l=0}^{k-2}\left(\omega\frac{\gamma^2 L^2\beta^2}{(1 - \beta)^4}\sum_{k'=l+1}^{k-1}\beta^{k'-1-l} + \frac{L\gamma\omega}{(1-\beta)^2} - \frac{1}{1-\beta}\right)\|\nabla f(x^{l})\|_2^2 \\
		&&\quad + \frac{\gamma}{2}\left(\frac{L\gamma\omega}{(1-\beta)^2} - \frac{1}{1-\beta}\right)\EE\|\nabla f(x^{k-1})\|_2^2\\
		&\le& k\left( \frac{L\gamma^2 \sigma^2}{2\alpha_p n(1-\beta)^2} + \frac{\gamma^3 L^2\beta^2\sigma^2}{2(1 - \beta)^5\alpha_p n} + \frac{\gamma^3 L^2\beta^2(1-\alpha_p)\zeta^2}{2(1 - \beta)^5\alpha_p n} \right)\\
		&&\quad + \frac{\gamma}{2}\sum_{l=0}^{k-1}\left(\omega\frac{\gamma^2 L^2\beta^2}{(1 - \beta)^5} + \frac{L\gamma\omega}{(1-\beta)^2} - \frac{1}{1-\beta}\right)\|\nabla f(x^{l})\|_2^2 .
	\end{eqnarray*}
	It holds $f^*\le f(z^k)$ and our assumption on $\beta$ implies that $\omega\frac{\gamma^2 L^2\beta^2}{(1 - \beta)^5} + \frac{L\gamma\omega}{(1-\beta)^2} - \frac{1}{1-\beta} \le -\frac{1}{2}$, so it all results in
	\begin{eqnarray*}
		\frac{1}{k}\sum_{l=0}^{k-1} \|\nabla f(x^{l})\|_2^2 \le  \frac{4(f(z^0) - f^*)}{\gamma k} + 2\gamma\frac{L \sigma^2}{\alpha_p  n(1-\beta)^2} + 2\gamma^2\frac{ L^2\beta^2\sigma^2}{(1 - \beta)^5\alpha_p n} + 2\gamma^2\frac{L^2\beta^2(1-\alpha_p)\zeta^2}{2(1 - \beta)^5\alpha_p n}.
	\end{eqnarray*}
	Since $\overline x^k$ is sampled uniformly from $\{x^0, \dotsc, x^{k-1}\}$, the left-hand side is equal to $\EE \|\nabla f(\overline x^k)\|_2^2$. Also note that $z^0=x^0$.
\end{proof}
\begin{corollary}\label{cor:TernGrad-momentum}
		If we set $\gamma=\frac{1-\beta^2}{2\sqrt{k}L\omega}$, where $\omega = \frac{n-1}{n} + \frac{1}{n\alpha_p}$, and $\beta$ such that $\frac{\beta^2}{(1 - \beta)^3}\le 4k\omega$ with $k>1$, then the accuracy after $k$ iterations is at most 
		\[
		\frac{1}{\sqrt{k}}\left(\frac{8L\omega(f(x^0)-f^*)}{1-\beta^2} + \frac{(1+\beta)\sigma^2}{\omega\alpha_p n(1-\beta)}\right) + \frac{1}{k}\frac{(1+\beta)^4\beta^2\sigma^2}{2(1 - \beta)\omega\alpha_pn} + \frac{1}{k}\frac{(1+\beta)^4\beta^2(1-\alpha_p)\zeta^2}{2(1 - \beta)\omega\alpha_pn}.
		\]
\end{corollary}
\begin{proof}
	Our choice of $\gamma = \frac{1-\beta^2}{2\sqrt{k}L\omega}$ implies that
	\[
		\frac{\beta^2}{(1 - \beta)^3}\le \frac{1 - \beta^2 - 2L\gamma\omega}{\gamma^2 L^2\omega} \Longleftrightarrow \frac{\beta^2}{(1 - \beta)^3}\le 4k\omega.
	\]
	After that it remains to put $\gamma = \frac{1-\beta^2}{2\sqrt{k}L\omega}$ in $\frac{4(f(z^0) - f^*)}{\gamma k} + 2\gamma\frac{L \sigma^2}{\alpha_p  n(1-\beta)^2} + 2\gamma^2\frac{ L^2\beta^2\sigma^2}{(1 - \beta)^5\alpha_p n} + \frac{1}{k}\frac{(1+\beta)^4\beta^2(1-\alpha_p)\zeta^2}{2(1 - \beta)\omega\alpha_pn}$ to get the desired result.
\end{proof}

\subsection{Strongly convex analysis}\label{sec:TernGrad-strongly-convex}
\begin{theorem}\label{thm:terngrad_strg_cvx_prox}
	Assume that each function $f_i$ is $\mu$-strongly convex and $L$-smooth. Choose stepsizes $\gamma^k = \gamma > 0$ satisfying
\begin{equation}\label{eq:gamma_cond_terngrad}\gamma \le \frac{2n\alpha_p}{(\mu + L)(2+(n-2)\alpha_p)}.\end{equation}
If we run Algorithm $\ref{alg:terngrad}$ for $k$ iterations with $\gamma^k=\gamma$, then
	\begin{eqnarray*}
    	\EE\left[\|x^k - x^*\|_2^2\right] \le (1 - \gamma\mu)^k \|x^0-x^*\|_2^2 +  \frac{\gamma}{\mu}\left(\frac{\sigma^2}{n\alpha_p} + \frac{2(1-\alpha_p)}{n^2\alpha_p}\sumin\|h_i^*\|_2^2\right),
    \end{eqnarray*}
    where $\sigma^2\eqdef \frac{1}{n}\sumin\sigma_i^2$ and $h_i^* = \nabla f_i(x^*)$.
\end{theorem}
\begin{proof}
	In the similar way as we did in the proof of Theorem~\ref{thm:DIANA-strongly_convex} one can derive inequality~\eqref{eq:buf89gh38bf98} for the iterates of {\tt TernGrad}:
	\begin{eqnarray*}
       \EE \|x^{k+1} - x^*\|_2^2 
        &\le & 
        \EE \|x^k - x^*\|_2^2 - 2\gamma \EE \< \nabla f(x^k) - h^*, x^k - x^*> \notag \\
        && \qquad + \frac{\gamma^2}{n} \sumin \EE \|\nabla f_i(x^k) - h_i^*\|_2^2 + \frac{\gamma^2}{n^2}\sum_{i=1}^n\left(\EE \Psi(g_i^k)\right) + \frac{\gamma^2 \sigma^2}{n}.
	\end{eqnarray*}
	By definition $\alpha_p(d_l) = \inf\limits_{x\neq 0,x\in\R^{d_l}}\frac{\|x\|_2^2}{\|x\|_1\|x\|_p} = \left(\sup\limits_{x\neq 0,x\in\R^{d_l}}\frac{\|x\|_1\|x\|_p}{\|x\|_2^2}\right)^{-1}$ and $\alpha_p = \alpha_p(\max\limits_{l=1,\ldots,m}d_l)$ which implies
	\begin{eqnarray*}
		\EE\left[\Psi_l(g_i^k) \right] &=& \EE\left[\|g_i^k(l)\|_1\|g_i^k(l)\|_p - \|g_i^k(l)\|_2^2 \right] = \EE\left[\|g_i^k(l)\|_2^2\left(\frac{\|g_i^k(l)\|_1\|g_i^k(l)\|_p}{\|g_i^k(l)\|_2^2} - 1\right) \right]\\
		&\le& \left(\frac{1}{\alpha_p(d_l)}-1\right)\EE\|g_i^k(l)\|_2^2 \le \left(\frac{1}{\alpha_p}-1\right)\EE\|g_i^k(l)\|_2^2.
	\end{eqnarray*}
	 Moreover, 
	\[
		\|g_i^k\|_2^2 = \sum_{l=1}^m\|g_i^k(l)\|_2^2,\quad \Psi(g_i^k) = \sum\limits_{l=1}^m\Psi_l(g_i^k).	
	\]
	This helps us to get the following inequality
	\begin{eqnarray*}
       \EE \|x^{k+1} - x^*\|_2^2 
        &\le & 
        \EE \|x^k - x^*\|_2^2 - 2\gamma \EE \< \nabla f(x^k) - h^*, x^k - x^*> \notag \\
        && \qquad + \frac{\gamma^2}{n} \sumin \EE \|\nabla f_i(x^k) - h_i^*\|_2^2 + \frac{\gamma^2}{n^2}\left(\frac{1}{\alpha_p}-1\right)\sum_{i=1}^n \EE\left[\|g_i^k\|_2^2\right] + \frac{\gamma^2 \sigma^2}{n}.
    \end{eqnarray*}
	Using tower property of mathematical expectation and $\EE\left[\|g_i^k \|_2^2\mid x^k\right] = \EE\left[\|g_i^k - \nabla f_i(x^k)\|_2^2\mid x^k\right] + \|\nabla f_i(x^k)\|_2^2 \le \sigma_i^2 + \|\nabla f_i(x^k)\|_2^2$ we obtain
	\[
		\EE\|g_i^k\|_2^2 \le \EE\|\nabla f_i(x^k)\|_2^2 + \sigma_i^2 \le 2\EE\|\nabla f_i(x^k) - h_i^*\|_2^2 + 2\|h_i^*\|_2^2 + \sigma_i^2,	
	\]
	where the last inequality follows from the fact that for all $x,y\in\R^n$ the inequality $\|x+y\|_2^2 \le 2\left(\|x\|_2^2 + \|y\|_2^2\right)$ holds.
	Putting all together we have
	\begin{eqnarray*}
		\EE\|x^{k+1}-x^*\|_2^2 &\le & \EE \|x^k - x^*\|_2^2 - 2\gamma \EE \< \nabla f(x^k) - h^*, x^k - x^*> \notag \\
        && \qquad + \frac{\gamma^2}{n}\left(1 + \frac{2(1-\alpha_p)}{n\alpha_p}\right) \sumin \EE \|\nabla f_i(x^k) - h_i^*\|_2^2\notag\\
        &&\qquad + \frac{2\gamma^2(1-\alpha_p)}{n^2\alpha_p}\sum_{i=1}^n \|h_i^*\|_2^2 + \frac{\gamma^2 \sigma^2}{n\alpha_p}.
	\end{eqnarray*}
	Using the splitting trick \eqref{eq:inner_product_splitting} we get
	\begin{eqnarray}
		\EE\|x^{k+1}-x^*\|_2^2 &\le & \left(1-\frac{2\gamma \mu L}{\mu+L}\right)\EE\|x^k - x^*\|_2^2\notag\\
		&&\qquad + \frac{1}{n}\left(\gamma^2\left(1 + \frac{2(1-\alpha_p)}{n\alpha_p}\right) - \frac{2\gamma}{\mu+L}\right)\sumin \EE\|\nabla f_i(x^k)-h_i^*\|_2^2\notag\\
		&& \qquad + \frac{2\gamma^2(1-\alpha_p)}{n^2\alpha_p}\sum_{i=1}^n \|h_i^*\|_2^2 + \frac{\gamma^2 \sigma^2}{n\alpha_p}\label{eq:trngrd_str_cvx_pre_final}.
	\end{eqnarray}
	Since $\gamma \le \frac{2n\alpha_p}{(\mu+L)(2+(n-2)\alpha_p)}$ the term $\left(\gamma^2\left(1 + \frac{2(1-\alpha_p)}{n\alpha_p}\right) - \frac{2\gamma}{\mu+L}\right)$ is non-negative. Moreover, since $f_i$ is $\mu$--strongly convex, we have
$\mu \|x^k-x^*\|_2^2 \leq \langle \nabla f_i(x^k) - h_i^*, x^k -x^* \rangle$. Applying the Cauchy-Schwarz inequality to further  bound the right hand side, we get the inequality $\mu \|x^k-x^*\|_2 \leq \|\nabla f_i(x^k) - h_i^*\|_2$. Using these observations, we can get rid of the second term in the \eqref{eq:trngrd_str_cvx_pre_final} and absorb it with the first term, obtaining
	\begin{eqnarray*}
		\EE\|x^{k+1}-x^*\|_2^2 &\le & \left(1 - 2\gamma\mu + \gamma^2\mu^2\left(1 + \frac{2(1-\alpha_p)}{n\alpha_p}\right)\right)\EE\|x^k - x^*\|_2^2\\
		&&\qquad + \frac{2\gamma^2(1-\alpha_p)}{n^2\alpha_p}\sum_{i=1}^n \|h_i^*\|_2^2 + \frac{\gamma^2 \sigma^2}{n\alpha_p}\\
		&\overset{\eqref{eq:conseq_gamma_choice_terngrad_prox}}{\le} & (1-\gamma\mu)\EE\|x^k - x^*\|_2^2 + \gamma^2\left(\frac{\sigma^2}{n\alpha_p} + \frac{2(1-\alpha_p)}{n^2\alpha_p}\sumin\|h_i^*\|_2^2\right).
	\end{eqnarray*}
Finally, unrolling the recurrence leads to
    \begin{align*}
    		\EE\|x^k-x^*\|_2^2 
    		&\le (1 - \gamma\mu)^k \|x^0-x^*\|_2^2 + \sum\limits_{l=0}^{k-1}(1-\gamma\mu)^l\gamma^2\left(\frac{\sigma^2}{n\alpha_p} + \frac{2(1-\alpha_p)}{n^2\alpha_p}\sumin\|h_i^*\|_2^2\right) \\
    		&\le (1 - \gamma\mu)^k \|x^0-x^*\|_2^2 + \sum\limits_{l=0}^{\infty}(1-\gamma\mu)^l\gamma^2\left(\frac{\sigma^2}{n\alpha_p} + \frac{2(1-\alpha_p)}{n^2\alpha_p}\sumin\|h_i^*\|_2^2\right) \\
    		&= (1 - \gamma\mu)^k \|x^0-x^*\|_2^2 + \frac{\gamma}{\mu}\left(\frac{\sigma^2}{n\alpha_p} + \frac{2(1-\alpha_p)}{n^2\alpha_p}\sumin\|h_i^*\|_2^2\right).
    \end{align*}
\end{proof}

\subsection{Decreasing stepsize}\label{sec:TernGrad-decreasing-stepsizes}

\begin{theorem}\label{thm:TernGrad-decreasing-stepsizes}
    Assume that $f$ is $L$-smooth, $\mu$-strongly convex and we have access to its gradients with bounded noise. Set $\gamma^k = \frac{2}{\mu k + \theta}$ with some $\theta \ge \frac{(\mu+L)(2+(n-2)\alpha_p)}{2n\alpha_p}$. After $k$ iterations of Algorithm~\ref{alg:terngrad} we have
    \begin{align*}
        \EE \|x^k - x^*\|_2^2 \le \frac{1}{\eta k+1}\max\left\{ \|x^0-x^*\|_2^2, \frac{4}{\mu\theta}\left(\frac{\sigma^2}{n\alpha_p} + \frac{2(1-\alpha_p)}{n^2\alpha_p}\sumin\|h_i^*\|_2^2\right) \right\},
    \end{align*}
    where $\eta\eqdef \frac{\mu}{\theta}$, $\sigma^2 = \frac{1}{n}\sumin\sigma_i^2$ and $h_i^* = \nabla f_i(x^*)$.
\end{theorem}
\begin{proof}
	To get a recurrence, let us recall an upper bound we have proved before in Theorem~\ref{thm:terngrad_strg_cvx_prox}:
    \[
        \EE\|x^{k+1}-x^*\|_2^2\le (1 - \gamma^k\mu)\EE\|x^k-x^*\|_2^2 + (\gamma^k)^2\left(\frac{\sigma^2}{n\alpha_p} + \frac{2(1-\alpha_p)}{n^2\alpha_p}\sumin\|h_i^*\|_2^2\right).
    \]
    Having that, we can apply Lemma~\ref{lem:sgd} to the sequence $\EE\|x^k-x^*\|_2^2$. The constants for the Lemma are: $N = \left(\frac{\sigma^2}{n\alpha_p} + \frac{2(1-\alpha_p)}{n^2\alpha_p}\sumin\|h_i^*\|_2^2\right)$ and $C=\max\left\{ \|x^0-x^*\|_2^2, \frac{4}{\mu\theta}\left(\frac{\sigma^2}{n\alpha_p} + \frac{2(1-\alpha_p)}{n^2\alpha_p}\sumin\|h_i^*\|_2^2\right) \right\}$.
\end{proof}
\begin{corollary}\label{cor:TernGrad-decreasing-stepsizes}
	If we choose $\theta=\frac{(\mu+L)(2+(n-2)\alpha_p)}{2n\alpha_p}$, then to achieve $\EE \|x^k-x^*\|_2^2\le \varepsilon$ we need at most $O\left( \frac{\kappa(1+n\alpha_p)}{n\alpha_p}\max\left\{ \|x^0-x^*\|_2^2, \frac{n\alpha_p}{(1+n\alpha_p)\mu L}\left(\frac{\sigma^2}{n\alpha_p} + \frac{1-\alpha_p}{n^2\alpha_p}\sumin\|h_i^*\|_2^2\right) \right\}\frac{1}{\varepsilon} \right)$ iterations, where $\kappa \eqdef \frac{L}{\mu}$ is the condition number of $f$.
\end{corollary}
\begin{proof}
	If $\theta=\frac{(\mu+L)(2+(n-2)\alpha_p)}{2n\alpha_p} = \Theta\left(\frac{L(1+n\alpha_p)}{n\alpha_p}\right)$, then $\eta = \Theta\left(\frac{n\alpha_p}{\kappa(1+n\alpha_p)}\right)$ and $\frac{1}{\mu\theta} = \Theta\left(\frac{n\alpha_p}{\mu L(1+n\alpha_p)}\right)$. Putting all together and using the bound from Theorem~\ref{thm:TernGrad-decreasing-stepsizes} we get the desired result.
\end{proof}

\clearpage

\section{Detailed Numerical Experiments}

\label{sec:A:detailsOfNumericalExperiments}

\subsection{Performance of DIANA, QSGD and Terngrad on the Rosenbrock function}

In Figure~\ref{fig:rosen} we illustrate the workings of {\tt DIANA}, {\tt QSGD} and {\tt TernGrad} with 2 workers on the 2-dimensional (nonconvex) Rosenbrock function: \[f(x,y)=(x-1)^2 + 10(y - x^2)^2,\] decomposed into average of $f_1=(x + 16)^2 + 10(y - x^2)^2 + 16y$ and $f_2= (x - 18)^2 + 10(y - x^2)^2 - 16y + \mathrm{const}$.
Each worker  has access to its own piece of the Rosenbrock function with parameter $a=1$ and $b=10$. The gradients used are not stochastic, and we use 1-bit version of {\tt QSGD}, so it also coincides with {\tt QGD} in that situation. For all methods, its parameters were carefully tuned except for momentum and $\alpha$, which were simply set  to $0.9$ and $0.5$ correspondingly. We see that {\tt DIANA} vastly outperforms the competing methods.

\begin{figure}[h!]
  \centering
    \includegraphics[scale=0.5 ]{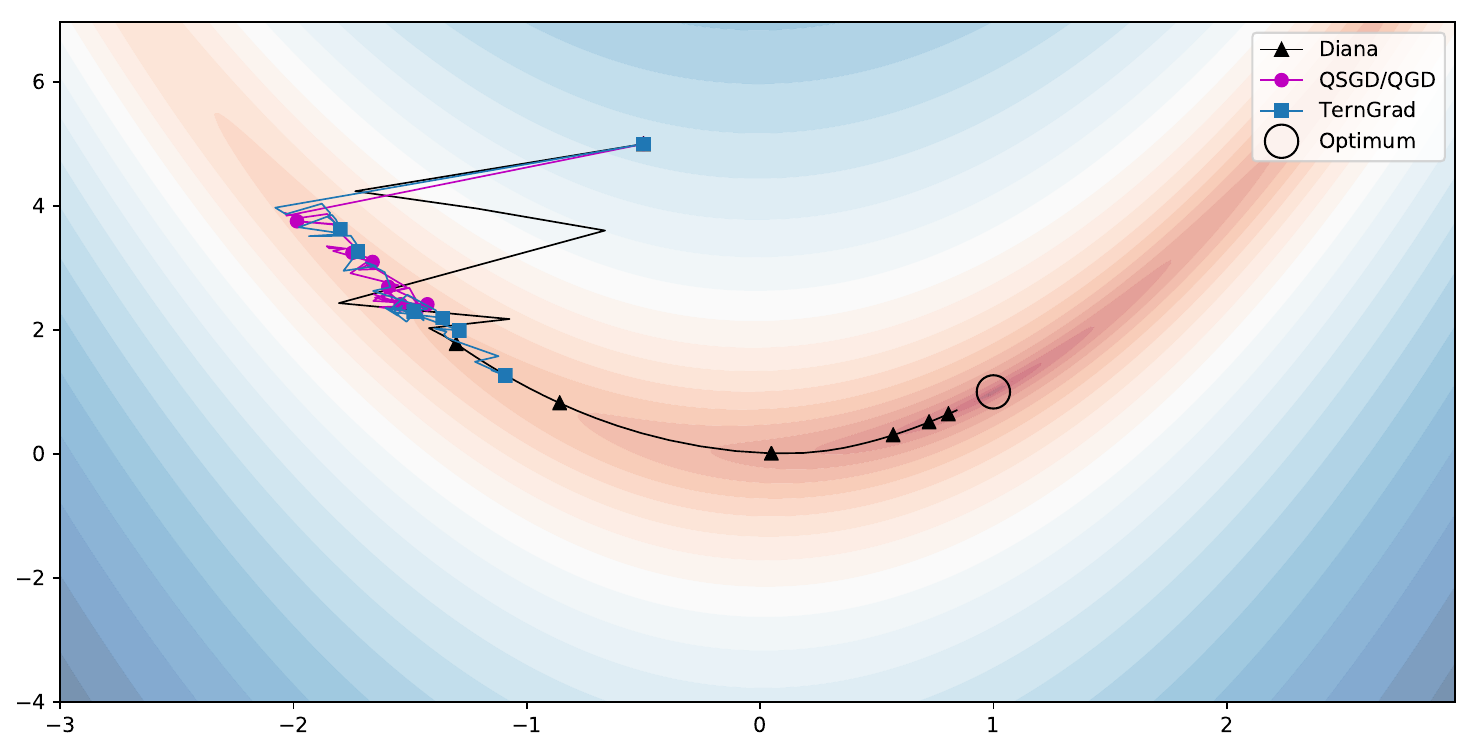}
  \caption{Illustration of the workings of {\tt DIANA}, {\tt QSGD} and {\tt TernGrad} on the Rosenbrock function.}
  \label{fig:rosen}
\end{figure}

\subsection{Logistic regression}\label{sec:log_reg}
We consider the logistic regression problem with $\ell_2$ and $\ell_1$ penalties for mushrooms dataset from LIBSVM. In our experiments we use $\ell_1$-penalty coefficient $l_1 = 2\cdot 10^{-3}$ and $\ell_2$-penalty coefficient $l_2 = \frac{L}{n}$. The coefficient $l_1$ is adjusted in order to have sparse enough solution ($\approx 20\%$ non-zero values). The main goal of this series of experiment is to compare the optimal parameters for $\ell_2$ and $\ell_\infty$ quantization.

\subsubsection{What $\alpha$ is better to choose?}
We run {\tt DIANA} with zero momentum ($\beta=0$) and obtain in our experiments that, actually, it is not important what $\alpha$ to choose for both $\ell_2$ and $\ell_\infty$ quantization. The only thing that we need to control is that $\alpha$ is small enough. 

\subsubsection{What is the optimal block-size?}
Since $\alpha$ is not so important, we run {\tt DIANA} with $\alpha = 10^{-3}$ and zero momentum ($\beta=0$) for different block sizes (see Figure~\ref{fig:block_tuning}). For the choice of $\ell_\infty$ quantization in our experiments it is always better to use full quantization. In the case of $\ell_2$ quantization it depends on the regularization: if the regularization is big then optimal block-size $\approx 25$ (dimension of the full vector of parameters is $d=112$), but if the regularization is small it is better to use small block sizes.
\begin{figure}[h!]
\centering

\includegraphics[scale=0.25]{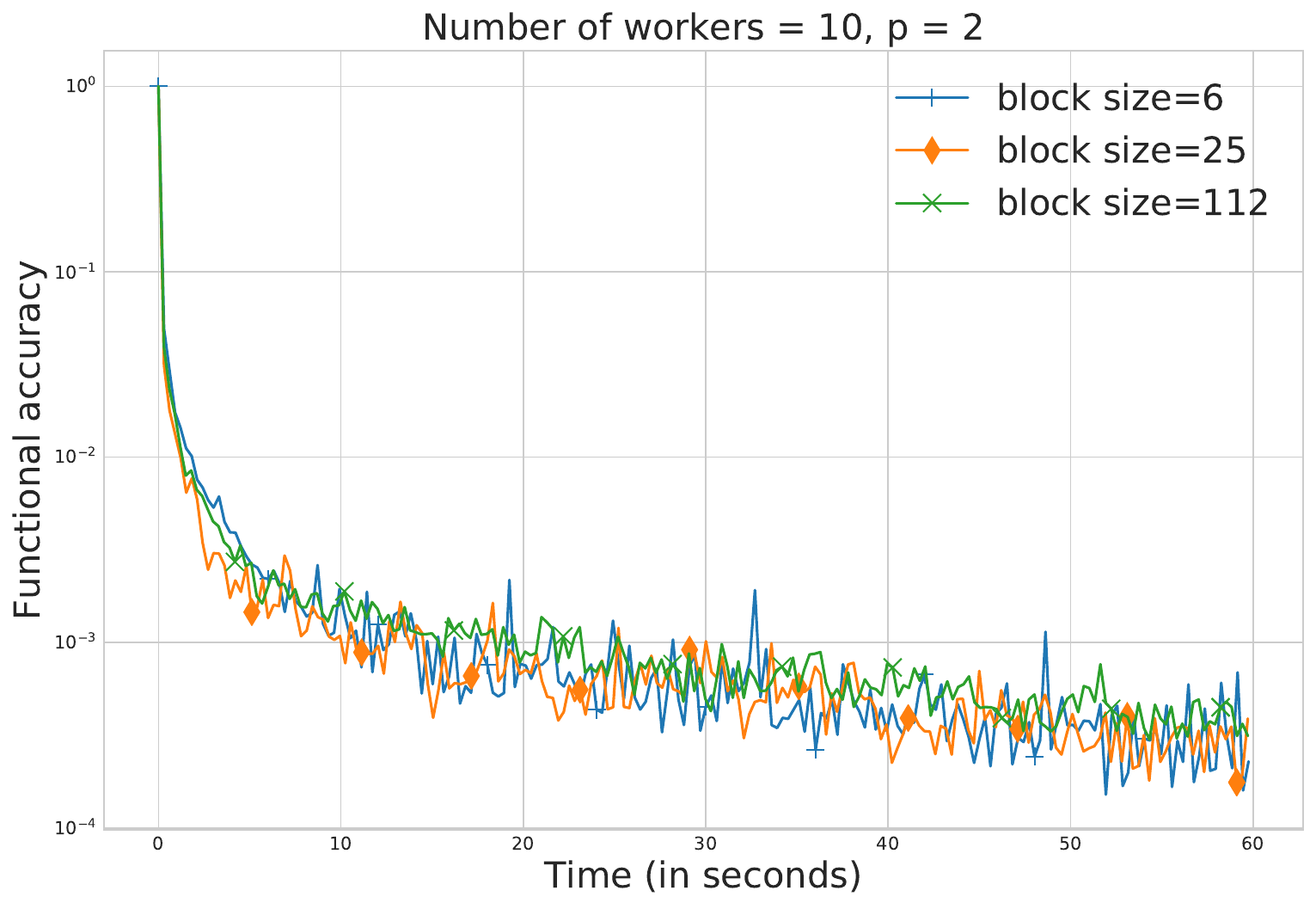}
\includegraphics[scale=0.25]{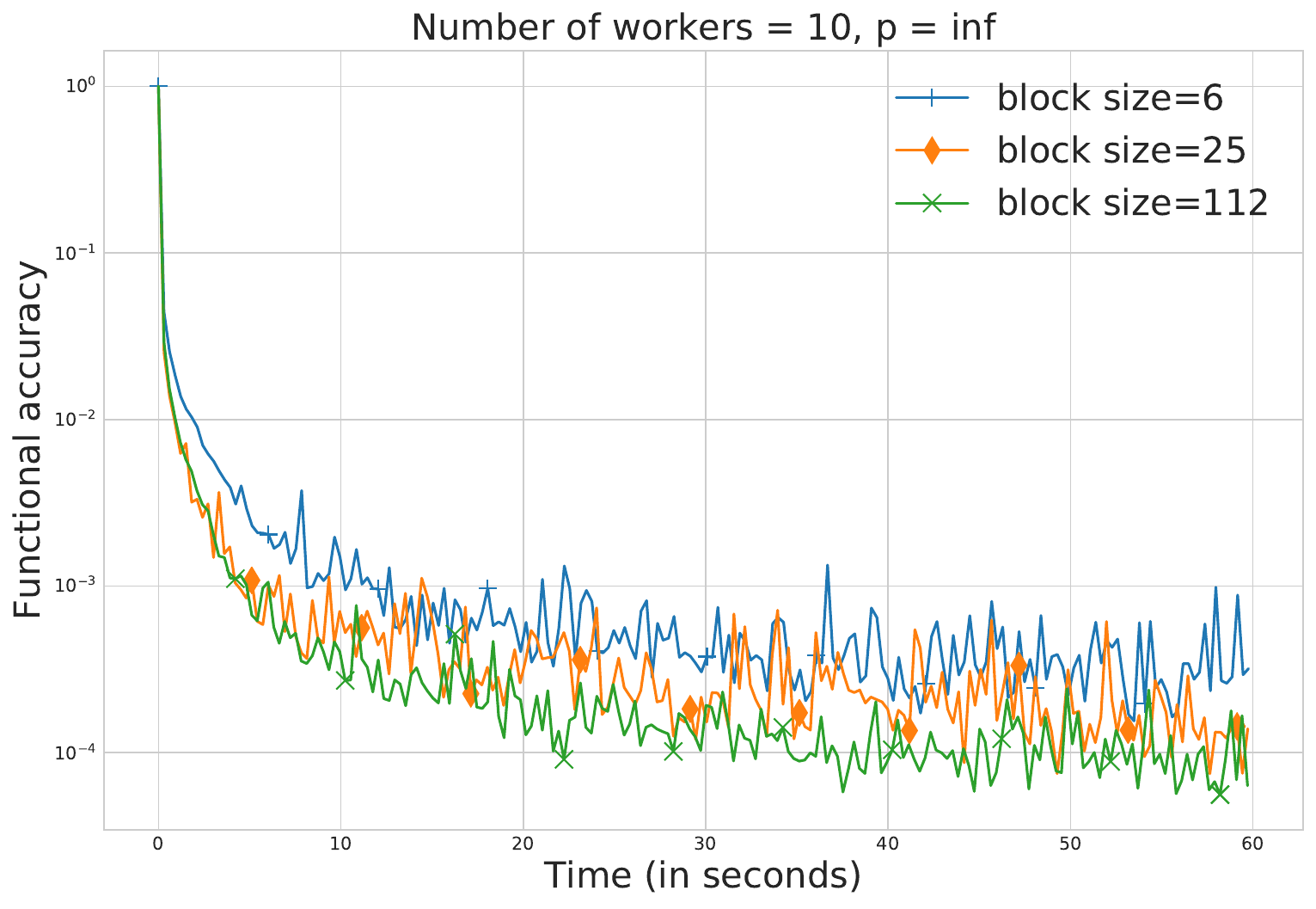}
\includegraphics[scale=0.25]{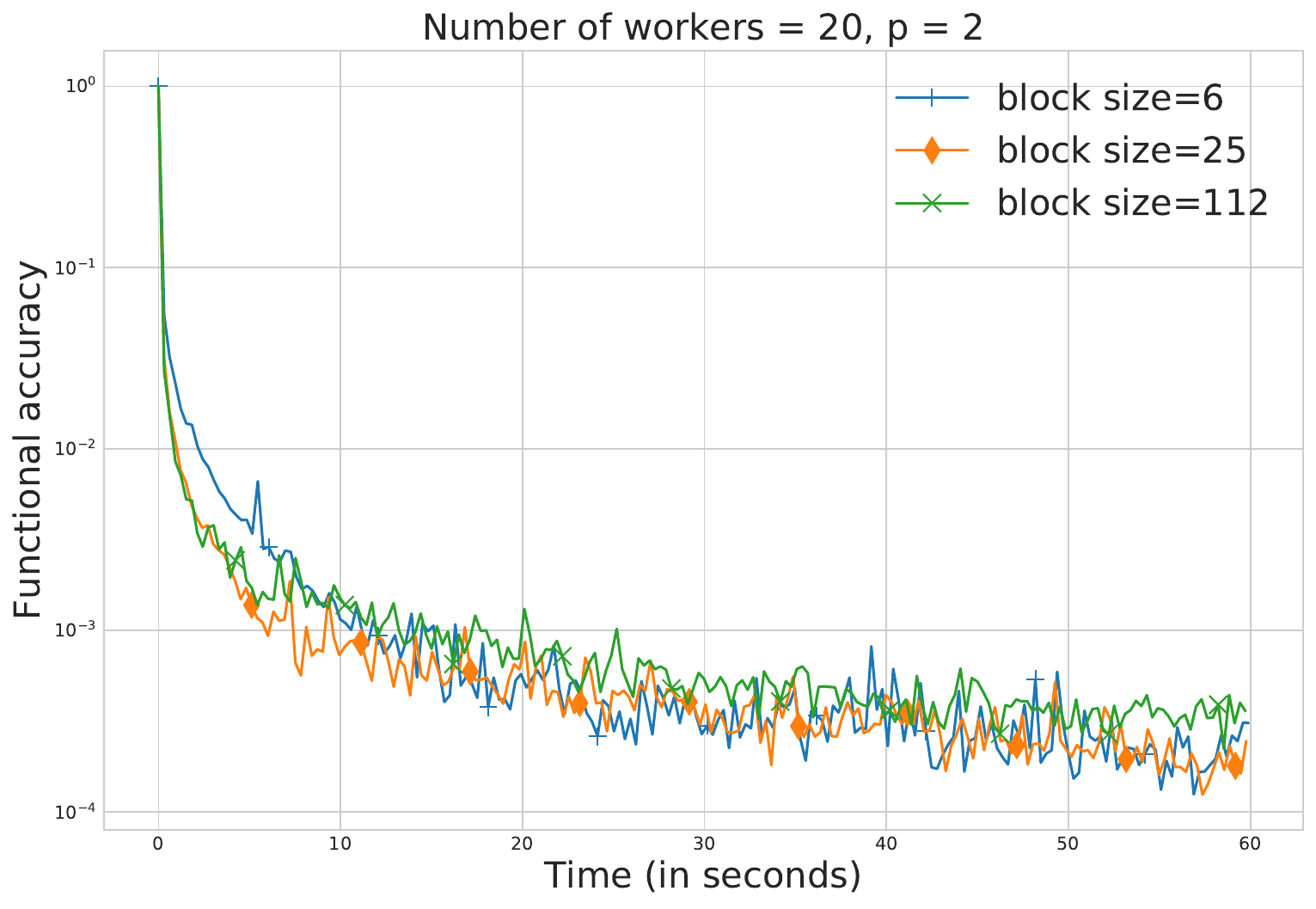}
\includegraphics[scale=0.25]{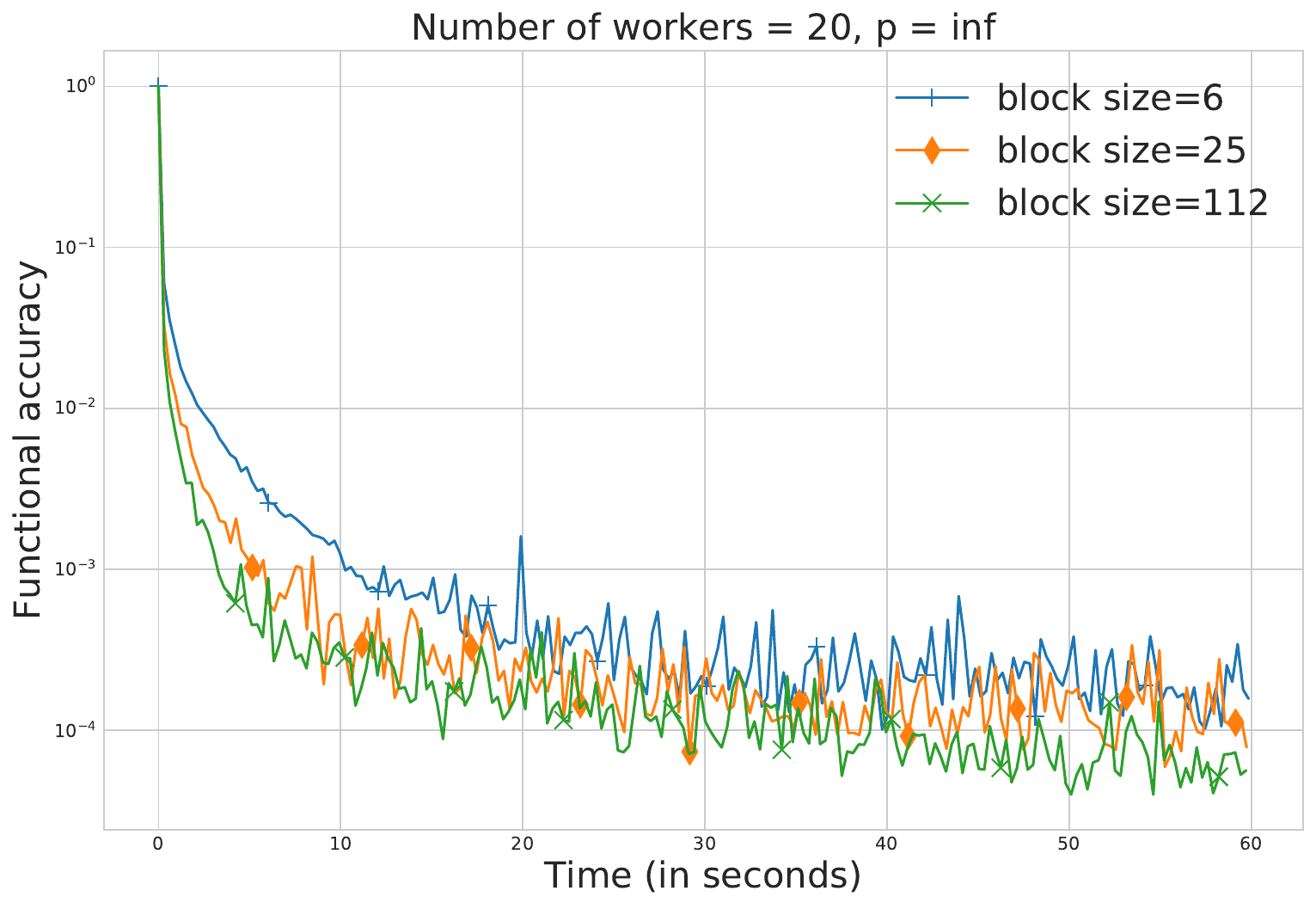}
\includegraphics[scale=0.25]{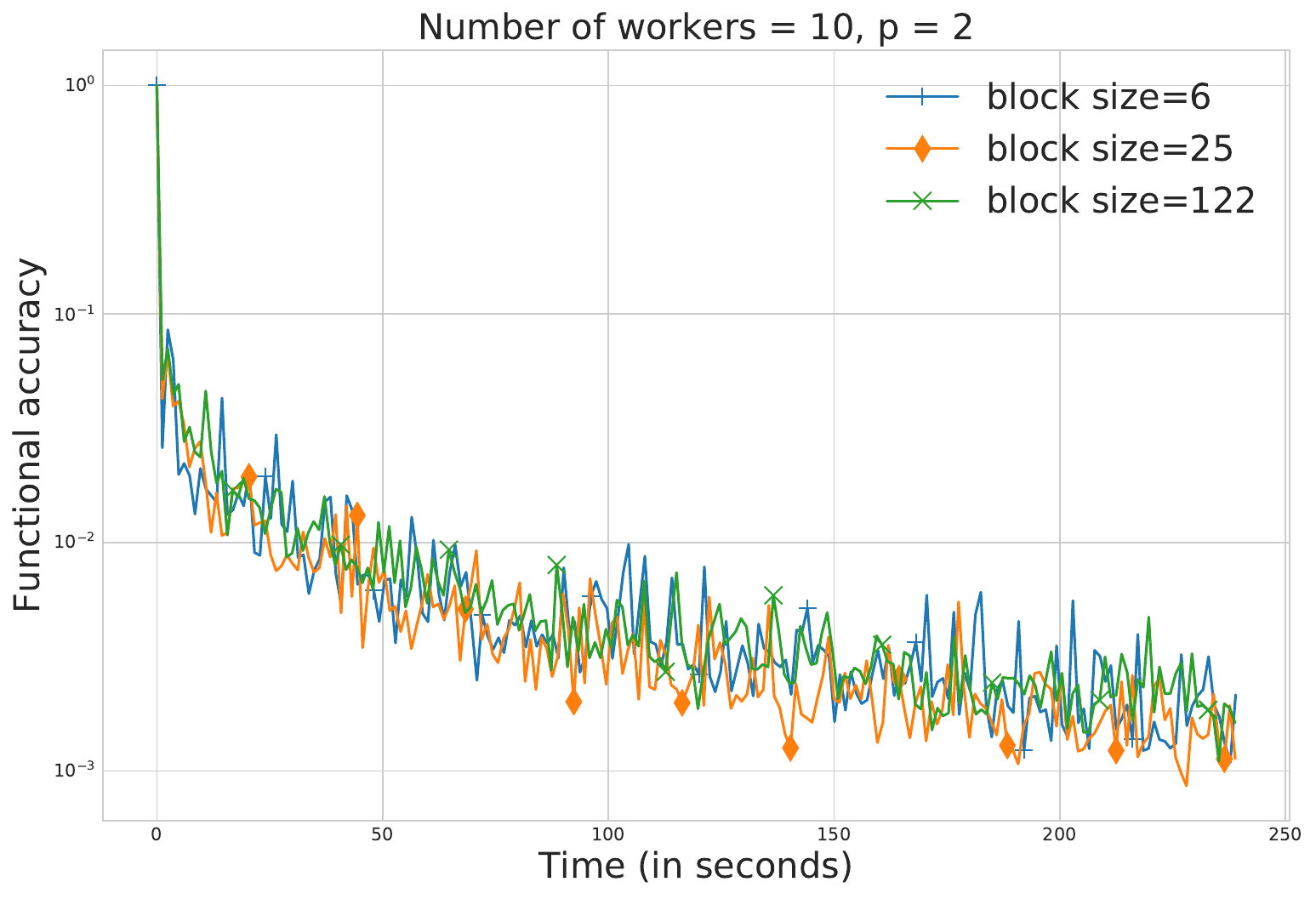}
\includegraphics[scale=0.25]{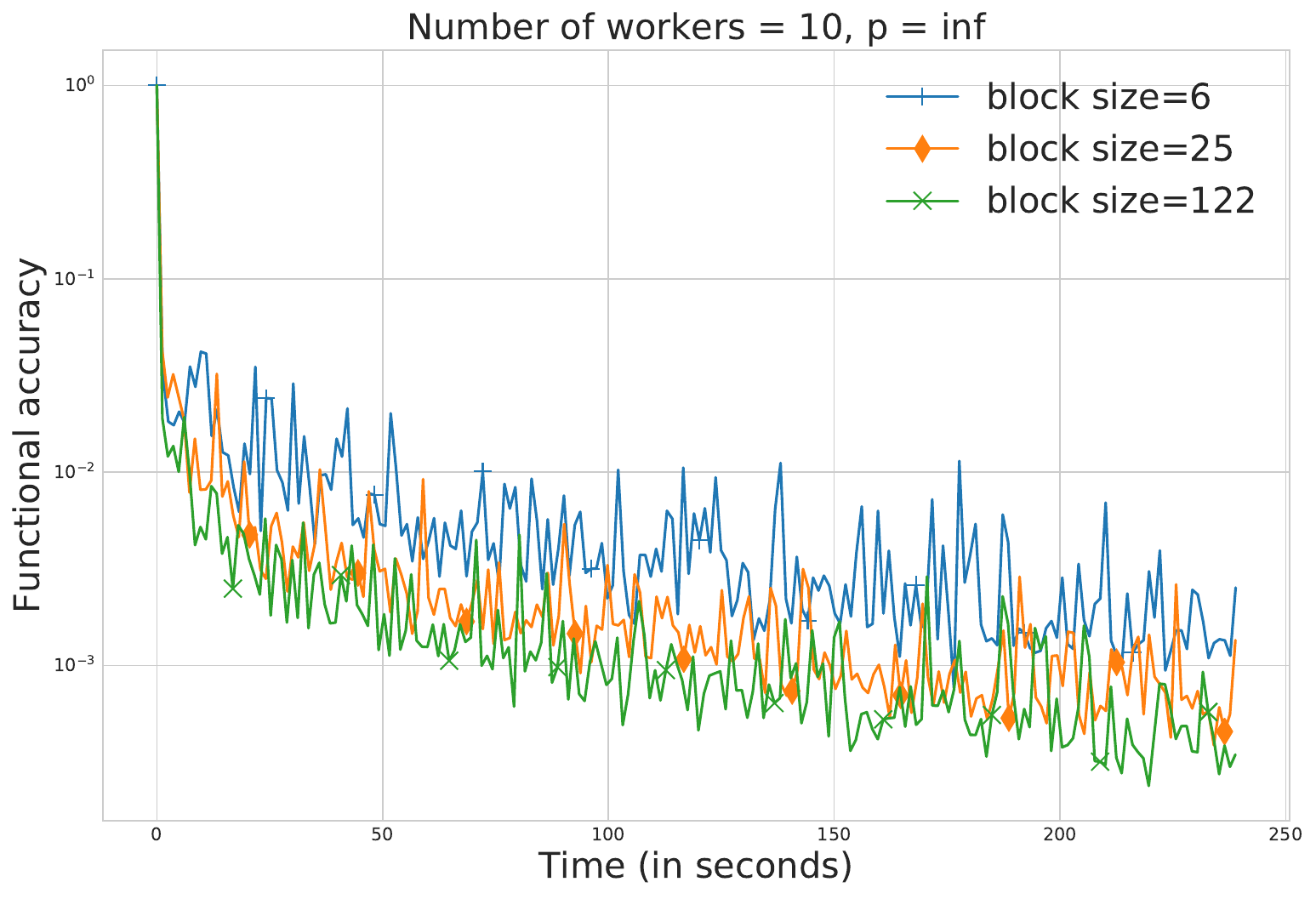}
\includegraphics[scale=0.25]{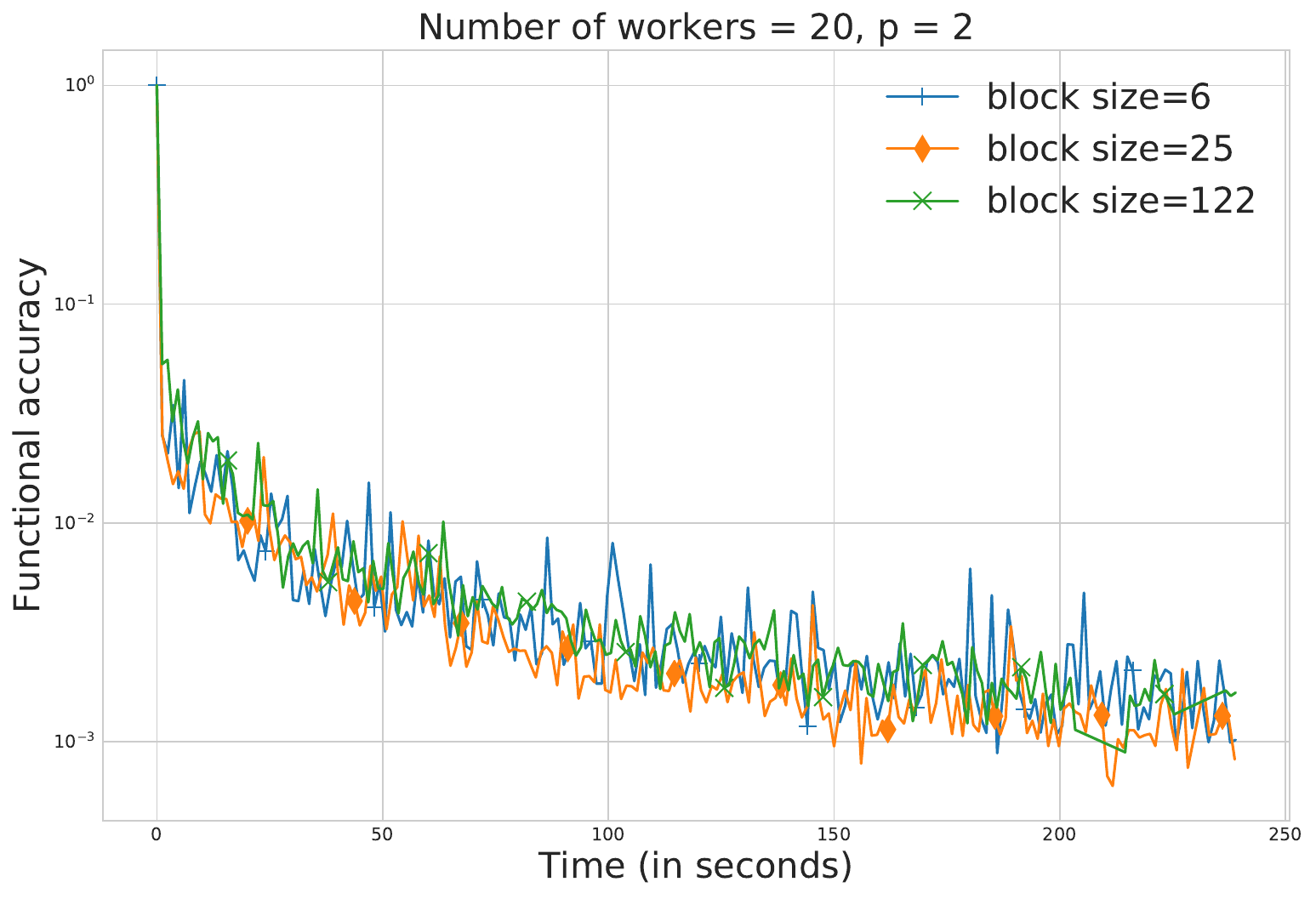}
\includegraphics[scale=0.25]{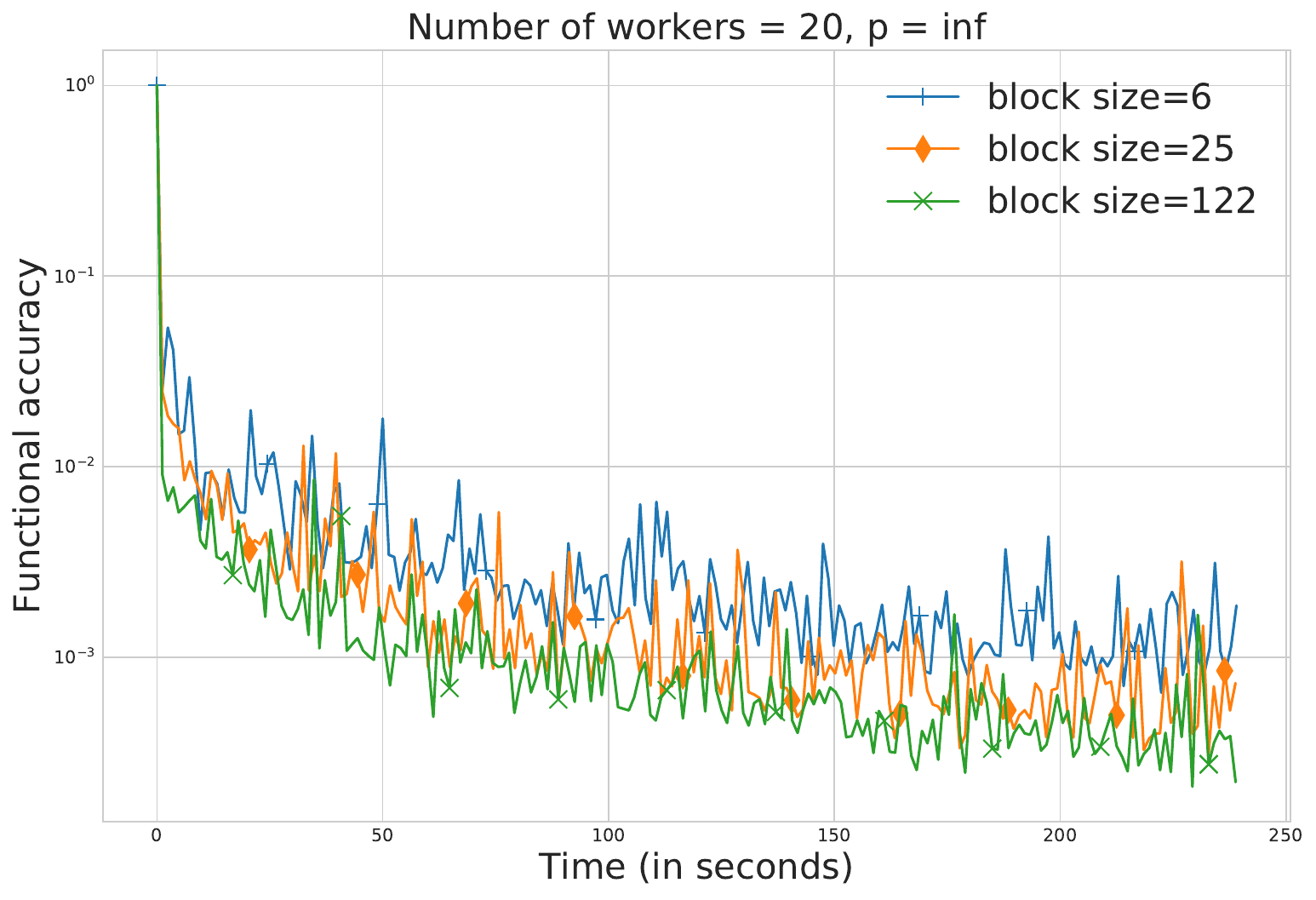}

\caption{Comparison of the influence of the block sizes on convergence for "mushrooms" (first row), "a5a" (second row)  datasets.} %Here in the first row results are presented for $10$ workers and in the second~--- $20$ workers.}%$\ell_2$ quantization and in the second row~--- for $\ell_\infty$ quantization, in the first column~--- for big $\ell_2$-penalty, in the second column~--- for small $\ell_2$-penalty.}
\label{fig:block_tuning}
\end{figure}

\begin{table}[ht!]
\begin{center}
\caption{Approximate optimal number of blocks for different dataset and configurations. Momentum equals zero for all experiments.
\label{tbl:opt_block}
}
\footnotesize
\begin{tabular}{|c|c|c|c|c||c|}
\hline
Dataset & $n$ & $d$ & Number of workers & Quantization &  Optimal block size (approx.)\\
\hline
mushrooms & $8124$ & $112$ & $10$ & $\ell_2$ & $25$\\ 
\hline
mushrooms & $8124$ & $112$ & $10$ & $\ell_\infty$ & $112$\\ 
\hline
mushrooms & $8124$ & $112$ & $20$ & $\ell_2$  & $25$\\ 
\hline
mushrooms & $8124$ & $112$ & $20$ & $\ell_\infty$ & $112$\\ 
\hline
a5a & $6414$ & $122$ & $10$ & $\ell_2$ & $25$\\ 
\hline
a5a & $6414$ & $122$ & $10$ & $\ell_\infty$ & $112$\\ 
\hline
a5a & $6414$ & $122$ & $20$ & $\ell_2$  & $25$\\ 
\hline
a5a & $6414$ & $122$ & $20$ & $\ell_\infty$ & $112$\\ 
\hline

\end{tabular}
\end{center}
\end{table}

\subsubsection{{\tt DIANA} vs {\tt QSGD} vs {\tt TernGrad} vs {\tt DQGD}}
We compare {\tt DIANA} (with momentum) with {\tt QSGD}, {\tt TernGrad} and {\tt DQGD} on the "mushrooms" dataset (See Figure~\ref{fig:diana_main}).

%\begin{figure}[h!]
%\centering
%
%\includegraphics[scale=0.5]{compare_10.pdf}
%
%\caption{Comparison of the {\tt DIANA} ($\beta = 0.95$) with {\tt QSGD}, {\tt TernGrad} and {\tt DQGD} on the logistic regression problem for the "mushrooms" dataset.}\label{fig:diana_vs_others}
%\end{figure}

\subsection{MPI - broadcast, reduce and gather}
\label{sec:A:MPI}

In our experiments, we are running 4 MPI processes per physical node.
Nodes are connected by Cray Aries High Speed Network.

We utilize 3 MPI collective operations, Broadcast, Reduce and Gather. When implementing {\tt DIANA}, we could use P2P communication, but based on our experiments, we found that using Gather to collect data from workers significantly outperformed P2P communications.

In Figure~\ref{fig:communication:scale}
we show the duration of different communications for various MPI processes and message length.
Note that Gather 2bit do not scale linearly (as would be expected). It turns out, we are not the only one who observed such a weird behavior when using cray MPI implementation (see \cite{chunduriperformance} for a nice study obtained by a team from Argonne National Laboratory).
\begin{figure}[h!]

\includegraphics[scale=0.45]{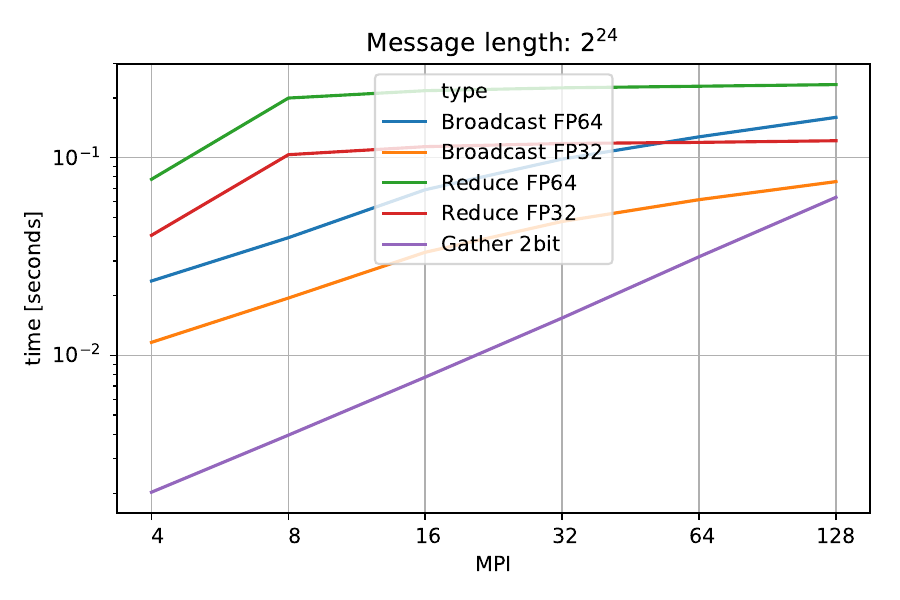}
\includegraphics[scale=0.45]{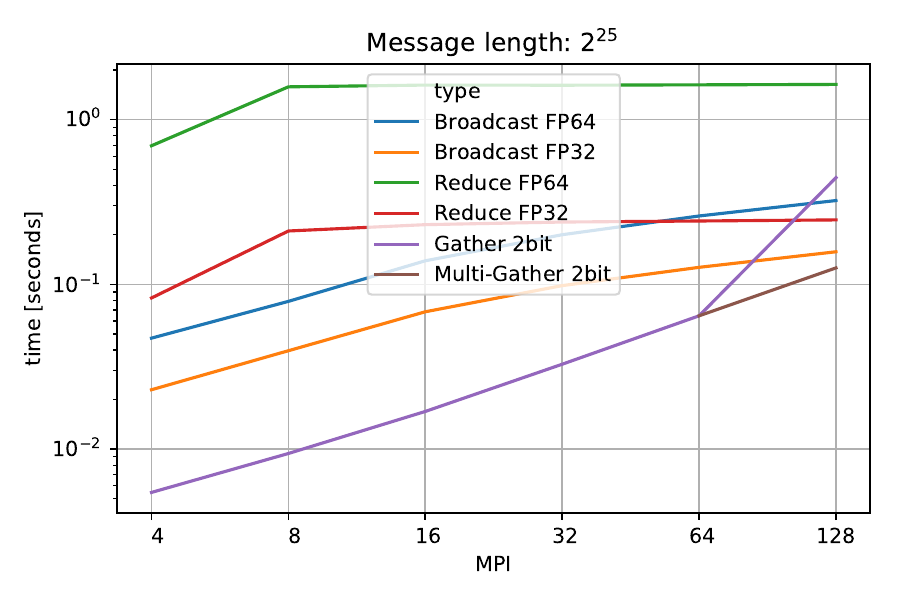}

\includegraphics[scale=0.45]{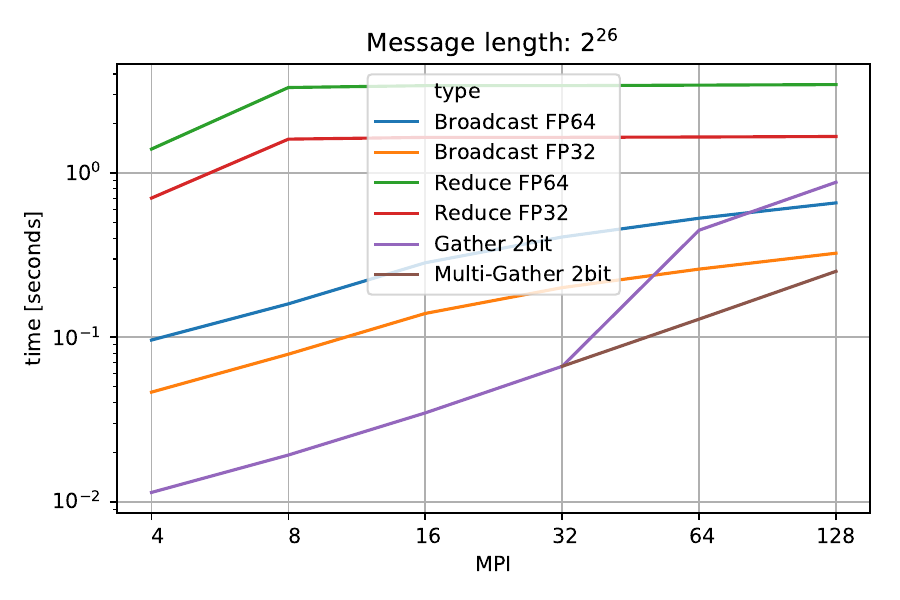}
\includegraphics[scale=0.45]{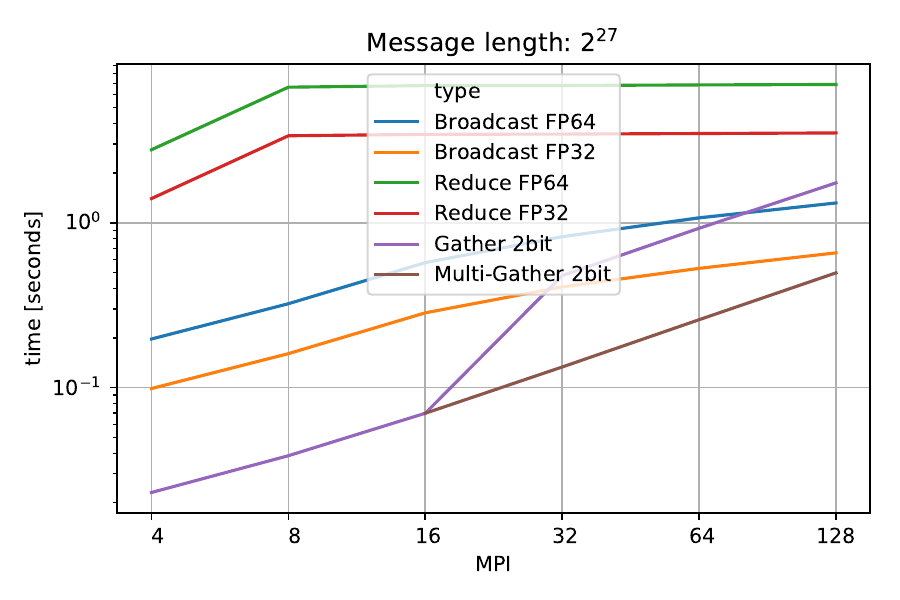}
\caption{
Time to communicate a vectors with different lengths for different methods as a function of \# of MPI processes. One can observe that Gather 2bit is not having nice scaling. We also show that the proposed Multi-Gather communication still achieves a nice scaling when more MPI processes are used.}

\label{fig:communication:scale}
\end{figure}
To correct for the unexpected behavior, 
we have performed MPI Gather multiple times on shorter vectors, such that the master node obtained all data, but in much faster time (see {\it Multi-Gather 2bit}).

\begin{figure}[b]
\centering

\includegraphics[scale=0.4]{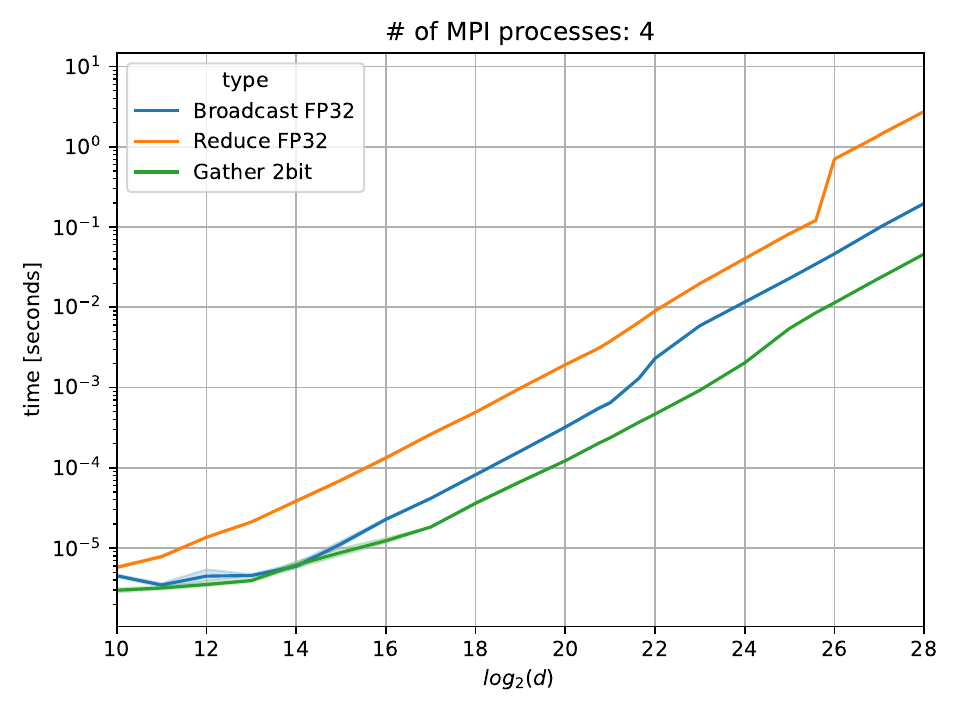}
\includegraphics[scale=0.4]{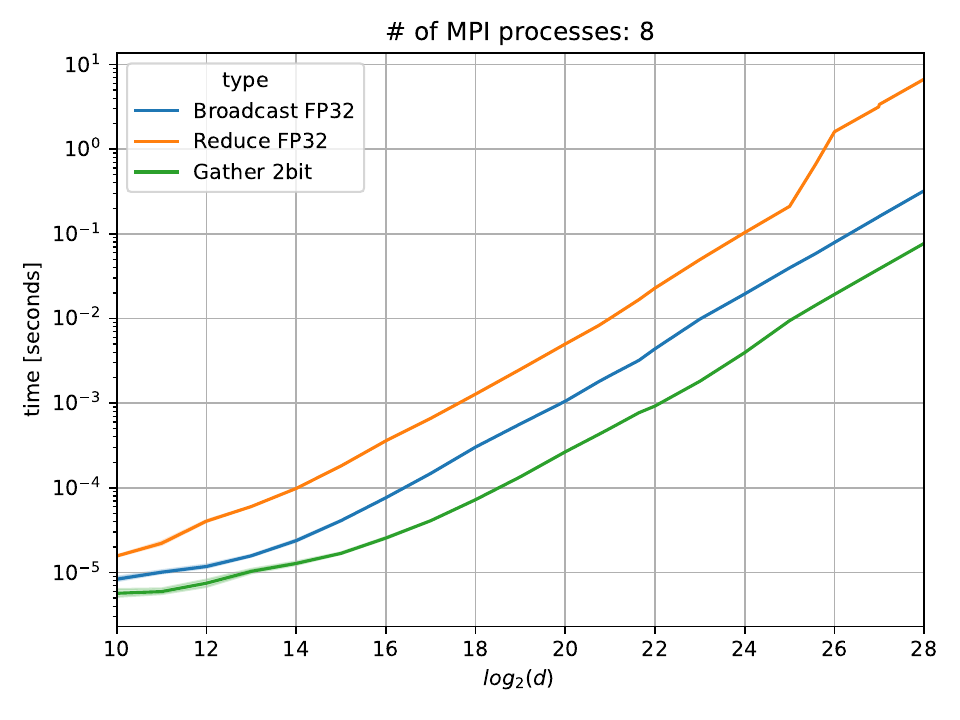}

\includegraphics[scale=0.4]{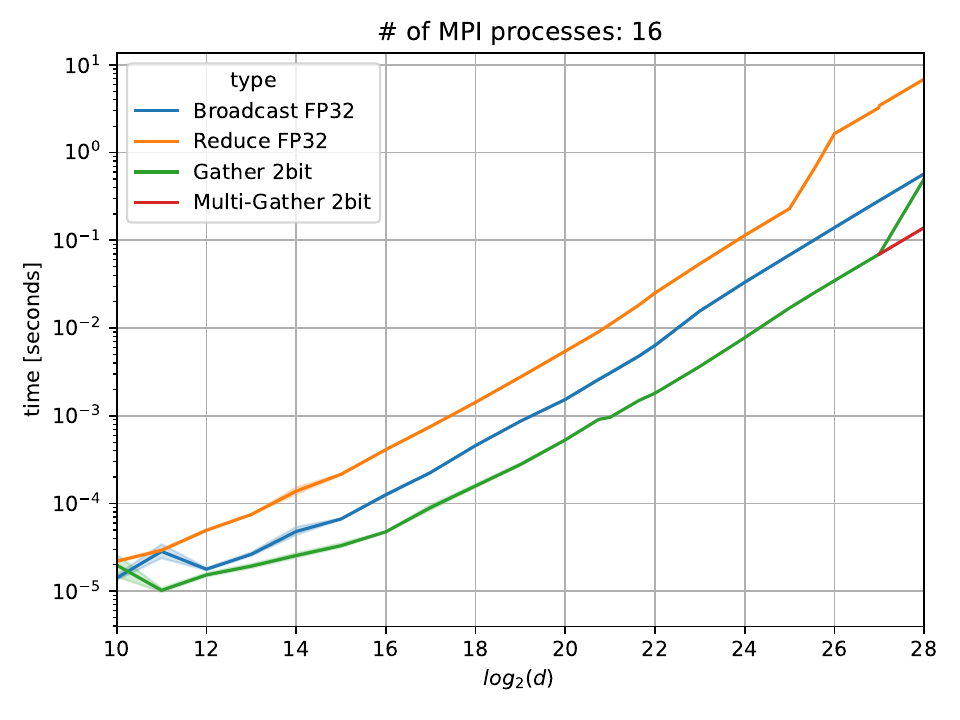}
\includegraphics[scale=0.4]{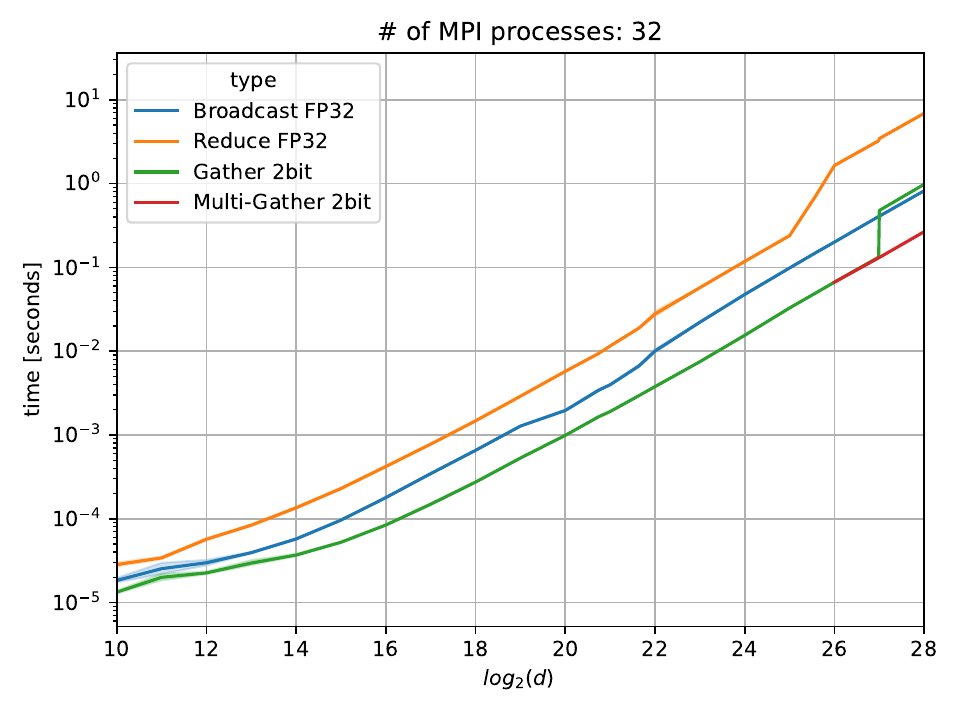}

\includegraphics[scale=0.4]{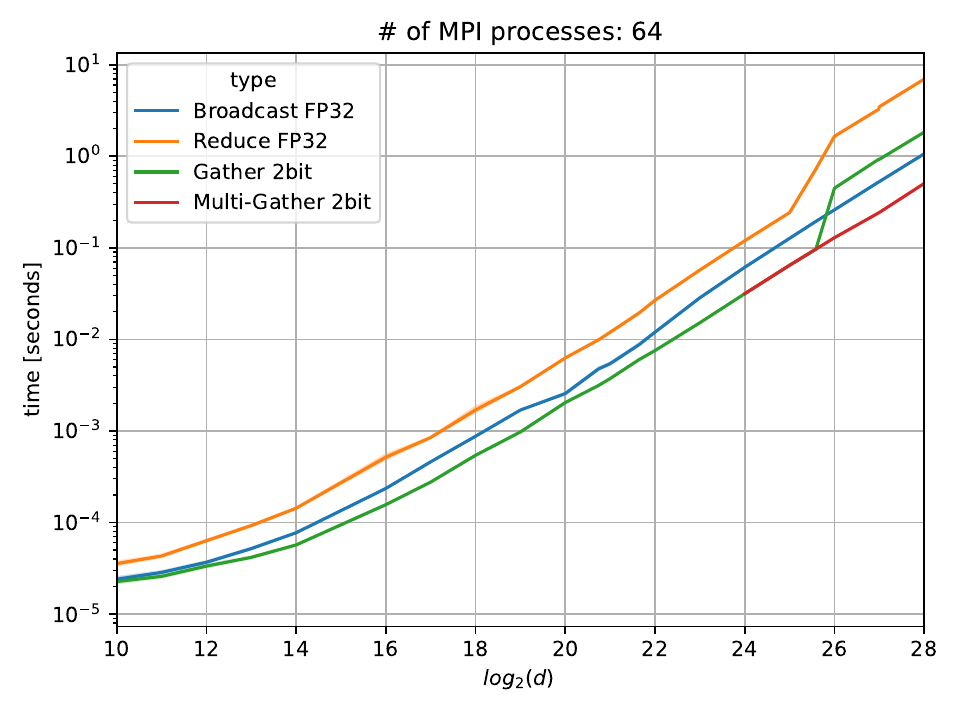}
\includegraphics[scale=0.4]{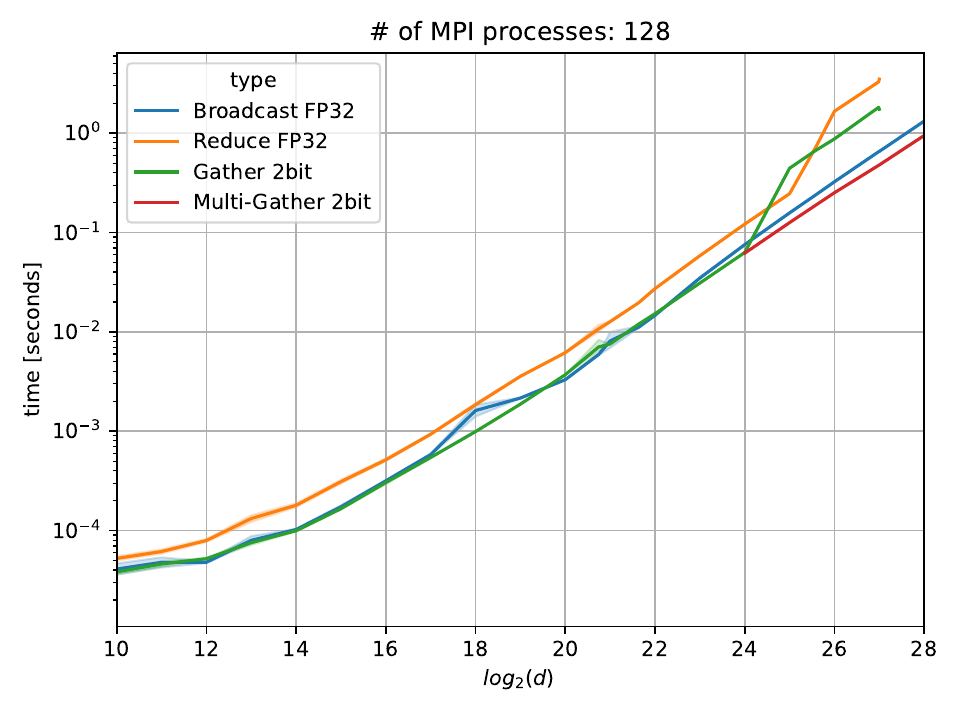}

\caption{The duration of communication for MPI Broadcast, MPI Reduce and MPI Gather. We show how the communication time depends on the size of the vector in $\R^d$ (x-axis) for various \# of MPI processes. In this experiment, we have run 4 MPI processes per computing node. For Broadcast and Reduce we have used a single precision floating point number. For Gather we used 2bits per dimension. For longer vectors and large number of MPI processes, one can observe that Gather has a very weird scaling issue. It turned out to be some weird behaviour of Cray-MPI implementation.}
\label{fig:communication_details}

\end{figure}

\begin{figure}[h!]
\centering

\includegraphics[scale=0.4]{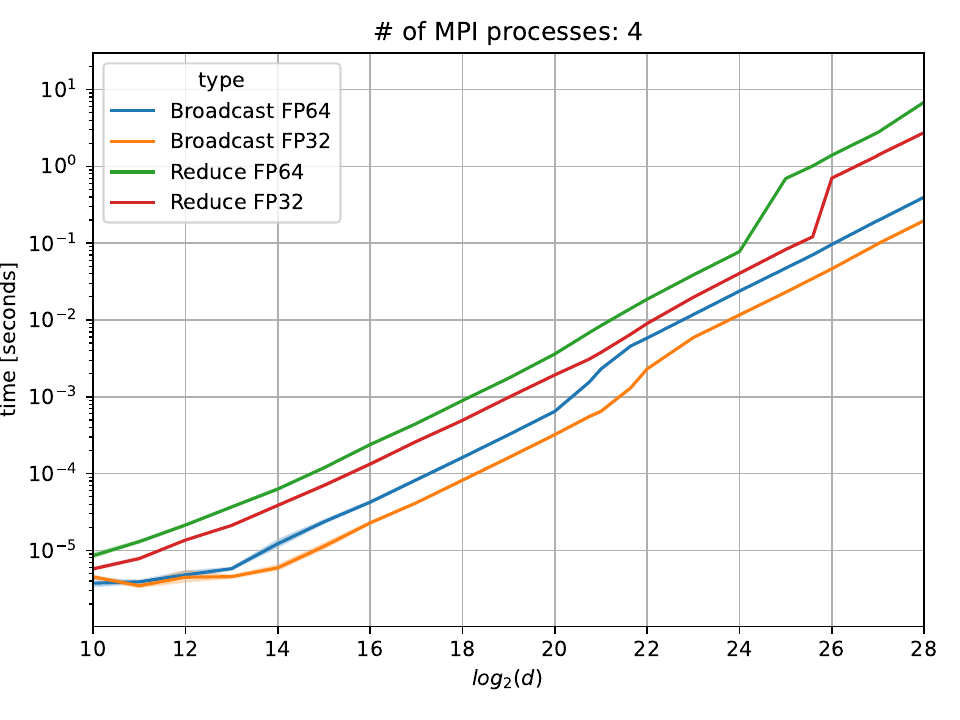}
\includegraphics[scale=0.4]{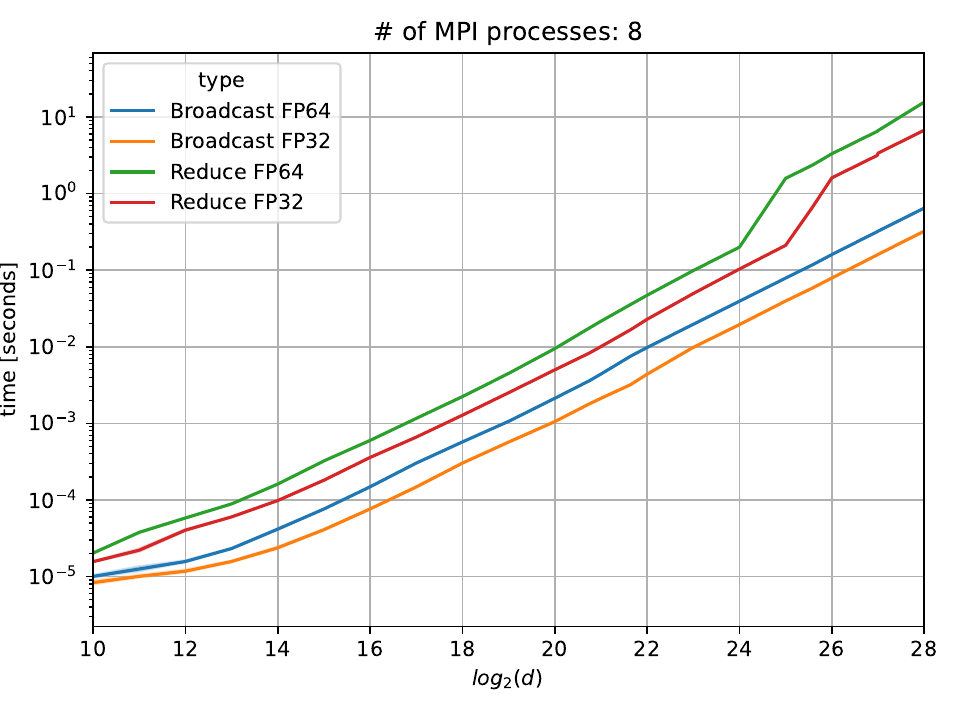}

\includegraphics[scale=0.4]{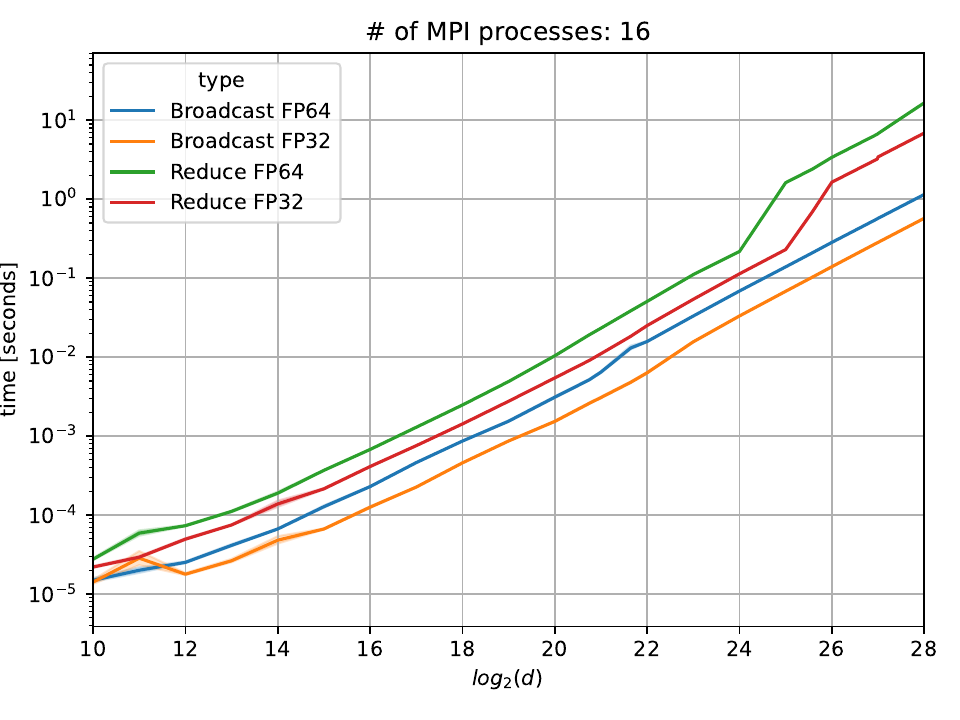}
\includegraphics[scale=0.4]{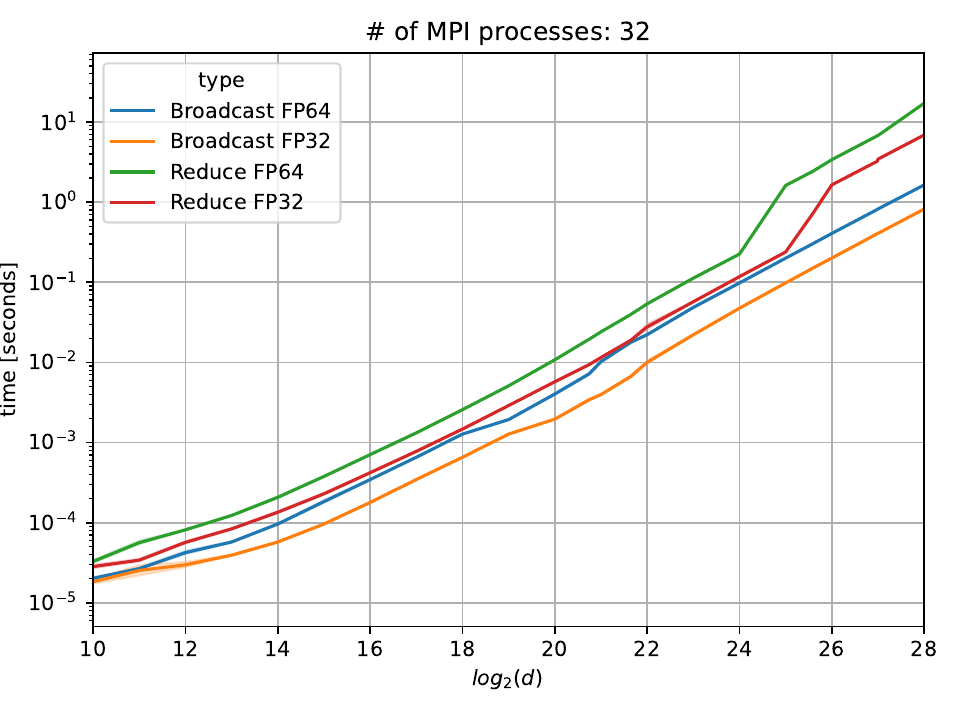}

\includegraphics[scale=0.4]{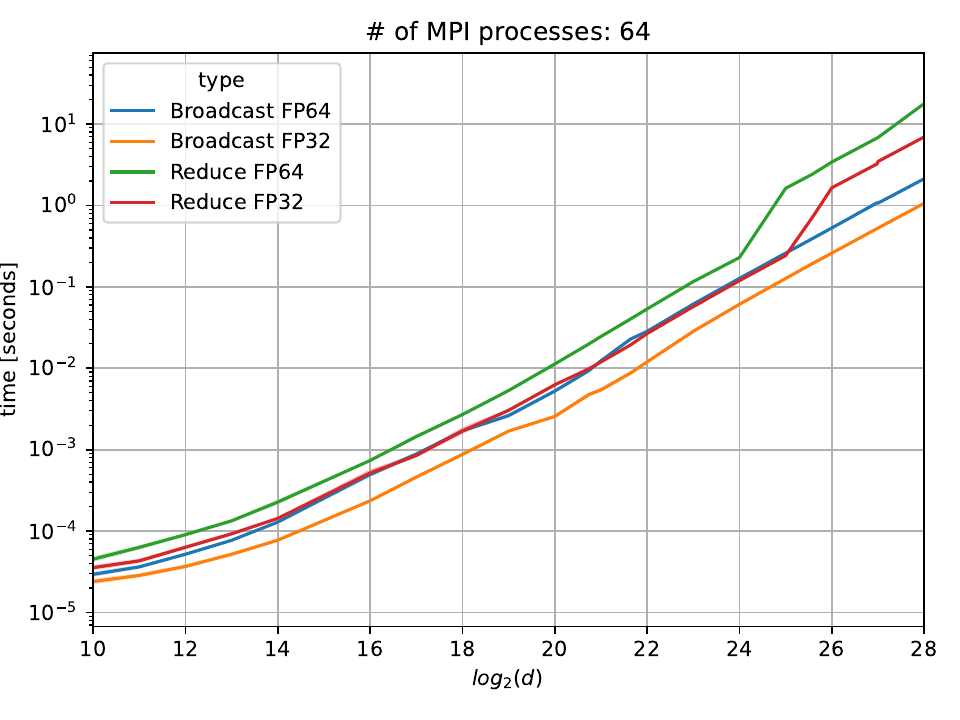}
\includegraphics[scale=0.4]{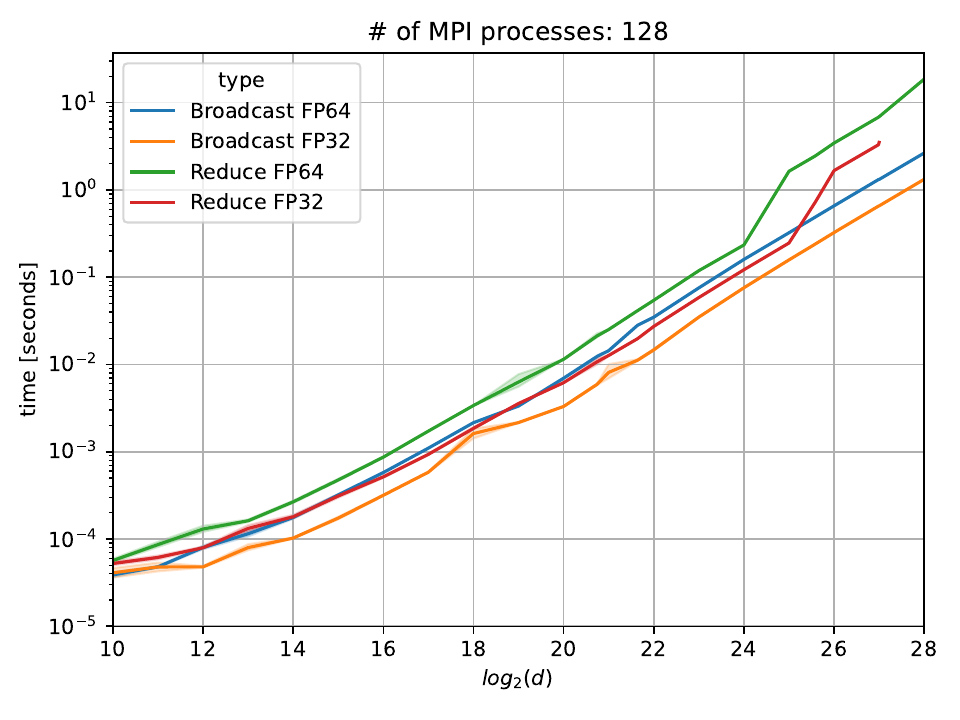}

\caption{The duration of communication for MPI Broadcast, MPI Reduce
for single precision (FP32) and double precision (FP64) floating numbers. We show how the communication time depends on the size of the vector in $\R^d$ (x-axis) for various \# of MPI processes. In this experiment, we have run 4 MPI processes per computing node. We have used Cray implementation of MPI.}
\label{fig:communication_32v64}

\end{figure}

\clearpage 
\subsection{Performance of GPU}

In Table~\ref{tbl:networks} 
we list the DNN networks we have experimented in this paper.

\begin{table}[h!] 
\centering
%\ra{1.3}
%\begin{small}

\caption{Deep Neural Networks used in the experiments. The structure of the DNN is taken from
\url{https://github.com/tensorflow/models/tree/master/research/slim}.}
\label{tbl:networks}
\begin{tabular}{lrcr}
\hline  
{\bf model} & {\bf$d$\quad} &  {\bf \# classes} & {\bf input} 
 \\  \hline \hline 
{\bf LeNet} & 3.2M & 10 & $28\times28\times3$ 
\\  

{\bf CifarNet} & 1.7M & 10 & $32\times32\times3$ 
\\

{\bf alexnet v2} & 50.3M & 1,000 & $224\times224\times3$ 
\\  
{\bf vgg a} & 132.8M  & 1,000 & $224\times224\times3$
\\
\hline 
\end{tabular}
%\end{small}

\end{table}

Figure~\ref{fig:computationCost}
shows the performance of a single P100 GPU
for different batch size, DNN network and operation.
\begin{figure}[h!]
\centering

\includegraphics[scale=0.7]{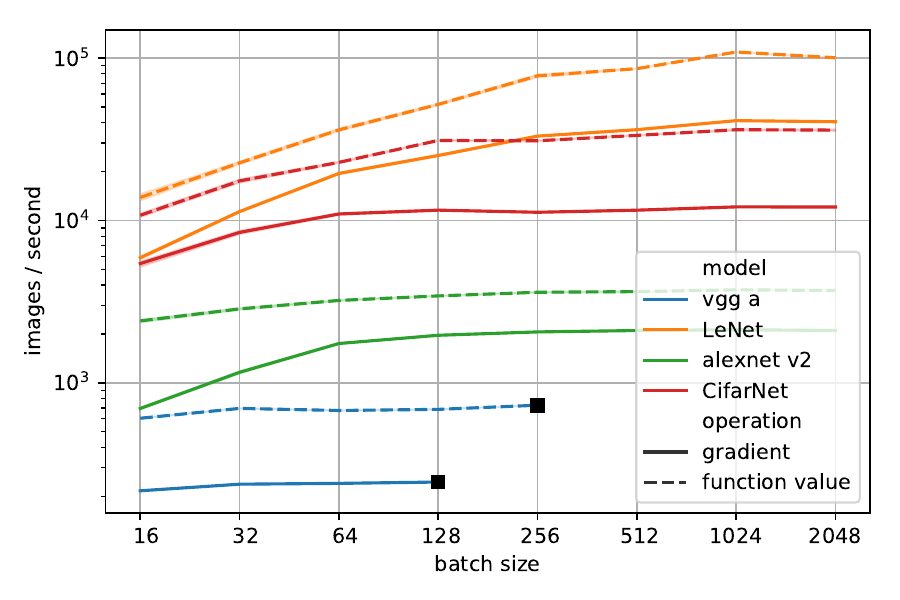}

\caption{The performance (images/second) of NVIDIA Tesla P100 GPU  on 4 different problems as a function of batch size.
We show how different choice of batch size affects the speed of function evaluation and gradient evaluation.
For vgg a, we have run out of memory on GPU for batch size larger than 128 (gradient evaluation) and 256 for function evaluation.
Clearly, this graph suggest that choosing small batch size leads to small utilization of GPU. Note that using larger batch size do not necessary reduce the training process.
}

\label{fig:computationCost}

\end{figure}

\clearpage 

\subsection{Diana vs. TenGrad, SGD and QSGD}

In Figure~\ref{fig:imagesPerSecond}
we compare the performance of {\tt DIANA} vs. doing a MPI reduce operation with 32bit floats. The computing cluster had Cray Aries High Speed Network.
However, for {\tt DIANA} we used 2bit per dimension, we have experienced an weird scaling behaviour, which was documented also in\cite{chunduriperformance}.
In our case, this affected speed for alexnet and vgg\_a beyond 64 or 32 MPI processes respectively.
For more detailed experiments, see Section~\ref{sec:A:MPI}.
In order to improve the speed of Gather, we impose a Multi-Gather strategy, when we call Gather multiple-times on shorter vectors. This significantly improved the communication cost of Gather (see 
Figure \ref{fig:communication_details}) and leads to much nicer scaling -- see green bars -- {\tt DIANA}-MultiGather in Figure~\ref{fig:imagesPerSecond}).
\begin{figure}[h!]

\centering 
\includegraphics[scale=0.45]{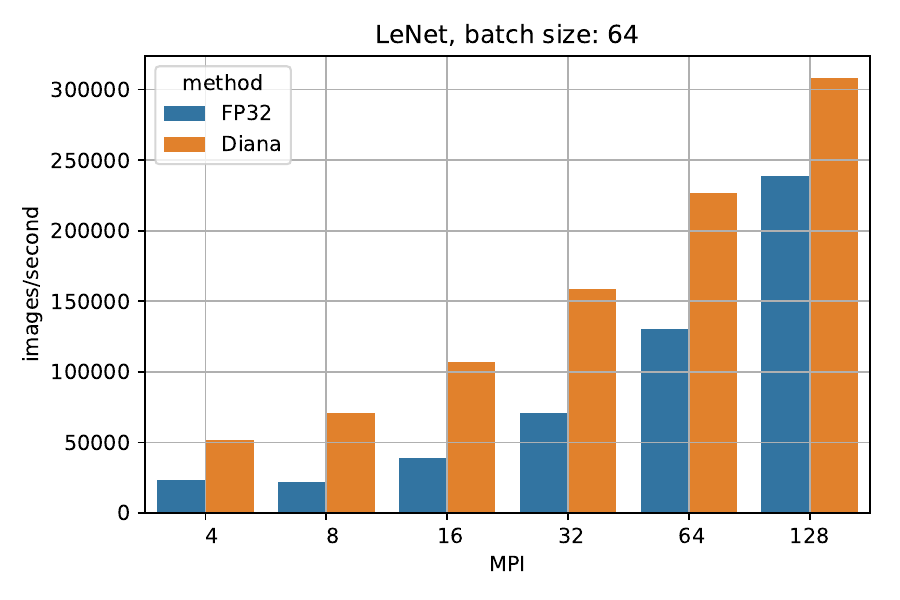}
\includegraphics[scale=0.45]{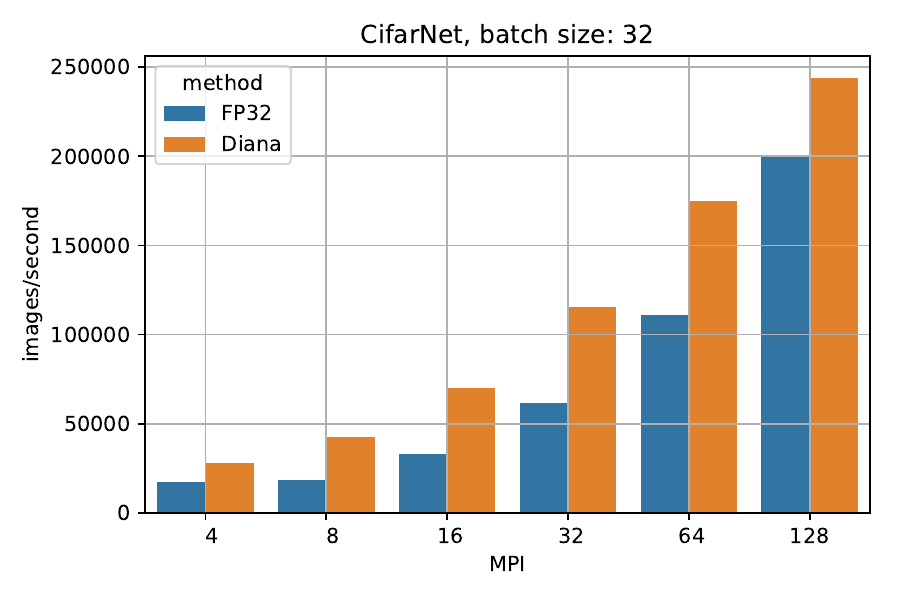}

\includegraphics[scale=0.45]{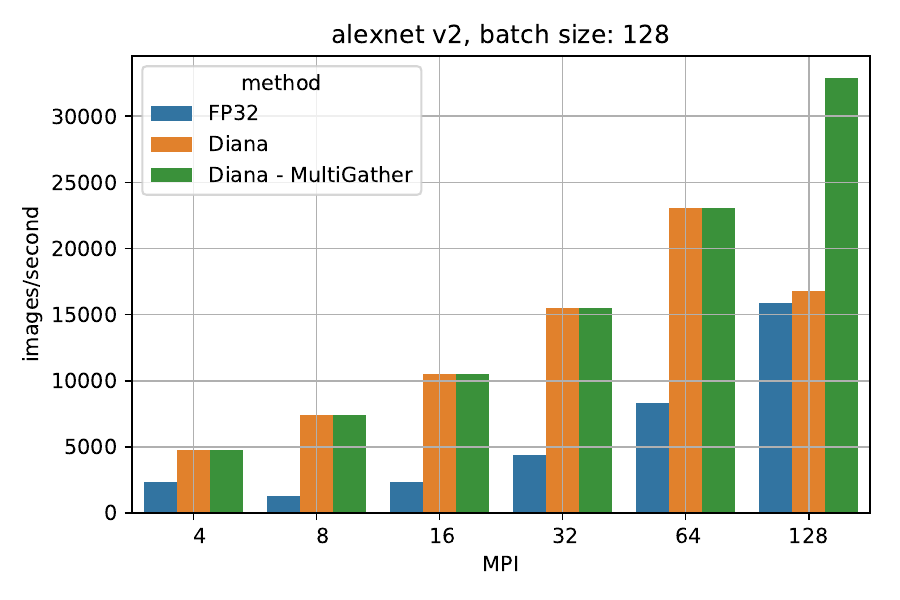}
\includegraphics[scale=0.45]{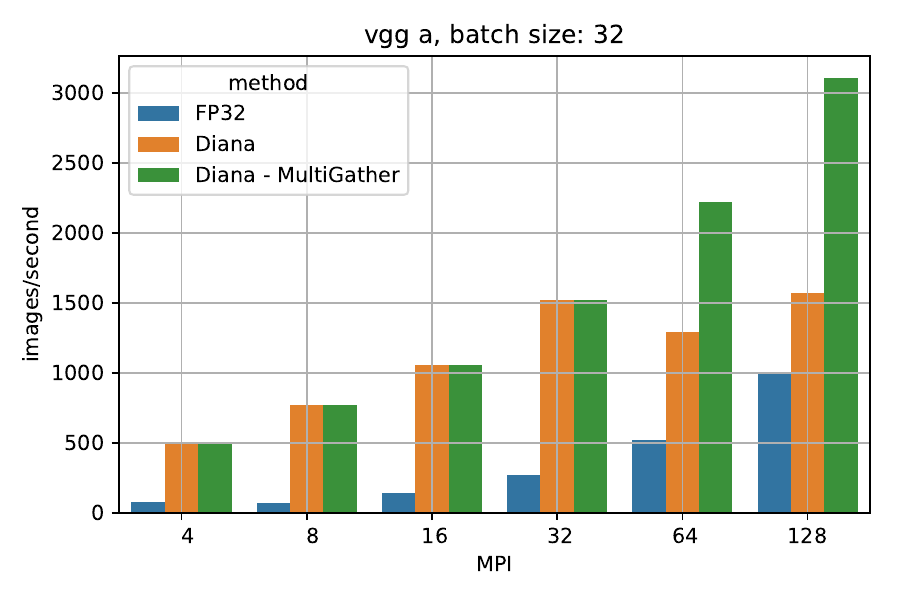}

\caption{Comparison of performance (images/second) for various number of GPUs/MPI processes and sparse communication {\tt DIANA} (2bit) vs. Reduce with 32bit float (FP32).
We have run 4 MPI processes on each node. Each MPI process is using single P100 GPU. 
Note that increasing MPI from 4 to 8 will not bring any significant improvement for FP32, because with 8 MPI processes, communication will happen between computing nodes and will be significantly slower when compare to the single node communication with 4MPI processes.
}
\label{fig:imagesPerSecond}

\end{figure}

In the next experiments, we run {\tt QSGD} \cite{alistarh2017qsgd}, {\tt TernGrad} \cite{wen2017terngrad}, {\tt SGD} with momentum and {\tt DIANA} on Mnist dataset and Cifar10 dataset for 3 epochs. We have selected 8 workers and run each method for learning rate from $\{0.1, 0.2, 0.05\}$.
For {\tt QSGD}, {\tt DIANA} and {\tt TernGrad}, we also tried various quantization bucket sizes in $\{32, 128, 512\}$.
For {\tt QSGD} we have chosen $2,4,8$ quantization levels.
For {\tt DIANA} we have chosen $\alpha \in 
\{0, 1.0/\sqrt{\mbox{quantization bucket sizes }}\}$
and have selected initial $h = 0$. 
For {\tt DIANA} and {\tt SGD} we also run a momentum version, with a momentum parameter in $\{0, 0.95, 0.99\}$.
For {\tt DIANA} we also run with two choices of norm $\ell_2$ and $\ell_\infty$.
For each experiment we have selected softmax cross entropy loss. Mnist-Convex is a simple DNN with no hidden layer, Mnist-DNN is a convolutional NN described here \url{https://github.com/floydhub/mnist/blob/master/ConvNet.py}
and Cifar10-DNN is a convolutional DNN described here
\url{https://github.com/kuangliu/pytorch-cifar/blob/master/models/lenet.py}.
In Figure~\ref{fig:DNN:evolution} we show the best runs over all the parameters for all the methods. 
For Mnist-Convex SGD and {\tt DIANA} makes use of the momentum and dominate all other algorithms.
For Mnist-DNN situation is very similar.
For Cifar10-DNN  both {\tt DIANA} and {\tt SGD} have significantly outperform other methods.
\begin{figure}[h!]
\centering 
\includegraphics[width=0.45\textwidth]{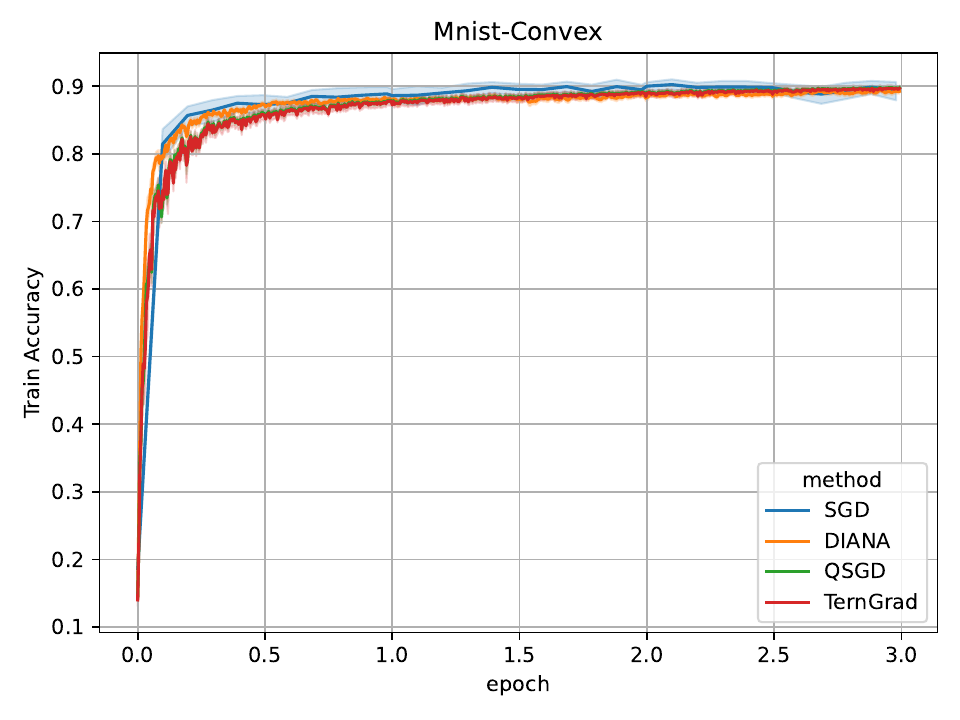}
\includegraphics[width=0.45\textwidth]{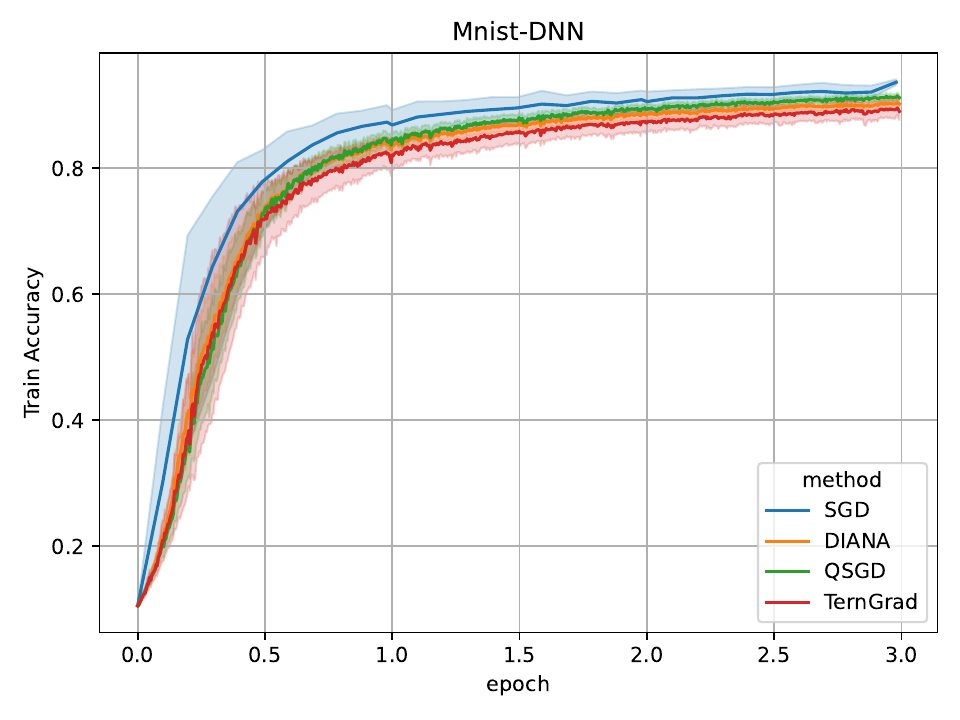}
\includegraphics[width=0.45\textwidth]{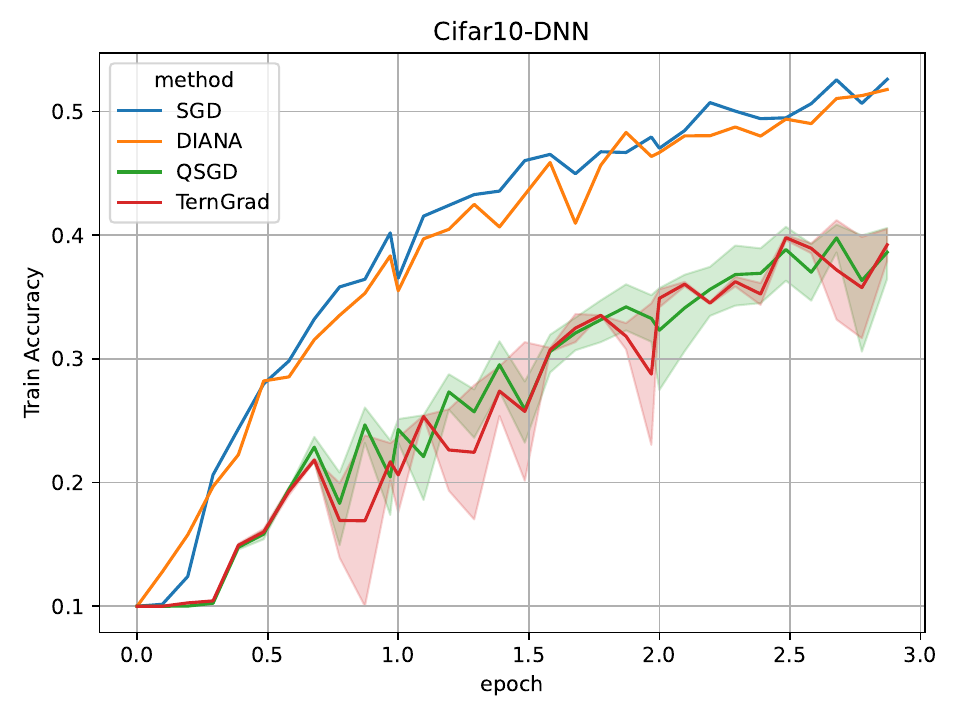}

\includegraphics[width=0.45\textwidth]{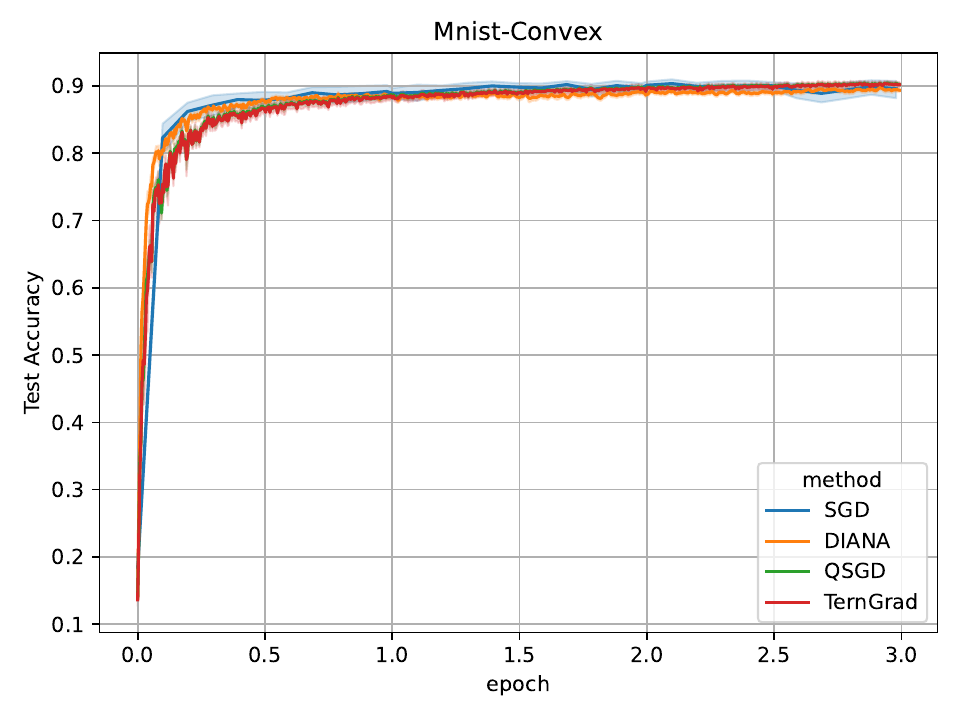}
\includegraphics[width=0.45\textwidth]{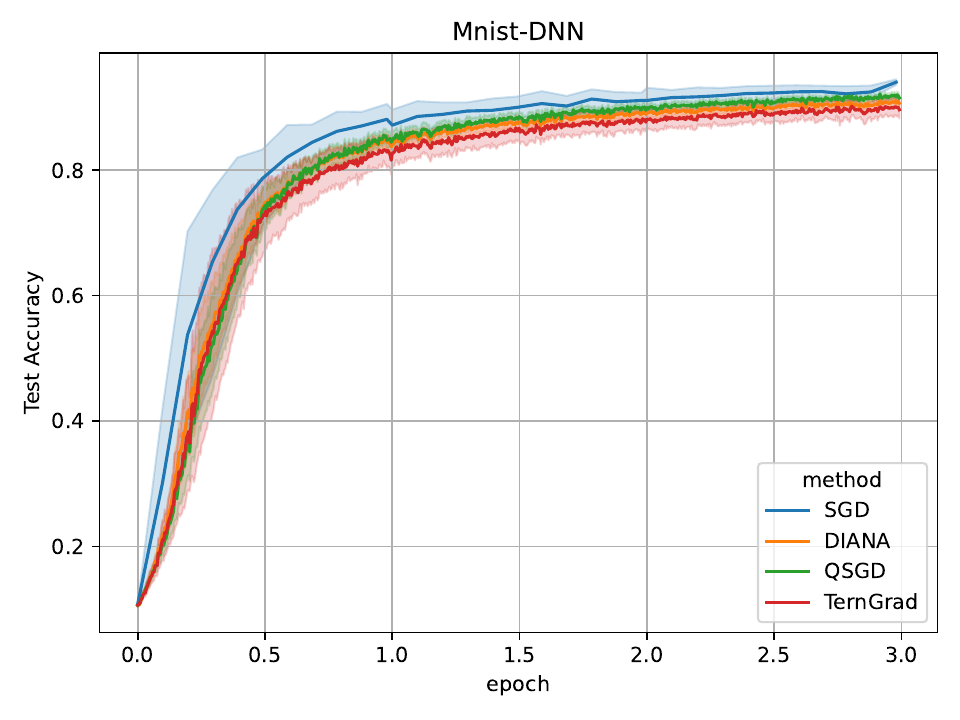}
\includegraphics[width=0.45\textwidth]{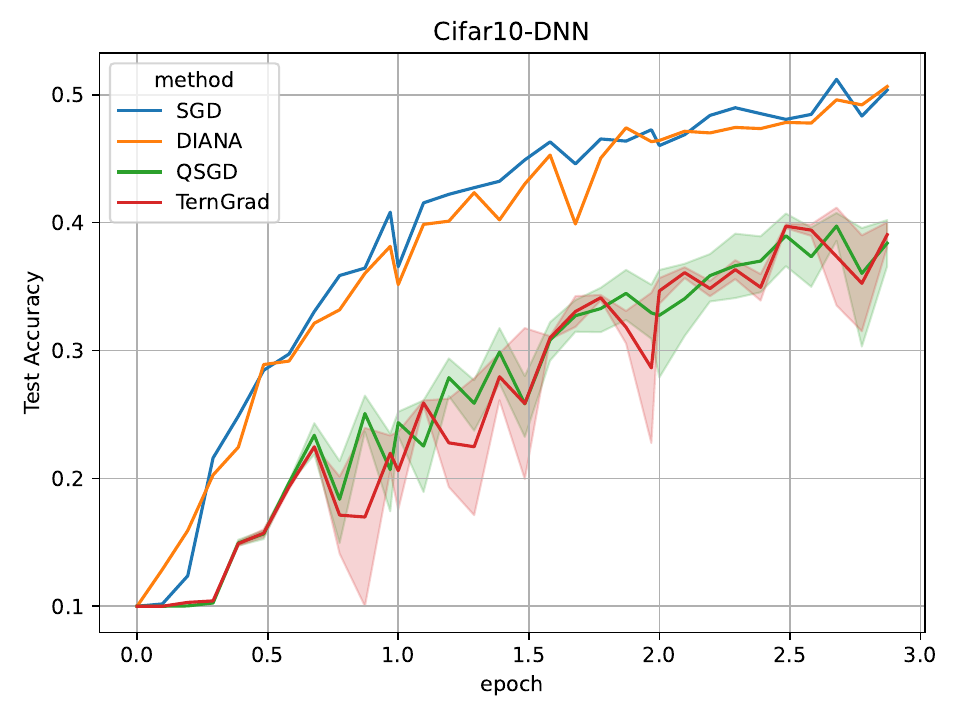}

\caption{Evolution of training and testing accuracy for 3 different problems, using 4 algorithms: {\tt DIANA}, {\tt SGD}, {\tt QSGD} and {\tt TernGrad}. 
We have chosen the best runs over all tested hyper-parameters.}
\label{fig:DNN:evolution}

\end{figure}

In Figure~\ref{fig:DNN:sparsity} show the evolution of sparsity of
the quantized gradient for the 3 problems and {\tt DIANA}, {\tt QSGD} and {\tt TernGrad}. For Mnist-DNN, it seems that the quantized gradients are becoming sparser as the training progresses.

\begin{figure}[h!]

\begin{center}
\includegraphics[width=0.45\textwidth]{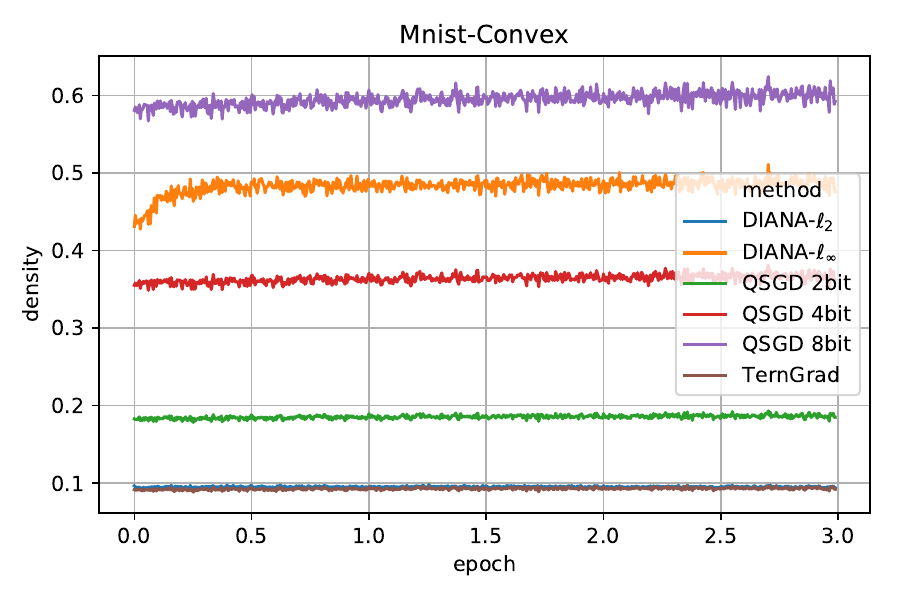}
\includegraphics[width=0.45\textwidth]{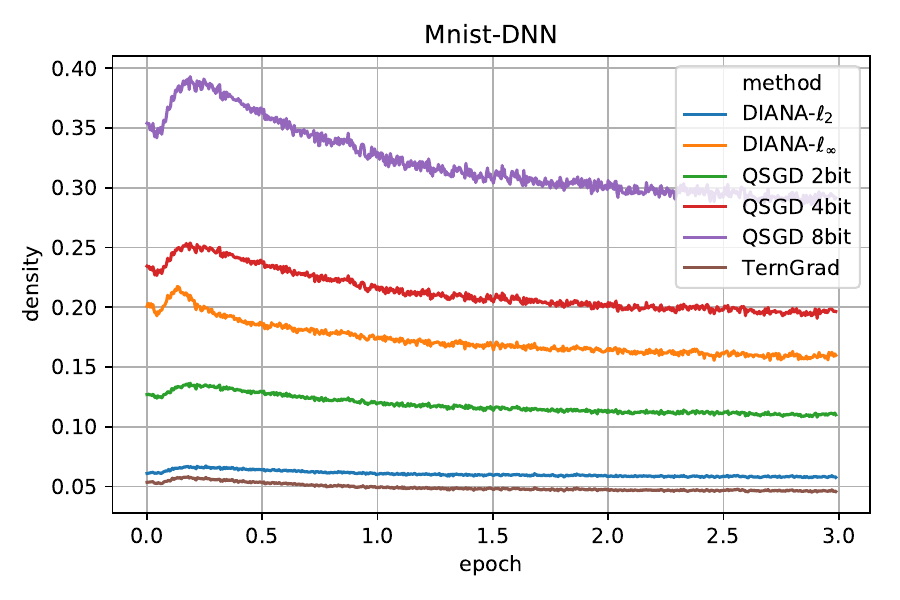}
\includegraphics[width=0.45\textwidth]{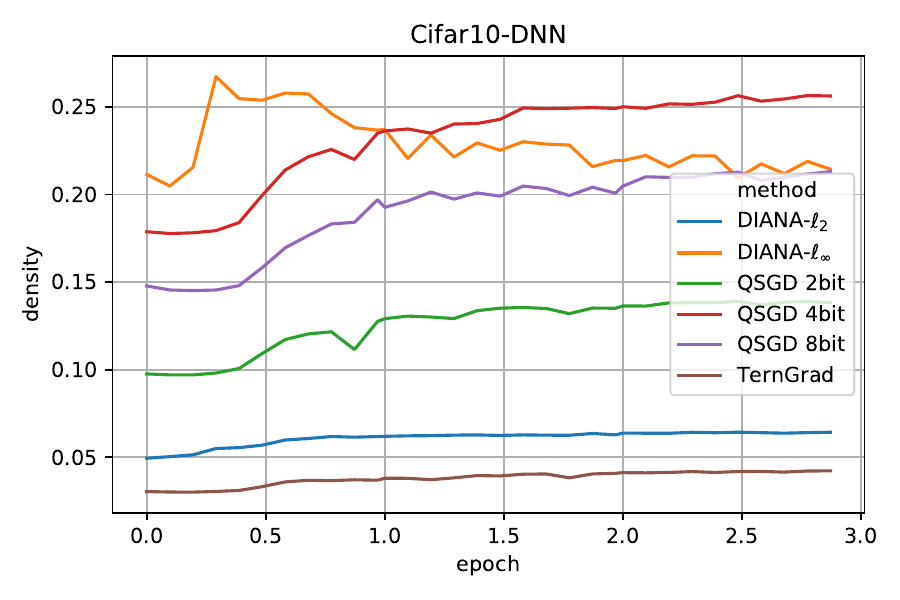}
\end{center}

\caption{Evolution of sparsity of the quantized gradient for 3 different problems and 3 algorithms.}
\label{fig:DNN:sparsity}

\end{figure}

\clearpage 
\subsection{Computational Cost}

\begin{figure}[h!]
\centering

\includegraphics[scale=0.5]{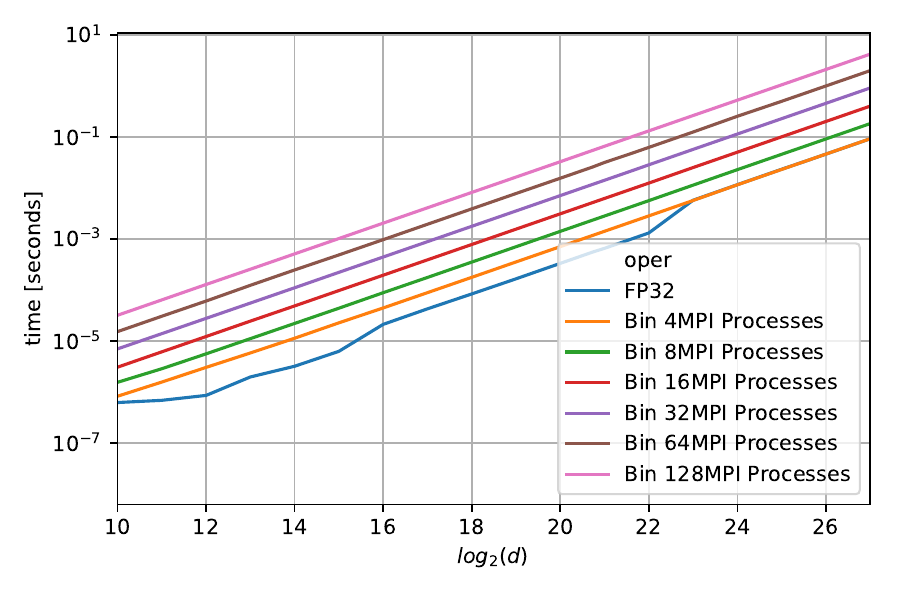}

\caption{Comparison of a time needed to update weights after a reduce vs.
the time needed to update the weights when using a sparse update from DIANA using 4-128 MPI processes and 10\% sparsity.
}

\label{fig:add}

\end{figure}

\iffalse
\clearpage 

\section{Mnist N-1}

\includegraphics[scale=.41]{diana_mnist_N1_lr_8.pdf}
\includegraphics[scale=.41]{diana_mnist_N1_sparsity_8.pdf}

\includegraphics[scale=.41]{diana_mnist_N1_lr_16.pdf}
\includegraphics[scale=.41]{diana_mnist_N1_sparsity_16.pdf}

\includegraphics[scale=.41]{QSGD_mnist_N1_lr_8.pdf}
\includegraphics[scale=.41]{QSGD_mnist_N1_sparsity_8.pdf}

\includegraphics[scale=.41]{QSGD_mnist_N1_lr_16.pdf}
\includegraphics[scale=.41]{QSGD_mnist_N1_sparsity_16.pdf}

\includegraphics[scale=.41]{signsgd_mnist_N1_lr_8.pdf}
\includegraphics[scale=.41]{signsgd_mnist_N1_lr_16.pdf}

\includegraphics[scale=.41]{sgd_mnist_N1_lr_8.pdf}
\includegraphics[scale=.41]{sgd_mnist_N1_lr_16.pdf}

\section{Mnist N-2}

\includegraphics[scale=.41]{diana_mnist_N2_lr_8.pdf}
\includegraphics[scale=.41]{diana_mnist_N2_sparsity_8.pdf}

\includegraphics[scale=.41]{diana_mnist_N2_lr_16.pdf}
\includegraphics[scale=.41]{diana_mnist_N2_sparsity_16.pdf}

\includegraphics[scale=.41]{QSGD_mnist_N2_lr_8.pdf}
\includegraphics[scale=.41]{QSGD_mnist_N2_sparsity_8.pdf}

\includegraphics[scale=.41]{QSGD_mnist_N2_lr_16.pdf}
\includegraphics[scale=.41]{QSGD_mnist_N2_sparsity_16.pdf}

\includegraphics[scale=.41]{signsgd_mnist_N2_lr_8.pdf}
\includegraphics[scale=.41]{signsgd_mnist_N2_lr_16.pdf}

\includegraphics[scale=.41]{sgd_mnist_N2_lr_8.pdf}
\includegraphics[scale=.41]{sgd_mnist_N2_lr_16.pdf}

\clearpage

\fi 
%\appendix
%\onecolumn
\setlength{\footskip}{20pt}

\begin{table}
\begin{center}
\caption{The table of all notations we use in this paper.
\label{tbl:notation-table}
}
\tiny
\begin{tabular}{|c|c|c|}
\hline
Notation & Definition & First appearance\\
\hline
$f(x)$ & Objective function, $f(x) = \frac{1}{n}\sumin f_i(x)$ & Eq.\eqref{eq:main}\\
\hline
$R(x)$ & Regularizer & Eq. \eqref{eq:main}\\
\hline
$n$ & Size of the dataset & Eq. \eqref{eq:main}\\
\hline
$d$ & Dimension of vector $x$ & Eq. \eqref{eq:expected_comm_cost}\\
\hline
$\sign(t)$ & \begin{tabular}{c} The sign of $t$  ($-1$ if $t < 0$, $0$ if $t=0$ and $1$ if $t>1$) \end{tabular} & Eq. \eqref{eq:quant}\\
\hline
$x_{(j)}$ & The $j$-th element of $x\in\R^d$& Eq. \eqref{eq:quant}\\
\hline
$x(l)$ & The $l$-th block of $x = (x(1)^\top,x(2)^\top,\ldots,x(m)^\top)^\top$, $x(l) \in \R^{d_l}$, $\sum\limits_{l=1}^m d_l = d$ & Def.~\ref{def:block-p-quant}\\
\hline
$\|x\|_p, p\ge 1$ & \begin{tabular}{c}$\ell_p$ norm of $x$: $\|x\|_p = \left(\sum\limits_{j=1}^d|x_{(j)}|^p\right)^{\frac{1}{p}}$ for $1\le p < \infty$\\ $\|x\|_\infty = \max\limits_{j=1,\ldots,d}|x_{(j)}|$ \end{tabular} & Eq. \eqref{eq:smoothness_functional}\\
\hline
$\|x\|_0$ & Number of nonzero elements of $x$ & Eq. \eqref{eq:expected_comm_cost}\\
\hline
$L$ & Lipschitz constant of the gradient of $f$ w. r. t. $\ell_2$ norm & Eq. \eqref{eq:smoothness_functional}\\
\hline
$\mu$ & Strong convexity constant of $f$ w. r. t. $\ell_2$ norm & Eq. \eqref{eq:strong_cvx_functional} \\
\hline
$\kappa$ & Condition number of function $f$: $\kappa = \frac{L}{\mu}$& Cor.~\ref{cor:DIANA-strong-convex} \\  
\hline
$g_i^k$ & Stochastic gradient of function $f_i$ at the point $x=x^k$ & Eq. \eqref{eq:bounded_noise} \\
\hline
$g^k$ & Stochastic gradient of function $f$ at the point $x=x^k$: $g^k = \frac{1}{n}\sumin g_i^k$ & Eq. \eqref{eq:hat_g_expectation} \\ 
\hline
$\sigma_i^2$ & Variance of the stochastic gradient $g_i^k$ & Eq. \eqref{eq:bounded_noise} \\
\hline
$\sigma^2$ & Variance of the stochastic gradient $g^k$: $\sigma^2 = \frac{1}{n}\sum\limits_{i=1}^n\sigma_i^2$ & Eq. \eqref{eq:hat_g_expectation} \\ 
\hline
$h_i^k$ & Stochastic approximation of the $\nabla f_i(x^*)$; $h_i^{k+1} = h_i^k + \alpha\hat\Delta_i^k$ & Alg.~\ref{alg:distributed1} \\
\hline
$\Delta_i^k$ & $\Delta_i^k = g_i^k - h_i^k$ & Alg.~\ref{alg:distributed1}\\
\hline
$\text{Quant}_p(\Delta)$ & Full $p$-quantization of vector $\Delta$ & Def.~\ref{def:p-quant}\\ 
\hline
$\text{Quant}_p(\Delta,\{d_l\}_{l=1}^m)$ & Block $p$-quantization of vector $\Delta$ with block sizes $\{d_l\}_{l=1}^m$ & Def.~\ref{def:block-p-quant} \\
\hline
$d_l$ & Size of the $l$-th block for quantization & Def.~\ref{def:block-p-quant}\\ 
\hline
$m$ & Number of blocks for quantization & Def.~\ref{def:block-p-quant}\\
\hline
$\alpha,\gamma^k$ & Learning rates & Alg.~\ref{alg:distributed1}\\
\hline
$\beta$ & Momentum parameter & Alg.~\ref{alg:distributed1}\\
\hline
$\hat \Delta_i^k$& Block $p$-quantization of $\Delta_i^k = g_i^k - h_i^k$ & Alg.~\ref{alg:distributed1} \\
\hline
$\hat \Delta$ & $\hat \Delta^k = \frac{1}{n}\sumin\hat\Delta_i^k$ &Alg.~\ref{alg:distributed1} \\ 
\hline
$\hat g_i^k$ & Stochastic approximation of $\nabla f_i(x^k)$; $\hat g_i^k = h_i^k + \hat\Delta_i^k$ & Alg.~\ref{alg:distributed1}\\
\hline
$\hat g^k$ & $\hat g^k = \frac{1}{n}\sumin \hat g_i^k$ & Alg.~\ref{alg:distributed1} \\
\hline
$v^k$ & Stochastic gradient with momentum: $v^k = \beta v^{k-1} + \hat g^k$ & Alg.~\ref{alg:distributed1} \\
\hline
$h^{k+1}$ & $h^{k+1} = \frac{1}{n}\sumin h_i^{k+1}$ & Alg.~\ref{alg:distributed1} \\
\hline
$\proxR(u)$ & $\arg\min\limits_{v}\left\{\gamma R(v) + \frac{1}{2}\|v-u\|_2^2 \right\}$ & Alg.~\ref{alg:distributed1} \\
\hline
$\Psi_l(x)$ & Variance of the $l$-th quantized block: $\Psi_l(x) = \|x(l)\|_1\|x(l)\|_p - \|x(l)\|_2^2$  & Eq.~\eqref{eq:tilde_v_moments1}\\
\hline
$\Psi(x)$ & Variance of the block $p$-quantized vector: $\Psi(x) = \sum\limits_{l=1}^m\Psi_l(x)$ & Eq.~\eqref{eq:hat_v_moments1} \\
\hline
$\alpha_p(d)$ & $\alpha_p(d) = \inf\limits_{x\neq 0, x\in\R^d}\frac{\|x\|_2^2}{\|x\|_1\|x\|_p}$ & Eq.~\ref{eq:alpha_p} \\
\hline
$\widetilde d$ & $\widetilde d = \max\limits_{l=1,\ldots,m}d_l$ & Th.~\ref{thm:DIANA-strongly_convex} \\
\hline
$\alpha_p$ & $\alpha_p = \alpha_p(\widetilde d)$ & Th.~\ref{thm:DIANA-strongly_convex} \\
\hline
$c$ & Such number that $\frac{1+nc\alpha^2}{1+nc\alpha} \le \alpha_p$ & Th.~\ref{thm:DIANA-strongly_convex} \\
\hline
$x^*$ & Solution of the problem \eqref{eq:main} & Eq. \eqref{eq:strong_convex_Lyapunov} \\
\hline
$h_i^*$ & $h_i^* = \nabla f_i(x^*)$ & Th.~\ref{thm:DIANA-strongly_convex}  \\
\hline
$V^k$ & Lyapunov function $V^k = \|x^k - x^*\|_2^2 + \frac{c\gamma^2}{n}\sumin\|h_i^k - h_i^*\|$ & Th.~\ref{thm:DIANA-strongly_convex} \\
\hline
$\zeta$ & Bounded data dissimilarity parameter: $\frac{1}{n}\sum_{i=1}^n\|\nabla f_i(x) - \nabla f(x)\|_2^2 \le \zeta^2$ & Eq. \eqref{eq:almost_identical_data}\\
\hline
$\delta, \omega$ & Parameters for the proof of momentum version of {\tt DIANA} & Th.~\ref{thm:DIANA-momentum}\\ 
\hline
$\eta, \theta, N, C$ & Parameters for the decreasing stepsizes results & Th.~\ref{th:str_cvx_decr_step} \\
\hline
$\EE_{Q^k}$ & Expectation w. r. t. the randomness coming from quantization & Lem.~\ref{lem:3in1} \\
\hline

\end{tabular}
\end{center}
\end{table}

\end{document}